\documentclass{article} % For LaTeX2e
\usepackage{times}
\usepackage[margin=1in]{geometry}
\usepackage{natbib}
\setcitestyle{authoryear,round,citesep={;},aysep={,},yysep={;}}

\usepackage{amsmath,amsfonts,bm}

\def\eqref#1{equation~\ref{#1}}
\def\1{\bm{1}}

\DeclareMathAlphabet{\mathsfit}{\encodingdefault}{\sfdefault}{m}{sl}
\SetMathAlphabet{\mathsfit}{bold}{\encodingdefault}{\sfdefault}{bx}{n}

\usepackage{etoc}
\usepackage[utf8]{inputenc}
\usepackage[T1]{fontenc}
\usepackage{microtype}
\usepackage{mathrsfs}
\usepackage{amsmath}
\usepackage{amssymb}
\usepackage{amsfonts}
\usepackage{amsthm}
\usepackage{mathtools}
\usepackage{stmaryrd}
\usepackage{bbm}
\usepackage{bm}
\usepackage{nicefrac}

\usepackage{graphicx}
\usepackage{booktabs}
\usepackage{tabularx}
\usepackage{multirow}
\usepackage{makecell}
\usepackage{array}
\usepackage{adjustbox}
\usepackage{arydshln}
\usepackage{float}

\usepackage{tikz}
\usetikzlibrary{positioning}

\usepackage{algorithm}
\usepackage{algpseudocode}

\algnewcommand\algorithmicinput{\textbf{INPUT:}}
\algnewcommand\INPUT{\item[\algorithmicinput]}
\algnewcommand\algorithmicoutput{\textbf{OUTPUT:}}
\algnewcommand\OUTPUT{\item[\algorithmicoutput]}

\usepackage{xcolor}

\usepackage{hyperref}
\hypersetup{hypertexnames=false} % Unique PDF destinations for tagged equations.
\usepackage{url}
\usepackage{cleveref}

\newtheorem{assumption}{Assumption}
\newtheorem{lemma}{Lemma}

\newtheorem{corollary}{Corollary}
\newtheorem{definition}{Definition}

\newtheorem{theorem}{Theorem}
\newtheorem{example}{Example}

\DeclareMathOperator*{\argmin}{arg\,min}

\newcommand{\R}{\mathbb R}

\newcommand{\E}{\mathbb E}
\newcommand{\Prob}{\mathbb P}
\newcommand{\F}{\mathcal F}

\title{Distributional Regression for Arbitrary Response Variables Using ReLU Neural Networks}

\author{
Junpeng Ren \\
Department of Statistics and Data Science \\
University of California, Los Angeles \\
Los Angeles, CA, USA \\
\texttt{junren18@g.ucla.edu} \\
\and
Carlos Misael Madrid Padilla \\
Department of Statistics and Data Science \\
Washington University in St Louis \\
St. Louis, MO, USA \\
\texttt{carlosmisael@wulst.edu} \\
\and
Oscar Hernan Madrid Padilla \\
Department of Statistics and Data Science \\
University of California, Los Angeles \\
Los Angeles, CA, USA \\
\texttt{oscar.madrid@stat.ucla.edu}
}

\begin{document}

\maketitle

\begin{abstract}
This paper develops a unified framework for distributional regression with arbitrary response variables. Rather than relying on scalar thresholds or Euclidean structure, we model conditional event probabilities indexed by measurable sets, which extends distributional regression to general response spaces while retaining a direct connection to classical event-based formulations. Within this framework, we study two complementary estimators: a preliminary estimator based on separate regression of event indicators, and a constrained estimator that enforces inclusion-consistent predictions through a directed acyclic graph (DAG) over a finite set family. We establish general error bounds for both estimators and further specialize the results to dense ReLU networks, yielding explicit convergence rates overcoming the curse of dimensionality. Simulation studies and real-data experiments further demonstrate the practical effectiveness of the proposed framework.
\end{abstract}
\section{Introduction}
\label{sec:intro}
While regression methods are widely used in statistics and machine learning, it is well recognized that the conditional mean alone often fails to capture the full relationship between a response variable and covariates \citep{koenker1978regression, shaked2007stochastic}. In contrast, distributional regression aims to directly model the conditional distribution of the response, providing a richer characterization of the underlying stochastic relationship \citep{kneib2023rage}. This has led to growing interest in distributional regression across a range of applications, such as economics \citep{foresi1995conditional, rothe2012partial, wang2023bivariate}, weather forecasting \citep{schlosser2019distributional}, and treatment effect estimation in clinical trials \citep{heller2022distributional}.

However, most existing approaches to distributional regression focus on real valued scalar \citep{hothorn2014conditional, padilla2025risk} or Euclidean-valued responses \citep{cevid2022distributional, klein2022multivariate, kock2025truly}, where the conditional distribution can be naturally described via cumulative distribution functions, quantiles, or threshold event probabilities. When the response variable is a more complex object, such as a network, a functional object, or more general structured data, it becomes much less clear how to define and estimate a conditional distribution in a unified and tractable way. {In this paper, we consider independent observations $\{(x_i,y_i)\}_{i=1}^n$ from $(X,Y)\in[0,1]^d\times\mathcal{Y}$, and {an arbitrary collection $\Lambda\subseteq\mathscr Y$ of measurable sets in the response space $(\mathcal{Y},\mathscr Y)$}.} Rather than modeling the conditional distribution of $Y$ directly, we study the family of conditional event probabilities
\begin{equation}
\label{eqn:main}
    F_A^*(x) = P(Y \in A \mid X = x), \qquad A \in \Lambda.
\end{equation}
This formulation is intrinsic to $\mathcal Y$ and places no Euclidean structure on the response.%, while reducing to existing threshold-event methods when $\mathcal Y = \mathbb R^u$.

We propose two complementary estimators within this framework. As a natural baseline approach considered in \citet{padilla2025risk}, the \emph{preliminary estimator} regresses the event indicators $\mathbbm{1}_{\{y_i\in A_j\}}$ on $x_i$ separately for each set. In our setting, this yields a simple and broadly applicable baseline for distributional regression with arbitrary response spaces, and allows us to study whether uniform statistical guarantees can be obtained over a possibly infinite family of events. 

The formulation in \eqref{eqn:main} also induces structural constraints: whenever $A \subset A'$, we must have $F_A^*(x) \leq F_{A'}^*(x)$ for all $x \in [0,1]^d$ (a standard setting in nonparametric theory; see \citet{gyorfi2002distribution}). This monotonicity is also desirable for interpretable probability estimates. To incorporate it, we introduce a \emph{constrained estimator}. We encode the inclusion relations among a {finite subfamily $\Lambda_m=\{A_j\}_{j=1}^m\subseteq\Lambda$} through a DAG and construct a joint estimator that enforces inclusion-consistent predictions by design. For these two estimators, general risk bounds are established, accommodating a variety of nonparametric estimators. Finally, motivated by the success of deep ReLU networks and their proven ability to overcome the curse of dimensionality as nonparametric estimators \citep{schmidt2020nonparametric, kohler2019rate, padilla2022quantile, padilla2024confidence, ren2026transfer}, we further derive explicit convergence rates under such instantiations as an application of our general theory.
\subsection{Summary of results}
\textbf{Framework and estimators.} We develop a unified framework for distributional regression over arbitrary response spaces, where the target is a family of conditional event probabilities $F_A^*(x)=P(Y\in A\mid X=x),$ indexed by a collection of measurable sets. Within this framework, we study two complementary estimators: the preliminary estimator provides a general baseline for arbitrary-response distributional regression, accommodating collections of sets of infinite cardinality, whereas the constrained estimator delivers structurally consistent joint estimation over a finite set family.

\textbf{General theory.} We establish a general theory for both estimators, accommodating a broad class of nonparametric estimators under mild complexity conditions on the estimators. For the preliminary estimator, as long as the set family is totally ordered under inclusion, Theorem~\ref{theo0} provides uniform empirical error bounds over the set family. Corollary~\ref{cor2} further extends Theorem~\ref{theo0} to the important case of multivariate Euclidean responses with lower-orthant set families, even though such families need not satisfy the full nesting condition of Theorem~\ref{theo0}. {For the constrained estimator, Theorem~\ref{theo1} and bound~\ref{eq:growing-m-main} establish {aggregated empirical squared-error guarantees} when the number of sets forming the inclusion DAG is allowed to grow with the sample size $n$, with the dependence on the DAG topology characterized by the number of ancestors of each node and the lengths of its root-to-node paths.}

{\textbf{Convergence rates for Dense ReLU Networks.}  We instantiate the general theory with dense ReLU network estimators.  Under hierarchical composition assumptions in \citep{kohler2019rate}, we derive explicit bounds for both estimators, with Corollary~\ref{cor1} covering the constrained estimator, building on a new DNN approximation theorem showing that the target probability functions can be well approximated through our DAG-constrained parametrization, and Corollaries~\ref{eqn:cor3} and~\ref{cor4} covering the preliminary estimator through specializations of Theorem~\ref{theo0} and Corollary~\ref{cor2}, respectively. Appendix~\ref{app:additional-estimators} further applies the same general theory to trend filtering \citep{mammen1997locally,kim2009ell_1,tibshirani2014adaptive} and orthogonal series regression, illustrating broader applicability.}

\subsection{Related work}

\textbf{Distributional regression.} Distributional regression aims to characterize the full conditional distribution of a response, rather than only its conditional mean. Existing approaches fall broadly into two categories. The first models the entire conditional distribution directly, typically in a parametric way through covariate dependent transformations of a reference distribution \citep{hothorn2014conditional, klein2022multivariate, kock2025truly}, with recent work also using generative neural network models \citep{shen2025engression}. The second, closer to our work, models conditional distributions through event indicators such as $\mathbbm{1}_{\{Y \le t\}}$; this line originates from \citet{foresi1995conditional} and has since been further developed in work on misspecification testing \citep{rothe2013misspecification}, distributional regression with random forests \citep{schlosser2019distributional, cevid2022distributional}, and neural networks \citep{padilla2025risk}. However, across these settings, the response is typically assumed to be scalar or Euclidean vector valued, and the target events are usually defined through threshold based sets in the response space. In contrast, our framework treats distributional regression over arbitrary response spaces via set indexed conditional event probabilities, allowing for general measurable sets.

\textbf{Structured estimation under order constraints.} Our constrained estimator is related to recent work on estimation under structural or order constraints. In particular, \citet{mosching2020monotone, henzi2021isotonic} study settings in which the conditional distribution or quantiles are required to vary monotonically with the covariates, while \citet{henzi2023consistent} imposes stochastic order constraints directly on a collection of distribution functions. In contrast to these works, our structural constraint is imposed on the family of measurable sets itself: we encode inclusion relations through a DAG and enforce estimated probabilities that satisfy these inclusion relations across the set family. More broadly, the importance of structural constraints also arises in the deep learning literature, where efforts have been made to incorporate structural knowledge into neural networks so that the resulting predictors satisfy prescribed monotonic or logical relations specified from prior knowledge \citep{sill1997monotonic, hu2016harnessing, you2017deep, xu2018semantic, imani2018improving}; however, these works are primarily methodological and generally do not provide statistical guarantees. Our work is related in spirit, since the measurable sets and their inclusion relations can likewise be specified abstractly from prior knowledge. Building on this perspective, we further derive rigorous risk bounds when deep ReLU networks are used as the underlying estimator in our constrained estimator, thereby extending the framework of \citet{kohler2019rate}.

\textbf{Regression with complex responses.} A related line of work considers regression problems where the response lies in a non-Euclidean or structured space, such as Fr\'echet regression \citep{petersen2019frechet, chen2022uniform, bhattacharjee2023single} and network regression \citep{calissano2022graph, zhou2022network}. These approaches demonstrate that regression can be extended beyond Euclidean responses, but they typically focus on estimating conditional means or barycenters of objects. In contrast, our goal is to estimate conditional distributional information through event probabilities indexed by sets, which provides a different complementary perspective.
\subsection{Notation}
\label{sec:notation}
	 Write $\tau(x) = 1/(1+e^{-x})$ for the sigmoid function. For $n\in\mathbb{N}$, we write $[n]=\{1,\ldots,n\}$. {We use $\mathbb{N}=\{1,2,\ldots\}$ and $\mathbb{N}_0=\mathbb{N}\cup\{0\}$.} For a function $g\,:\, \mathcal{X} \,\rightarrow \, \mathbb{R}$, we define 
	$\|g\|_n=(n^{-1}\sum_{i=1}^n g(x_i)^2)^{1/2}, \,\,\,\,\text{and}\,\,\,\, \|g\|_\infty=\sup_{x\in\mathcal X}|g(x)|$.
	%where $n_{\ell} \,:=\, \vert \{ i\,:\, z_i  = \ell  \}\vert$.
	For two positive sequences $a_n$ and $b_n$, let $a_n=O(b_n)$ or $a_n \lesssim b_n$ when $a_n \leq Cb_n$ for some constant $C>0$ which is independent of $n$, and $a_n = \Theta(b_n)$ or $a_n \asymp b_n$ when $a_n=O(b_n)$ and $b_n=O(a_n)$. {For a metric space $(\mathcal X,d)$, a subset $\mathcal C\subset\mathcal X$ is an $r$-external cover of $K\subset\mathcal X$ if $K\subseteq\bigcup_{x\in\mathcal C}B_r(x,d)$. The external covering number $N(r,K,d)$ is the minimal cardinality among such covers. For $\eta>0$, define $B_\eta(0)=\{f\in\R^{\mathcal X}:\|f\|_n\le\eta\}$. Finally, $X_n=O_{\mathbb{P}}(a_n)$ means that for every $\epsilon>0$ there exist $C,n_0$ such that $\mathbb{P}(|X_n|>Ca_n)\le\epsilon$ for all $n\ge n_0$.} {Throughout, $F_A^*(x)$ denotes a fixed measurable version of
$\E[\1_{\{Y\in A\}}\mid X=x]$. We assume that these versions are chosen
so that $F_A^*(x)\le F_B^*(x)$ whenever $A\subseteq B$, for all $x$;
this holds, for example, when the conditional law of $Y$ given $X=x$
admits a regular conditional probability kernel. Standard measurability
conditions for the empirical-process suprema and estimator selections
are assumed.}

\subsection{Outline}
The remainder of this paper is organized as follows. Section~\ref{sec:methodology} introduces the model setup and the two estimators studied in this paper. Section~\ref{sec:theory}  establishes the theoretical guarantees for both estimators, with Section~\ref{subsec:pre} covering the general theory for the preliminary estimator, Section~\ref{sec:general_constrained} for the constrained estimator, {and Section~\ref{sec:nn_theo_constrain} specializing the constrained result to dense ReLU networks. Appendix~\ref{app:additional-estimators} gives theoretical results for trend filtering and orthogonal series regression.} Section~\ref{sec:num-experiment} and Appendix~\ref{app:additional_scenarios} report comprehensive simulation studies. Section~\ref{sec:real-experiment} and Appendix~\ref{app:chicago} together present two real-data examples. Finally, Section~\ref{sec:conclusion} concludes the paper with potential extensions.

\section{Methodology}
\label{sec:methodology}

Given the setup in Section~\ref{sec:intro}, we now formalize the estimation procedure by considering an arbitrary function class $\mathcal{F}$  from which the estimator is constructed.

\textbf{Preliminary estimator:} {For each $A\in\Lambda$, the preliminary estimator is obtained by:
\begin{equation}
{\widehat F_A^{\mathrm{pre}}\in\argmin_{f\in\F_A}\sum_{i=1}^n(z_i^A-f(x_i))^2,\quad z_i^A=\1_{\{y_i\in A\}},\quad\F_A\subseteq\F.}\tag{2}\label{eqn:preliminary_estimator}
\end{equation}
In finite computational applications we take $A=A_j$, $j\in[m]$, and typically $\F_{A_j}=\F$.} The estimator in \eqref{eqn:preliminary_estimator} serves as a general baseline for arbitrary response spaces $\mathcal{Y}$, as it relies only on set membership evaluation. However, since the estimators in \eqref{eqn:preliminary_estimator} are computed separately for each $j$, they do not necessarily satisfy the inclusion-induced monotonicity condition $F_A^*(x) \leq F_{A'}^*(x)$ for $A \subset A'$. This limitation motivates the constrained estimator introduced next.

 \textbf{Constrained estimator:} To enforce inclusion consistency by construction, we encode {the finite set family $\Lambda_m$} as a directed acyclic graph.
Consider $G=(V,E)$ with $V=\{1,\ldots,m\}$ and $E\subset V\times V$
such that $(j,j')\in E$ if and only if $A_j\subset A_{j'}$ and, for all
$k\neq j,j'$, it does not hold that $A_j\subset A_k\subset A_{j'}$. {For $j\in V$, let $\mathcal P(j)=\{i:(i,j)\in E\}$ denote its parents
and let $\mathcal R=\{j:\mathcal P(j)=\varnothing\}$ denote the roots.
Given functions $g_j:[0,1]^d\to\R$, $j\in V$, let
$\bm g=(g_1,\ldots,g_m)$ denote the collection of the functions.
Define the auxiliary outputs recursively by
\[
\widetilde H_j(\bm g)(x)=
\begin{cases}
g_j(x),&j\in\mathcal R,\\
\displaystyle
\max_{k\in\mathcal P(j)}\widetilde H_k(\bm g)(x)
+\log(1+e^{g_j(x)}),&j\notin\mathcal R,
\end{cases}
\qquad
H_j(\bm g)=\tau\circ\widetilde H_j(\bm g).
\]}The sigmoid $\tau$ guarantees outputs in $(0,1)$, while the softplus
increment $\log(1+e^{g_j})$ is strictly positive, so
$\widetilde H_{j'}(\bm g)\geq\widetilde H_j(\bm g)$ whenever there is
a path from $j$ to $j'$ in $G$. Since $\tau$ is monotone, this yields $H_{j'}(\bm g)(x)\geq H_j(\bm g)(x)$ for all such pairs and all $x\in[0,1]^d$. With this notation in hand, for a function class $\mathcal F$ we define
the \emph{constrained estimator} as
\begin{equation}
    \label{eqn:proposed_estimator1}
    \widehat{\bm g}
    :=\underset{\bm g\in\mathcal F^m}{\arg\min}\,
    \sum_{j=1}^m\sum_{i=1}^n
    \bigl(z_i^j-H_j(\bm g)(x_i)\bigr)^2,
\end{equation}
and set
\begin{equation}
    \label{eqn:proposed_estimator2}
    \widehat F_{A_j}
    :=H_j(\widehat{\bm g}),
    \qquad j=1,\ldots,m.
\end{equation}
 so the resulting estimator is inclusion-consistent by construction. In practice, the function class $\mathcal{F}$ can be chosen flexibly. {For example, one may take the DNN class $\mathcal{F} = \{f \in \mathcal{F}(L,w) : \|f\|_\infty \leq B\}$, where $\mathcal{F}(L,w)$ denotes the class of ReLU neural networks with $L$ hidden layers and $w$ neurons per hidden layer (see Appendix~\ref{sec:relu} for the formal definition).}

	\section{Theory}
    \label{sec:theory}
%   We now establish the general theoretical guarantees for the two estimators introduced in Section~\ref{sec:methodology}. Throughout, results are stated for a generic function class $\mathcal{F}$ of the estimator under mild complexity conditions, and therefore apply broadly beyond the neural network setting.
%We now establish the theoretical guarantees for the two estimators introduced in Section 2. Section 3.1 develops the general theory for the preliminary estimator, Section 3.2 for the constrained estimator, and Section 3.3 instantiates both with dense ReLU networks. Throughout, results are stated for a generic function class $\mathcal{F}$ of the estimator under mild complexity conditions, and therefore apply broadly beyond the neural network setting.

	\subsection{General theory for preliminary estimator}
    \label{subsec:pre}
	Recall the motivation behind the preliminary estimator: we aim to obtain a \emph{uniform} estimation error bound for $\{F_A^*\}_{A \in \Lambda}$ over a set family $\Lambda$ that may be uncountable, on an arbitrary response space $\mathcal{Y}$. To make such uniform control possible, some structure on the set family is required. The following assumption imposes that $\Lambda$ be totally ordered under inclusion.

	\begin{assumption}
		\label{as0} Let $\Lambda\subset 2^\mathcal{Y}$ be a collection of measurable sets. Suppose for any  $A , B \in \Lambda$ with $A \neq B$ either $ A \subsetneq B$ or   $B  \subsetneq  A.$
	\end{assumption}
    Although Assumption~\ref{as0} may appear strong at first sight, it captures many naturally arising set families. As a broadly applicable illustration:
\begin{example}
    \label{ex1} Suppose that $(\mathcal{Y},d_{\mathcal{Y}})$ is a metric space and let $y_0 \in \mathcal{Y}$ be a fixed element of $\mathcal{Y}$. Let
    \[
     B(y_0,t) \,:=\,\{ y\in \mathcal{Y}\,:\,  d_{\mathcal{Y}}(y,y_0)\leq t  \}
    \]
    be the ball of radius $t$ centered at $y_0$. Then, $\Lambda\,:=\,\{  B(y_0,t)    \,:\,  t>0 \}$
    satisfies Assumption \ref{as0}. 
\end{example}
    Example~\ref{ex1} shows that Assumption~\ref{as0} allows us to define distributional regression problems on \emph{any} metric space---including networks, function spaces, and manifolds---so that the distribution of the response's distance to a reference point $y_0$ can be uniformly estimated. This is precisely what our goal of distributional regression over arbitrary response spaces is meant to achieve. Under this assumption, we obtain the following result for the preliminary estimator.

	\begin{theorem}
	    \label{theo0}
        {Suppose Assumption~\ref{as0} and the measurability conventions in Section~\ref{sec:notation} hold}.         Consider $z_{i}^A = \mathbbm{1}_{\{ y_i \in A  \}  }$ for $i=1,\ldots,n$ and $A \in \Lambda$. Let $\widehat{F}^{ \text{pre} }_{A}  \,\in \underset{f \in \mathcal{F}_A  }{\arg \min}\, \sum_{i=1}^{n} ( z_{i}^A -  f(x_i)  )^2,$
     where $\mathcal{F}_A\subset \mathcal{F}$.
     	{For each $A$, choose $G_A\in\F_A$ and put $b=\sup_{A\in\Lambda}\|G_A-F_A^*\|_\infty$.}
    Suppose $\| F_A\|_{\infty} \,\leq\, U$,  for all $F_A\in \mathcal{F}_A $ and $A \in \Lambda$, for some $U>0$. {For $1<\eta<2U\sqrt n$, set $J=\left\lceil\log_2(4nU^2/\eta^2)\right\rceil$,
\[
\quad \rho_j=\frac{2^{j/2}\eta}{\sqrt n},
\quad\mathcal D(\rho)=(\F-\F)\cap B_\rho(0),
\quad
\mathcal J(\rho)=\int_0^\rho\sqrt{\log N(u,\mathcal D(\rho),\|\cdot\|_n)}\,du.
\]
Then, conditionally on the design,
\begin{equation}
\begin{split}
\Prob\left\{\sup_{A\in\Lambda}\|\widehat F_A^{\mathrm{pre}}-F_A^*\|_n>
\frac{\eta}{\sqrt n}+b\,\middle|\, \{x_i\}_{i=1}^n \right\}
\lesssim \frac{\sqrt n}{\eta^2}\sum_{j=1}^J\frac{\mathcal J(\rho_j)}{2^{j}}
+\frac{\sqrt{\log(2n+2)}}{\eta}+\frac{\sqrt n b}{\eta}.
\end{split}\label{eq:precorrect}
\end{equation}}
    
    \end{theorem}

    %Theorem~\ref{theo0} provides uniform empirical error control over a totally ordered family for arbitrary estimator classes $\F_A$.  {The right-hand side is a tail upper bound. For concrete classes, evaluating the shell-dependent entropy integral permits balancing approximation and stochastic errors. The factor $4U^2$ is the squared diameter of an envelope-$U$ class; the localized class must expand with the shell.}
    
    Theorem~\ref{theo0} provides uniform empirical error control over the totally ordered family $\Lambda$ for arbitrary estimator classes $\F_A$. {The threshold contains the approximation error $b$ and a term $\eta/\sqrt n$, while the tail bound is governed by the entropy integral $\mathcal J(\rho_j)$ over the shells $\rho_j$. For concrete classes, evaluating $\mathcal J(\rho_j)$ yields an explicit stochastic bound.}
    
%The supremum in Theorem~\ref{theo0} is taken over totally ordered set families. In many applications, such as multivariate responses $Y \in \mathbb{R}^u$, one instead wishes to control the estimation error uniformly over \emph{rectangular} set families (lower-orthant rectangles), which need not satisfy the total-order condition of Assumption~\ref{as0}. Nevertheless, the indicator vectors induced by lower-orthant rectangles still take at most $n^u$ distinct values, and a similar argument yields a uniform upper bound at the cost of replacing the $\sqrt{\log n}$ factor by $\sqrt{u \log n}$ in Corollary~\ref{cor2}.

The supremum in Theorem~\ref{theo0} is taken over totally ordered set families. In many applications, such as $Y\in\R^u$, one instead wishes to control the estimation error uniformly over lower-orthant rectangles, which need not be totally ordered. {Nevertheless, the indicator vectors induced by lower-orthant rectangles still take at most $(n+1)^u$ distinct values, and a similar argument yields a uniform bound.}
%at the cost of replacing the $\sqrt{\log n}$ factor by $\sqrt{u \log n}$ in Corollary~\ref{cor2}.

\begin{corollary}
    \label{cor2} Consider the above notation of Theorem \ref{theo0} where  $\mathcal{Y} = \mathbb{R}^u$ for some $u \in \mathbb{N}$, and  let $A(t) \,=\, \{y \in \mathbb{R}^u\,:\,  y_j \leq  t_j \,\,\forall j \in \{1,\ldots,u\} \}$
    for $t \in \mathbb{R}^u$. Then, the conclusion of Theorem \ref{theo0} holds for $
    \Lambda\,:=\, \{ A(t) \,:\, t \in \mathbb{R}^u   \},$
 {with $\sqrt{\log(2n+2)}$ replaced by $\sqrt{\log 2+u\log(n+1)}$}.
\end{corollary}

	\subsection{General theory for constrained estimator}
\label{sec:general_constrained}
% REVISION START: direct probability approximation and growing set family
While the preliminary estimator enjoys broad generality with uniform convergence over infinite set families on arbitrary response spaces, it does not guarantee inclusion consistency: for $A \subset A'$, one may have $\widehat{F}_A^{\mathrm{pre}}(x) > \widehat{F}_{A'}^{\mathrm{pre}}(x)$. For a collection of sets encoded by a DAG, the constrained estimator in Section~\ref{sec:methodology} enforces this monotonicity and estimates all set probabilities jointly. {We next establish its statistical guarantees.

\begin{theorem}
\label{theo1}
For a node $j$, let $\mathcal A(j)$ denote its set of ancestors, let
$\overline{\mathcal A}(j)=\mathcal A(j)\cup\{j\}$, and let
$\mathcal{PR}(j)$ denote the collection of paths from a root to
$j$. Define $a_j=|\overline{\mathcal A}(j)|$ and
$\kappa_j=\max_{p\in\mathcal{PR}(j)}|p|$, where $a_j$ is the number of
nodes that may affect the $j$th output and $\kappa_j$ is the largest number
of vertices on a root-to-$j$ path.   {In particular, a root has $a_j=\kappa_j=1$.} Further, write
$\bar a_m=m^{-1}\sum_{j=1}^m a_j$,
$a_{\max,m}=\max_{j\leq m}a_j$, and
$\kappa_m=\max_{j\leq m}\kappa_j$. For an estimator function class $\mathcal F$ and
$\widetilde{\bm g}=(\widetilde g_1,\ldots,\widetilde g_m)\in\mathcal F^m$,
define
$\Delta_{n,m}(\widetilde{\bm g})
=\max_{j\leq m}\|F_{A_j}^*-H_j(\widetilde{\bm g})\|_\infty$.
Also let $\mathcal I_{j,n}(\widetilde{\bm g})
=\{H_j(\bm g)-H_j(\widetilde{\bm g}):
\bm g\in\mathcal F^m\}$. For any $\eta$ satisfying
$nm>\eta^2>C_1nm\{\Delta_{n,m}(\widetilde{\bm g})\}^2$, let
$K=\lceil\log(nm/\eta^2)/\log 2\rceil$. Then, 
\begin{align}
&\mathbb P\left\{
\sum_{j=1}^m\sum_{i=1}^n
\bigl(F_{A_j}^*(x_i)-\widehat F_{A_j}(x_i)\bigr)^2
>\eta^2
\right\} \nonumber\\
&\quad\leq
C_2\sum_{j=1}^m\sum_{k=1}^K
\frac{1}{2^{k-2}\eta^2}
\left[
1+\sqrt n
\int_{1/n}^{C_3 2^{k/2}\eta/\sqrt n}
\sqrt{\log N\!\left(
\epsilon,\mathcal I_{j,n}(\widetilde{\bm g}),\|\cdot\|_n
\right)}
\,d\epsilon
\right],
\label{eq:growing-m-entropy}
\end{align}
where $C_1,C_2,C_3$ are positive universal constants.
\end{theorem}

Theorem~\ref{theo1} provides a general tail bound for the aggregated MSE of
the constrained estimator. The quantity
$\Delta_{n,m}(\widetilde{\bm g})$ measures how well the target probability
vector can be approximated by the DAG outputs given a candidate
$\widetilde{\bm g}\in\mathcal F^m$, while the covering entropy of
$\mathcal I_{j,n}(\widetilde{\bm g})$ controls the corresponding stochastic
estimation error. Thus, the theorem applies to any base function class
$\mathcal F$ for which these approximation and entropy terms can be
controlled.

{A separate,  high-probability oracle inequality follows from the least-squares basic inequality and localized multiplier concentration see Lemma~\ref{lem:theo2-localized-multiplier}).} Suppose that $\mathcal F$ is uniformly
bounded by $B_n$ and satisfies
$\log N(\epsilon,\mathcal F,\|\cdot\|_n)
\leq v_n\log(A_n/\epsilon)$ for $0<\epsilon\leq B_n$, where $v_n\geq1$ is a complexity quantity and $A_n\geq eB_n$. Define the approximation error $\Delta_{n,m}:=\inf_{\widetilde{\bm g}\in\mathcal F^m}
\Delta_{n,m}(\widetilde{\bm g}).$ Then, for every
$\delta\in(0,1/2)$, with probability at least $1-\delta$,
\begin{equation}
\frac{1}{m}\sum_{j=1}^m
\bigl\|\widehat F_{A_j}-F_{A_j}^*\bigr\|_n^2
\lesssim 
\Delta_{n,m}^2
+\frac{\bar a_m v_n}{n}
\log\left(
\frac{eA_nnm\sqrt{\kappa_m a_{\max,m}}}{\delta}
\right).
\label{eq:growing-m-main}
\end{equation}
The bound separates the direct joint approximation error from the estimation
error. Specifically, the estimation error depends on the DAG structure, rather than on the number
of sets $m$ alone, primarily through the average ancestor size $\bar a_m$,
while $\kappa_m$ enters only logarithmically. Thus, DAGs with the same $m$ can exhibit different rates depending on their structure. Appendix~\ref{app:dag-examples} illustrates this dependence for several representative DAG structures.}

% REVISION: end superseded constrained-estimator theory

\subsection{ReLU  neural networks theory for constrained estimator}
	\label{sec:nn_theo_constrain}
% REVISION START: probability-level hierarchical-composition DNN consequence
{The general theory in Section~\ref{sec:general_constrained} applies to a neural network estimator once its approximation error and covering entropy are specified. 
In this subsection, we apply that result to the hierarchical composition class $\mathcal{H}(l,\mathcal{P})$ studied by \citet{kohler2019rate}. These functions are obtained by composing multiple smooth functions, each of which may depend on an input of dimension smaller than $d$. Its formal definition is given in Appendix~\ref{sec:hierarchical_class}.

\begin{corollary}
\label{cor1}
Suppose $F_{A_j}^*\in\mathcal H(l_j,\mathcal P_j)$ for
$j=1,\ldots,m$, where $l_j\in\mathbb N$ and
$\mathcal P_j\subset[1,\infty)\times\mathbb N$. In addition, suppose each function $\mu$ in the definition of {$F_{A_j}^*$} can have
different smoothness $p_\mu=q_\mu+u_\mu$, for
$q_\mu\in\mathbb N_0$ and $u_\mu\in(0,1]$, and possibly different
input dimension $d_\mu$, so that $(p_\mu,d_\mu)\in\mathcal P_j$.
Let $d_{\max}^j$ be the largest input dimension and $p_{\max}^j$ the
largest smoothness of any function $\mu$ associated with $g_j^*$.
We require that all partial derivatives of order less than or equal to
$q_\mu$ are uniformly bounded by a constant that does not
depend on $j$, and that each function $\mu$ is Lipschitz continuous
with Lipschitz constant $C_{\mathrm{Lip}}\geq1$. Also assume that
$
\max_{j\leq m}\max\{p_{\max}^j,d_{\max}^j\}=O(1).$ Define
\[
\phi_n
=
\max_{j\leq m}\max_{(p,d)\in\mathcal P_j}
n^{-2p/(2p+d)}.
\]
{For $(L_n,w_n,R_n)$ and the complexity quantity $v_n$ in Corollary~\ref{appx:cor1-expanded}, take $\mathcal{F}_n=\F^{\mathrm{cl}}(L_n,w_n,R_n)$, which is the clipped DNN class in Definition~\ref{def:dnn-expanded-calibrated-class}}. Then the constrained estimator
satisfies
\[
\frac1m\sum_{j=1}^m
\bigl\|\widehat F_{A_j}-F_{A_j}^*\bigr\|_n^2
=
O_{\mathbb P}\left[
\phi_n+
\frac{\bar a_m v_n}{n}
\log\left(
CR_nnm\sqrt{\kappa_m a_{\max,m}}
\right)
\right].\]
If $m$ grows no faster than a polynomial in $n$, then the rate
simplifies to $O_{\mathbb P}\left(
\phi_n+
\frac{\bar a_m v_n\log n}{n}
\right).$ Finally, for a fixed DAG, this further reduces to
$O_{\mathbb P}(\phi_n\log^4n).$
\end{corollary}

Corollary~\ref{cor1} provides an explicit convergence rate for the constrained
DNN estimator over inclusion DAGs of general topology and growing size. {Here $v_n$ measures the complexity of the clipped DNN class
$\F^{\mathrm{cl}}(L_n,w_n,R_n)$ and is controlled by its depth, width, and number of trainable parameters; see Corollary~\ref{appx:cor1-expanded} for its explicit form.} For a
fixed DAG, the rate reduces to $\phi_n$ up to logarithmic factors, matching
the standard DNN estimation rate for the hierarchical composition class
\citep{kohler2019rate}. A key ingredient in proving Corollary~\ref{cor1} is the new DNN approximation result in Appendix~\ref{sec:nn_general}, which shows that dense ReLU networks composed with calibration ReLU
networks approximating the logit and inverse-softplus functions can still be
constructed and represented as dense ReLU networks such that $\max_{j\leq m}
\bigl\|F_{A_j}^*-H_j(g_1,\ldots,g_m)\bigr\|_\infty^2
\lesssim\phi_n.$ We refer the reader to Corollary~\ref{appx:cor1-expanded} for the detailed statement of
Corollary~\ref{cor1}, and to Appendix~\ref{sec:nn_general} for the complete
DNN approximation theory. 
}

% REVISION: end superseded latent-function DNN corollary

	    \subsection{ReLU  neural networks theory for preliminary estimator}
        \label{nn_theo_pre}
	
Analogous instantiations of Theorem~\ref{theo0} and Corollary~\ref{cor2} yield explicit convergence rates for the preliminary estimator, defined in Theorem \ref{theo0},  when DNNs are used under hierarchical composition.

%{given the limited space, maybe we can move Corollary  \ref{eqn:cor3} to the appendix and only briefly mention it here. and then present  corollary \ref{cor4} }

%\red{Carlos: [Inconsistent indexing: using both $A \in \mathcal{A}$ and $A \in \{1,\dots,m\}$. Same in corollary 5]}
     \begin{corollary}
         \label{eqn:cor3}
         {Consider the notation and assumptions from Theorem~\ref{theo0}, and that for all $A\in\Lambda$ it holds that $F_A^*\in\mathcal H(l_A,\mathcal P_A)$ satisfies conditions analogous to those imposed in Corollary~\ref{cor1}.} Also, assume that $\underset{A \in  \Lambda }{{\sup}  } \max\{p_{\max}^{A},d_{\max}^{A} \}    \,=\,  O(1).$
		%$\max\{p_{\max},K_{\max} \} =O(1)$.  
		Let  $	\phi_{n} =    \underset{A \in  \Lambda  }{ {\sup}  } \max_{(p, d) \in \mathcal P_{A} } n^{\frac{-2p}{ (2p+d)}}.$ % = \max_{\left(p, K \right) \in \mathcal P  }\left(n\bar{m}\right)^{\frac{-2p}{\left(2p+K\right)}}
		If $L$ and $w$ are appropriately chosen (see Corollary~\ref{appx:cor3} for a formal statement), then, taking the deep ReLU networks class $\mathcal{F} = \{f \in  \mathcal{F}(L,w)\,:\,  \|f\|_{\infty}\leq B \}$ for a sufficiently large constant $B = O(1)$,    we have
		\[
	   \underset{A  \in \Lambda}{\sup} \,   \frac{1}{n}\sum_{i=1}^n  ( F_{A}^*(x_i)  -    \widehat{F}_{A}^{ \text{pre} }(x_i)     )^2 \,=\, O_{\mathbb{P}} \left(  \phi_n \log^4 n \right).
		\]
       \end{corollary}
As a direct application of Theorem~\ref{theo0}, Corollary~\ref{eqn:cor3} extends the advantage of DNNs to the preliminary estimator: under the hierarchical composition assumption, it achieves the uniform convergence rate $\phi_n \log^4 n$ over the totally ordered family $\Lambda$, which may contain uncountably many sets.

Finally, returning to the Euclidean response setting with rectangles, we further instantiate Corollary~\ref{cor2}:

\begin{corollary}
\label{cor4}
    Suppose that the conditions and notations of Corollary \ref{eqn:cor3} hold and that $\mathcal{Y} = \mathbb{R}^u $. Then letting $\Lambda$ be the set of rectangles as in Corollary \ref{cor2}, we have that 
    	\[
	   \underset{A  \in \Lambda }{\sup} \,   \frac{1}{n}\sum_{i=1}^n  ( F_{A}^*(x_i)  -    \widehat{F}_{A}^{ \text{pre} }(x_i)     )^2 \,=\, O_{\mathbb{P}} \left( \frac{u \log n}{n} \,+\,   \phi_n \log^4 n \right).
		\]
\end{corollary}

Corollary~\ref{cor4} shows that in the multivariate Euclidean setting, the preliminary estimator satisfies a two-term uniform error bound over the lower-orthant family. The term $u\log n/n$ reflects the cost of controlling the estimator uniformly over a $u$-dimensional lower-orthant class, while $\phi_n \log^4 n$ is the rate induced by the dense ReLU approximation. Hence, the overall rate is determined by the slower of these two terms. In particular, when the response dimension $u$ is fixed, the term $\phi_n \log^4 n$ continues to govern the main convergence behavior, reflecting the advantage of DNNs to overcome the curse of dimensionality. More generally, the same conclusion remains valid even when $u$ is allowed to grow with $n$, as long as the dimension-dependent term $u\log n/n$ is dominated by $\phi_n \log^4 n$.

{Taken together, Corollaries~\ref{cor1},~\ref{eqn:cor3}, and~\ref{cor4}
demonstrate that our general theory accommodates DNNs in distributional regression with
arbitrary response variables, yielding explicit rates for both estimators under
hierarchical composition of the target probabilities. We provide a comparison with other recent theoretical results in Table~\ref{tab:comparison}.}

    \section{Numerical experiments}
    \label{sec:num-experiment}
    We consider {eight} simulation scenarios with various response types (Euclidean, manifold-valued, functional, and network responses) and different DAG structures, along with one special scenario based on real data with a simplified toy task. DNNs are used as estimators: the preliminary estimator (\textbf{Pre-NN}), which trains a separate DNN for each set independently, and the constrained estimator (\textbf{Constrain-NN}), which jointly trains all DNNs using a loss function induced by the DAG describing the sets. We also compare with a naive baseline that adds constraints to Pre-NN post hoc, where we clip Pre-NN's outputs (\textbf{Pre-NN-Clip}) according to the DAG such that if a child's predicted probability is smaller than some of its parents', we raise it to the maximum among its parents to satisfy monotonicity. As additional baselines, Distributional Random Forests (\textbf{DRF},~\citet{cevid2022distributional}) perform distributional regression over the sets via random forests with inclusion-consistent outputs, and \textbf{Engression}~\citep{shen2025engression} fits a generative DNN of the conditional distribution and estimates set probabilities by averaging the set indicator over sampled responses. 
    
The sets for each scenario are predefined with a specific inclusion structure. Except for the special scenario, we generate $n_{\text{train}} \in \{1000, 2000, 4000\}$ training samples per scenario, and are repeated over 30 replications. The performance is evaluated using aggregated MSE across all sets on a held-out test set of size $n_{\text{test}} = 10000$, averaged over rounds. The true conditional probabilities $\mathbb{P}(Y \in A_j \mid X = x)$ serve as the estimation target and are computed via Monte Carlo integration with $10000$ samples per data point at test time. Due to space constraints, we briefly outline the two challenging scenarios (S1, S2) and the special scenario (S3) here, and refer to Appendix~\ref{app:additional_scenarios} for the full details of all scenarios and the implementation settings for all methods.

%\textbf{S1 (Swiss roll manifold).} To consider a response supported on a nonlinear manifold, let $X\sim\mathrm{Unif}([0,2]^{50})$ and generate $X_{\mathrm{noisy}}=X+\epsilon$, where $\epsilon\sim\mathcal{N}(0,\sigma^2(X)I_{50})$ with $\sigma^2(X)=1+0.25\|X\|^2/50$. We map $X_{\mathrm{noisy}}$ to latent parameters $t\in[1.5\pi,4.5\pi]$ and $h\in[0,1]$ through two randomly parameterized sigmoid functions, and set $Y=(t\cos t,h,t\sin t)\in\mathbb{R}^3$. We construct $m=20$ sets using five marginal quantile levels of $t$ and four marginal quantile levels of $h$, with $A_{(l_t,l_h)}=\{(t\cos t,h,t\sin t):t\le q_t^{(l_t)},\,h\le q_h^{(l_h)}\}$, forming a $5\times4$ inclusion grid under componentwise ordering; see Appendix~\ref{app:additional_scenarios} for the full construction.

\textbf{S1 (Swiss roll manifold).} We consider a response on a nonlinear manifold. Let $X\sim\mathrm{Unif}([0,2]^{50})$ and generate $X_{\mathrm{noisy}}=X+\epsilon$, where $\epsilon\sim\mathcal{N}(0,\sigma^2(X)I_{50})$ with $\sigma^2(X)=1+0.25\|X\|^2/50$. We map $X_{\mathrm{noisy}}$ to latent coordinates
$t=1.5\pi+3\pi\tau(w_t^\top X_{\mathrm{noisy}}+b_t)$ and
$h=\tau(w_h^\top X_{\mathrm{noisy}}+b_h)$, where
$\tau(u)=(1+e^{-u})^{-1}$, the entries of
$w_t,w_h\in\mathbb{R}^{50}$ are drawn i.i.d.\ from
$\mathcal{N}(0,0.25)$, and $b_t,b_h\sim\mathcal{N}(0,0.01)$.
Thus, $t\in[1.5\pi,4.5\pi]$ and $h\in[0,1]$, and we set our response $Y=(t\cos t,h,t\sin t)$. We construct $m=20$ sets using nested marginal quantile levels of the latent coordinates $t$ and $h$, forming a $5\times4$ inclusion grid; see Appendix~\ref{app:additional_scenarios} for design details.

%\textbf{S1 (Gaussian mixture).} Let $X \sim \mathrm{Unif}([0,2]^{30})$. We generate $Y \in \mathbb{R}^5$ from a mixture $Y \mid X \sim \sum_{k=1}^{4} \pi_k(X)\, \mathcal{N}(\mu_k(X),\, \sigma_k^2(X)\, I_5),$ where component means follow $\mu_k(X) = \Theta_{k,1} X + \Theta_{k,2} X^{\odot 2}$ with $\Theta_{k,1} \in \mathbb{R}^{5 \times 30}$ having entries from $\mathrm{Unif}([-1,1])$ and $\Theta_{k,2} \in \mathbb{R}^{5 \times 30}$ from $\mathrm{Unif}([-0.5, 0.5])$, and $X^{\odot 2}$ denotes element-wise square of $X$. The mixing weights are $\pi(X) = \mathrm{softmax}(W_\pi X + b_\pi)$ with $W_\pi \in \mathbb{R}^{4 \times 30}$. The variance is heteroscedastic: $\sigma_k^2(X) = \sigma_{k,0}^2 + \sigma_{k,1}^2 \|\mu_k(X)\|^2$, where $\sigma_{k,0} \sim \mathrm{Unif}([1, 2])$ and $\sigma_{k,1} \sim \mathrm{Unif}([0.3, 0.5])$. We construct $m = 19$ sets as axis-aligned hyperrectangles using nested marginal quantile levels of each coordinate of $Y$; see Appendix~\ref{app:additional_scenarios} for design details.

\textbf{S2 (Random dot product graph).} To illustrate that our framework applies to arbitrary response spaces, we let $Y \in \{0,1\}^{N \times N}$ be the adjacency matrix with $N = 50$ nodes generated by a random dot product graph (RDPG,~\citet{athreya2018statistical}), where each node $i$ carries a unit-length embedding vector $z_i \in \mathbb{S}^{d-1}$ with $d = 5$, and the probability of an edge between nodes $i$ and $j$ is $\tfrac{1}{2}(\langle z_i, z_j \rangle + 1)$. The RDPG is endowed with a community structure of $K = 3$ communities, where both the community assignment probabilities $\pi_k(X)$ and the unit-length community embedding centers $a_k(X) \in \mathbb{S}^{d-1}$ are non-linearly parametrized by the covariate $X \sim \mathrm{Unif}([0,2]^{50})$ for $k\in\{1,2,3\}$. Each node $i$ is then independently assigned $c_i \sim \mathrm{Cat}(\pi(X))$, with embedding $z_i = (a_{c_i}(X) + \varepsilon_i) / \|a_{c_i}(X) + \varepsilon_i\|$ where $\varepsilon_i \sim \mathcal{N}(0, 0.3^2 I_d)$, yielding the random graph $Y$. We estimate the distribution of the network statistic transitivity $T(Y)$~\citep{graham2015methods,graham2016homophily}, a standard measure of global clustering in network analysis, with $m = 10$ nested sets $A_j = \{Y : T(Y) \leq q_{\alpha_j}\}$ over 10 equally spaced quantile levels $\alpha_j$ from the 5th to the 95th percentile of the empirical marginal distribution of $T(Y)$.

\textbf{S3 (Pixel density estimation on single noisy natural image).}
We consider a setting where $X \in [0,1]^2$ denotes a pixel coordinate in an image, and the Euclidean response $Y = (Y_R, Y_G, Y_B) \in [0,1]^3$ is a noisy observation of the underlying clean pixel value $J(X) = (J_R, J_G, J_B)$ at coordinate $X$. Our goal is to estimate the conditional distribution of $Y \mid X$ for a single given noisy image. The noisy response is generated by Poisson-Gaussian noise model~\citep{foi2008practical} with $Y_c = \operatorname{Poisson}(50 \cdot J_c(X))/50 + \varepsilon_c$ and $\varepsilon_c \sim \mathcal{N}(0, 0.05^2)$ i.i.d. across channels $c \in \{R,G,B\}$. The target sets are 13 hyperrectangles defined by marginal quantiles of $Y$. To ground the simulation in realistic images, we use the STL-10 dataset~\citep{coates2011analysis} as the source of clean images $J$, which contains $96 \times 96$ color photographs across 10 object categories (see Figure~\ref{fig:stl10} for illustration). For each image, we run an independent experiment: pixels on a stride-2 grid form the training set ($2304$ pixels) and the remaining $6912$ pixels form the test set, preserving the spatial distribution of pixels across splits. Results are averaged over 30 images randomly selected per class.

Tables~\ref{tab:all_scenarios} and~\ref{tab:stl10_results} report the performance across these simulation scenarios.\footnote{We remark that for S2, \textbf{Engression} cannot learn a direct generative model of the high-dimensional RDPG response within a reasonable time; we therefore adapt it to learn the network statistics directly.} Specifically, \textbf{Constrain-NN} outperforms the competing methods in most settings, while \textbf{Pre-NN} also improves as $n$ increases. This pattern is further confirmed by Table~\ref{apptab:all_scenarios} with the additional scenarios reported in Appendix~\ref{app:additional_scenarios}, supporting the effectiveness of our proposed estimators.

\begin{table}
\centering
\scriptsize
\caption{Aggregated MSE ($\times10^2$) for S1 and S2 averaged over 30 trials. Best per column in bold.}
\label{tab:all_scenarios}
\begin{tabular}{lcccccc}
\toprule
& \multicolumn{3}{c}{S1 (Swiss roll manifold)} & \multicolumn{3}{c}{S2 (RDPG)} \\
\cmidrule(lr){2-4} \cmidrule(lr){5-7}
Method & $n{=}1000$ & $n{=}2000$ & $n{=}4000$ & $n{=}1000$ & $n{=}2000$ & $n{=}4000$ \\
\midrule
Pre-NN        & 3.07 (0.15) & 1.91 (0.10) & 1.04 (0.05) & 4.64 (0.10) & 3.72 (0.07) & 3.09 (0.06) \\
Constrain-NN  & \textbf{0.96 (0.07)} & \textbf{0.72 (0.04)} & \textbf{0.57 (0.04)} & \textbf{3.48 (0.11)} & \textbf{3.05 (0.07)} & \textbf{2.81 (0.06)} \\
Pre-NN-Clip   & 2.49 (0.16) & 1.56 (0.11) & 0.88 (0.05) & 4.39 (0.10) & 3.55 (0.07) & 2.99 (0.06) \\
DRF           & 1.28 (0.13) & 1.22 (0.12) & 1.16 (0.11) & 7.13 (1.46) & 6.94 (1.43) & 6.30 (1.21) \\
Engression    & 4.93 (2.55) & 4.66 (2.44) & 4.36 (2.39) & 4.20 (1.35) & 3.95 (1.34) & 3.83 (1.32) \\
\bottomrule
\end{tabular}
\end{table}

\begin{table}
\centering
\caption{Aggregated MSE ($\times 10^2$) for S3 across STL-10 classes, averaged over 30 images per class. The ``Overall'' reports the average across all classes.  See full results with standard deviations in Table~\ref{tab:stl10_results_std}.}
\label{tab:stl10_results}
\scriptsize
\begin{tabular}{lccccccccccc}
\toprule
Method / Category & Overall & Airplane & Bird & Car & Cat & Deer & Dog & Horse & Monkey & Ship & Truck \\
\midrule
Pre-NN       & 3.63 & 2.68 & 3.75 & 3.70 & 3.70 & 4.17 & 3.72 & 4.50 & 3.42 & 2.83 & 3.79 \\
Constrain-NN & \textbf{3.37} & \textbf{2.62} & \textbf{3.52} & \textbf{3.45} & \textbf{3.50} & \textbf{3.99} & \textbf{3.56} & \textbf{4.10} & \textbf{3.18} & \textbf{2.35} & \textbf{3.38} \\
Pre-NN-Clip  & 3.61 & 2.66 & 3.73 & 3.69 & 3.69 & 4.15 & 3.71 & 4.49 & 3.41 & 2.86 & 3.76 \\
DRF          & 4.26 & 3.26 & 4.68 & 4.24 & 4.41 & 4.87 & 4.55 & 5.37 & 4.34 & 2.73 & 4.14 \\
Engression &  7.96& 6.09 & 7.91 & 8.91 & 7.47 & 7.97 & 7.92 & 9.23 & 8.12 & 7.08 & 8.92 \\
\bottomrule
\end{tabular}
\end{table}

    \section{Real data experiments}
    \label{sec:real-experiment}
    We conduct two real-data experiments: the Chicago crime dataset analyzed in \citet{tansey2018maximum, padilla2025risk} and the Beijing Multi-Site Air Quality dataset~\citep{zhang2017cautionary}. In this section, we present partial results for the more complicated Beijing air quality dataset and refer to Appendix~\ref{app:real-experiment} for the complete results across both experiments.

	\textbf{Experiments on the Beijing Multi-Site Air Quality dataset.} We evaluate on a multivariate distributional air quality forecasting task using the Beijing Multi-Site Air Quality dataset, which records hourly measurements of six pollutants (PM2.5, PM10, SO$_2$, NO$_2$, CO, O$_3$) and six meteorological covariates (temperature, pressure, etc.) at 12 monitoring stations in Beijing from March 2013 to February 2017. Under international air quality standards, the overall Air Quality Index (AQI) is determined by the single pollutant with the highest individual pollution grade. Consequently, many forecasting work focuses on predicting a scalar AQI directly \citep{janarthanan2021deep, sarkar2022air}. However, the same overall AQI grade can arise from very different pollutant profiles; an AQI of ``moderate pollution'' may be driven by elevated PM2.5 alone, by NO$_2$ alone, or by multiple pollutants exceeding thresholds simultaneously. Understanding which pollutants are likely to be responsible for the pollution, is essential for targeted public health interventions. We therefore reformulate the task: let $Y = (Y_1, \ldots, Y_p)$ denote the
levels of $p$ pollutants, and let $q_{j,\ell}$ denote the threshold for
pollutant $j$ at pollution grade $\ell$, chosen with reference to Chinese air
quality standards. We predict the joint event
$\{Y_1 \geq q_{1,\ell_1}, \ldots, Y_p \geq q_{p,\ell_p}\}$, in which all
pollutants exceed their respective thresholds $q_{j,\ell_j}$ for some
specified grades $(\ell_1, \ldots, \ell_p)$. After preprocessing, this yields 76 sets whose inclusion relation induces a DAG. The prediction is conducted independently at each of the 12 monitoring stations; for each station, the covariate vector $x_s$ at time~$s$ consists of time-cycle encodings and lagged observations of all meteorological and pollutant variables over a window of $W$ consecutive hours ending $h$ hours before time~$s$. To systematically evaluate the forecasting performance, we vary the forecast horizon $h \in \{3, 6, 9\}$ and the input window length $W \in \{3, 6, 9\}$, yielding 9 experimental configurations with different difficulties. All models are trained independently per station and tested on a three-month holdout each year, with consecutive observations within each held-out month, yielding 75/25 train/test split overall. We refer the reader to Appendix~\ref{app:beijing} for a detailed configuration, including the procedure used to define the 76 sets, the input variables, and the configuration of all estimators. {More general guidelines for constructing set families are provided in Appendix~\ref{app:set-guideline}.}

Table~\ref{tab:beijing-mse-grid} reports the MSE averaged first over the 76 sets within each station, then across all 12 stations. The \textbf{Constrain-NN} achieves the best performance in almost all configurations. Several findings emerge. First, longer forecast horizons $h$ yield consistently higher MSE, as prediction difficulty grows with the temporal gap between the available observations and the target time. Second, for a fixed $h$, the MSE of the proposed DNN-based estimators decreases as the input window length $W$ increases, consistent with the intuition that richer historical context aids prediction; \textbf{DRF}, however, does not benefit consistently from this, potentially due to the curse of dimensionality. Third, \textbf{Constrain-NN} outperforms \textbf{Pre-NN}, which in turn outperforms \textbf{Pre-NN-Clip}, suggesting that incorporating the DAG structure improves accuracy while post-hoc clipping may substantially degrade performance in practice. These findings also hold at the per-station level (see Table~\ref{tab:beijing-mse-per-station} in Appendix~\ref{app:beijing}).
\begin{table}
\footnotesize
\centering
\caption{MSE ($\times 10^{2}$) across 76 sets and 12 stations. Each cell reports results for \textbf{Pre-NN} / \textbf{Constrain-NN} / \textbf{Pre-NN-Clip} / \textbf{DRF} / \textbf{Engression}. The best performance is bolded. Full results with standard deviations are reported in Table~\ref{tab:beijing-mse-expanded}.}
\label{tab:beijing-mse-grid}
\renewcommand{\arraystretch}{1.3}
\begin{tabular}{c ccc}
\toprule

 & $W=3$ & $W=6$ & $W=9$ \\
\midrule
$h=3$
& 3.43 / \textbf{2.28} / 3.54 / 2.29 / 3.79
& 3.41 / \textbf{2.24} / 3.51 / 2.33 / 3.94
& 3.35 / \textbf{2.23} / 3.46 / 2.36 / 3.92 \\

$h=6$
& 3.78 / \textbf{2.84} / 3.88 / 2.90 / 4.12
& 3.71 / \textbf{2.81} / 3.79 / 2.91 / 4.24
& 3.65 / \textbf{2.79} / 3.73 / 2.92 / 4.18 \\

$h=9$
& 3.97 / \textbf{3.15} / 4.06 / 3.21 / 4.31
& 3.87 / \textbf{3.10} / 3.95 / 3.20 / 4.40
& 3.79 / \textbf{3.09} / 3.87 / 3.17 / 4.30 \\
\bottomrule
\end{tabular}
\end{table}

\section{Conclusion}
\label{sec:conclusion}

This paper develops a unified framework for distributional regression with arbitrary response variables with two complementary estimators. The preliminary estimator provides a simple baseline based on separate regression of event indicators and yields uniform empirical error bounds, accommodating set families with infinitely many sets. The constrained estimator incorporates inclusion relations for the sets through a DAG and enforces inclusion-consistent probability predictions, for which we derive upper bounds on the averaged MSE that allow the number of sets to grow with the sample size. {We further instantiate the general theory with DNNs and other estimators, obtaining explicit rates. Several directions remain for future work. In particular, we do not derive lower bounds for DNNs, as such bounds for the hierarchical function classes are available only in specific cases; see Remark 2 of~\citet{kohler2019rate}. In addition, the framework requires the target sets to be specified in advance. Appendix~\ref{app:set-guideline} provides practical guidelines for their construction, while identifying scientifically meaningful sets may itself be nontrivial in some applications and constitutes an interesting future direction.}

\subsection*{AI use statement}

Generative AI tools were used to assist with literature review and language polishing. All research ideas, methodology, and theoretical contributions were independently developed by the authors, with AI used only for subsequent verification of specific theoretical results. The authors take full responsibility for the final work.

\bibliographystyle{plainnat}
\bibliography{reference}

\appendix
\etocsetlocaltop.toc{part}
{%
  \etocsettocstyle{\section*{Appendix}}{}%
  \etocsetnexttocdepth{subsection}%
  \localtableofcontents
}

\section{Neural Network Preliminaries}
\subsection{Dense ReLU networks}
\label{sec:relu}
We introduce the class of deep ReLU neural networks, which serve as the base function class $\mathcal{F}$ in both estimators \eqref{eqn:preliminary_estimator} and \eqref{eqn:proposed_estimator1} for deriving explicit convergence rates, although our general theory applies to a broad class of nonparametric estimators. Let $\rho(x) = \max\{0, x\}$ denote the ReLU activation, applied componentwise to vectors. {For integers $L\geq1$ and $w\geq1$, let $\mathcal F(L,w)$ be the class of functions
\[
f=\Phi_{L+1}\circ\rho\circ\Phi_L\circ\cdots\circ\rho\circ\Phi_1,
\]
where $\Phi_s(z)=\mathbf W_s z+\mathbf b_s$ is affine,
$w_0=d$, $w_1=\cdots=w_L=w$, and $w_{L+1}=1$; hence the network has $L$ hidden ReLU layers of common width $w$ and a scalar affine output layer.}

Following \citet{kohler2019rate}, we use the bounded dense-network class $\mathcal{F} = \{f \in \mathcal{F}(L, w) : \|f\|_\infty \leq B\}$ for a sufficiently large constant $B > 0$. {For the constrained estimator, a fixed latent bound is not assumed in Corollaries~\ref{cor1} and~\ref{appx:cor1-expanded}; those results use the clipped class with the growing bound $R_n$.}

\subsection{Hierarchical composition models}
\label{sec:hierarchical_class}

Neural networks have been shown in Kohler \& Langer (2021) to enjoy structural advantages over classical nonparametric methods {when low-dimensional hierarchical composition is present}. To this end, we impose a hierarchical composition structure on the target probability functions. {Before} defining the class, we recall smoothness.

   %Berfore defining the hierarchical composition class, we recall the notion of $(p,C)$-smooth function.

    \begin{definition}[$(p, C)$-smoothness]\label{definition: p,c smoothness}
\label{def:pcsmooth}
Let $p=q+s$ for some $q \in \mathbb{N}=\mathbb{Z}^{+}\cup\left\{0\right\}$ and $0<s \leq 1$. We say that a  function $g: \mathbb{R}^d \rightarrow \mathbb{R}$ is called $(p, C)$-smooth, if for every $\alpha=\left(\alpha_1, \ldots, \alpha_d\right) \in \mathbb{N}^d$,  with $d \in \mathbb{Z}^{+}$, where $\sum_{j=1}^d \alpha_j=q$ the partial derivative $\partial^q g /\left(\partial a_1^{\alpha_1} \ldots \partial a_d^{\alpha_d}\right)$ exists and satisfies
\[
\left|\frac{\partial^q g}{\partial a_1^{\alpha_1} \ldots \partial a_d^{\alpha_d}}\left(a\right)-\frac{\partial^q g}{\partial a_1^{\alpha_1} \ldots \partial a_d^{\alpha_d}}\left(b\right)\right| \leq C\| a- b\|^s
\]
for all $a, b \in \mathbb{R}^d$.% where $\|\cdot\|$ denotes the Euclidean norm.	
\end{definition}

We are now ready to define the hierarchical composition function class.

	\begin{definition} [Space of hierarchical composition models \citep{kohler2019rate}]\label{definition: hierarchical composition model}
		For $l=1$ and  smoothness constraint $\mathcal{P} \subseteq(0, \infty) \times \mathbb{N}$ the  space of hierarchical composition models is defined as
		\[
		\setlength{\arraycolsep}{1pt}
		\begin{array}{lll}
			\mathcal{H}(1, \mathcal{P})  & := &\left\{h: \mathbb{R}^d \rightarrow \mathbb{R}: h(a )=\mu\left(a_{\pi(1)}, \ldots, a_{\pi(K)}\right),\right. \\
			&&  \,\,\text { where }   \mu: \mathbb{R}^K \rightarrow \mathbb{R} \text { is }\,(p, C) \text {-smooth for some } \\
			& & \,\,(p, K) \in \mathcal{P} \text { and } \pi:\{1, \ldots, K\} \rightarrow\{1, \ldots, d\}\} .
		\end{array}
		\]
		For $l>1$, we recursively define
		\[
		\setlength{\arraycolsep}{1pt}
		\begin{array}{lll}
			\mathcal{H}(l, \mathcal{P})    &:=&  \left\{h: \mathbb{R}^d \rightarrow \mathbb{R}: h(a)=\mu\left(f_1(a), \ldots, f_K(a)\right),\right.\\
			& &  \,\,\,\text {where }  \mu: \mathbb{R}^K \rightarrow \mathbb{R} \text { is } (p, C) \text {-smooth for some }\\
			&& \,\,\,(p, K) \in \mathcal{P}\text { and } \left.f_i \in \mathcal{H}(l-1, \mathcal{P})\right\}.
		\end{array}
		%\begin{aligned}
		%& 
		%& g: \mathbb{R}^K \rightarrow \mathbb{R} \text { is }(p, C) \text {-smooth for some }(p, K) \in \mathcal{P} \\
		%&\text { and } \left.f_i \in \mathcal{H}(l-1, \mathcal{P})\right\}
		%\end{aligned}
		\]
	\end{definition}

\section{Proofs for the General Theory}
    \subsection{Proof of Theorem \ref{theo0} }
    \label{sec:theo0}

      \begin{proof}
We proceed as in the proof of Theorem 3 in  \citet{padilla2025risk}. Notice that for any 
$A \in \Lambda$ we have that 
\[
 \|F_A- H_A\|^2_n \,\leq  \,  {4U^2}
\]
for all $F_A, H_A \in \mathcal{F}_A$.  It follows that  that $\sup_{A\in \Lambda  }\|\widehat{F}_A^{\text{pre}}-G_A\|^2_n \le {4U^2}$. {Use $J=\lceil\log_2(4nU^2/\eta^2)\rceil$ and radius $r_j=2^{j/2}\eta/\sqrt n$ as in the theorem.
All probabilities and expectations below are conditional on the design, with the measurability conventions of Section~\ref{sec:notation}.}

%^{ \text{pre} }
Next, to control the probability of large deviations, we focus on 
\[
\left\{\sup_{A \in  \Lambda}\|\widehat{F}_A^{ \text{pre} }-G_A\|^2_n> \eta^2/n\right\}
\cap
\left\{\sup_{A \in \Lambda}\|\widehat{F}_A^{ \text{pre} }-G_A\|^2_n\le {4U^2}  \right\},
\]
which confines the deviations to the range {$(\eta^2/n,4U^2]$}.  Proceeding with  a peeling argument, we obtain that 
\[
\begin{array}{l}
\mathbb{P}\left(      \underset{A\in \mathcal{A}}{\sup}\,   \|  \widehat{F}_A^{ \text{pre} } - G_A \|_n   \,>\,\frac{\eta}{ \sqrt{n} }  \,\bigg|\, \{x_i\}_{i=1}^n   \right)  \\
= \displaystyle  \mathbb{P}\left(      \underset{A\in \Lambda  }{\sup}\,   \|  \widehat{F}_A^{ \text{pre} } - G_A \|^2_n   \,>\,\frac{\eta ^2 }{n},\,\,\underset{A\in \Lambda}{\sup}\,   \|  \widehat{F}_A^{ \text{pre} } - G_A \|^2_n \leq  {4U^2}    \,\bigg|\, \{x_i\}_{i=1}^n \right)  \\
 \leq \displaystyle \sum_{j=1}^J   \mathbb{P}\left(      \underset{A\in \Lambda }{\sup}\,   \|  \widehat{F}_A^{ \text{pre} } - G_A \|^2_n   \,>\, { \frac{2^{j-1}\eta ^2}{n} },\,\,\underset{A\in \Lambda}{\sup}\,   \|  \widehat{F}_A^{ \text{pre} } - G_A \|^2_n \leq  {\frac{2^j \eta^2}{n}}     \,\bigg|\, \{x_i\}_{i=1}^n \right)  \\
  \leq\displaystyle \sum_{j=1}^J   \mathbb{P}\bigg(      \underset{A\in \Lambda}{\sup}\,  2\sum_{i=1}^{n} ( {G_A(x_i)}-   \mathbbm{1}_{  \{ y_i\in A \} }  )(G_A(x_i)-\widehat{F}_A^{ \text{pre} }(x_i) )    \,>\,2^{j-1}\eta ^2 ,\,\,\\
   \,\,\,\,\,\,\,\,\,\,\,\,\,\,\,\,\,\,\,\,\,\, \underset{A\in \Lambda}{\sup}\,   \|  \widehat{F}_A - G_A \|^2_n \leq 2^j \eta^2/n   \,\bigg|\, \{x_i\}_{i=1}^n  \bigg)  \\
    \leq\displaystyle \sum_{j=1}^J    \mathbb{P}\bigg(  \underset{A\in \Lambda }{\sup}  \,\,\underset{F_A \in \mathcal{F}_A \,:\,  \|  F_A- G_A\|^2_n\leq  2^{j }\eta^2/n   }{\sup} \, \sum_{i=1}^{n} (G_A(x_i)-   \mathbbm{1}_{  \{ y_i\in A \} }  )(G_A(x_i)-F_A(x_i) )    \\
     \,\,\,\,\,\,\,\,\,\,\,\,\,>\,2^{j-2}\eta ^2  \,\bigg|\, \{x_i\}_{i=1}^n  \bigg),
\end{array}
\]
where the first inequality follows from the union bound and the second from the basic inequality. {Indeed, expanding the least-squares inequality against $G_A$ gives}
\begin{equation}
{n\|\widehat F_A^{\mathrm{pre}}-G_A\|_n^2
\le2\sum_{i=1}^n\big(G_A(x_i)-\1_{\{y_i\in A\}}\big)
\big(G_A(x_i)-\widehat F_A^{\mathrm{pre}}(x_i)\big).}
\tag{R4}\label{eq:prebasic}
\end{equation}
Furthermore, 
\begin{equation}
    \label{eqn:8}
      \begin{array}{l}
\displaystyle  \mathbb{P}\bigg(  \underset{A\in  \Lambda   }{\sup}  \,\,\underset{F_A \in \mathcal{F}_A \,:\,  \|  F_A- G_A\|^2_n\leq  2^{j }\eta^2/n   }{\sup} \, \sum_{i=1}^{n} (G_A(x_i)-   \mathbbm{1}_{  \{ y_i\in A \} }  )(G_A(x_i)-F_A(x_i) )  \,>\,  2^{j-2} \eta^2   \,\bigg|\, \{x_i\}_{i=1}^n\bigg)\\
  \leq  	\displaystyle     \frac{1}{2^{j-2} \eta^2}\mathbb{E}\left(  \underset{A\in \Lambda}{\sup}  \,\,\underset{F_A \in \mathcal{F}_A \,:\,  \|  F_A- G_A\|^2/n\leq  2^{j }\eta^2/n   }{\sup} \, \sum_{i=1}^{n} (G_A(x_i)-   \mathbbm{1}_{  \{ y_i \in A\} }  )(G_A(x_i)-F_A(x_i) )  \,\bigg|\, \{x_i\}_{i=1}^n  \right)\\
		 \leq\displaystyle     \frac{1}{2^{j-2} \eta^2}\mathbb{E}\left(  \underset{A\in \Lambda}{\sup}  \,\,\underset{F_A \in \mathcal{F}_A \,:\,  \|  F_A- G_A\|^2_n\leq  2^{j }\eta^2/n   }{\sup} \, \sum_{i=1}^{n} (F_A^*(x_i)-   \mathbbm{1}_{  \{ y_i\in A \} }  )(G_A(x_i)-F_A(x_i) )   \,\bigg|\, \{x_i\}_{i=1}^n \right)\\
	\displaystyle     \,+\,  \frac{1}{2^{j-2} \eta^2} \underset{A\in \Lambda }{\sup}  \,\,\underset{F_A \in \mathcal{F}_A \,:\,  \|  F_A- G_A\|^2_n\leq  2^{j }\eta^2/n   }{\sup} \, \sum_{i=1}^{n}  (  G_A(x_i)- F_A^*(x_i)  )(G_A(x_i)-F_A(x_i)  \,\bigg|\, \{x_i\}_{i=1}^n)\\
	\leq\displaystyle     \frac{1}{2^{j-2} \eta^2}\mathbb{E}\left(  \underset{A\in \Lambda}{\sup}  \,\,\underset{F_A \in \mathcal{F}_A \,:\,  \|  F_A- G_A\|^2_n\leq  2^{j }\eta^2/n   }{\sup} \, \sum_{i=1}^{n} (F_A^*(x_i)-   \mathbbm{1}_{  \{ y_i\in A \} }  )(G_A(x_i)-F_A(x_i) )   \,\bigg|\, \{x_i\}_{i=1}^n \right)\\
	\displaystyle     \,+\,  \frac{4\sqrt{n}   }{2^{j/2} \eta} \,\, \underset{A\in \Lambda}{\sup}\,  \| G_A-F^*_A \|_{\infty} 
\end{array}
\end{equation}
where the first inequality follows from Markov's inequality,   the last inequality holds by Cauchy–Schwarz inequality, and  $C>0$ is a constant.  Moreover, for $\xi_1,\ldots,\xi_n$ independent Rademacher random variables independent $\{(x_i, y_i)\}_{i=1}^n$, we have, by Symmetrization, that 
\begin{equation}
    \label{eqn:9}
     \begin{array}{l}
 \displaystyle 	\mathbb{E}\left(  \underset{A\in \Lambda}{\sup}  \,\,\underset{F_A \in \mathcal{F}_A \,:\,  \|  F_A- G_A\|_n^2\leq  2^{j }\eta^2/n   }{\sup} \, \sum_{i=1}^{n} (F_A^*(x_i)-   \mathbbm{1}_{  \{ y_i\in A \} }  )(G_A(x_i)-F_A(x_i) )  \,\bigg|\, \{x_i\}_{i=1}^n  \right)\\
  \displaystyle \leq  2   \mathbb{E}\left(  \underset{A\in \Lambda}{\sup}  \,\,\underset{F_A \in \mathcal{F}_A \,:\,  \|  F_A- G_A\|_n^2\leq  2^{j }\eta^2/n   }{\sup} \, \sum_{i=1}^{n} \xi_i \mathbbm{1}_{  \{ y_i\in A \} }  (G_A(x_i)-F_A(x_i) )  \,\bigg|\, \{x_i\}_{i=1}^n  \right)\\
    \displaystyle = 2   \mathbb{E}\left(\mathbb{E}\left(  \underset{A\in \Lambda}{\sup}  \,\,\underset{F_A \in \mathcal{F}_A \,:\,  \|  F_A- G_A\|_n^2\leq  2^{j }\eta^2/n   }{\sup} \, \sum_{i=1}^{n} \xi_i \mathbbm{1}_{  \{ y_i\in A \} }  (G_A(x_i)-F_A(x_i) )  \,\bigg|\, \{x_i\}_{i=1}^n ,  \{y_i\}_{i=1}^n \right) \bigg| \{x_i\}_{i=1}^n \right).\\
 \end{array}
\end{equation}
Next notice that for  fixed $\{(x_i,y_i)\}_{i=1}^n$,  {the vectors
\[
\mathcal V_n(\Lambda)=
\{(\1_{\{y_i\in A\}})_{i=1}^n:A\in\Lambda\}
\]
form a chain in $\{0,1\}^n$ under coordinatewise order, by Assumption~\ref{as0}.
Distinct vectors in such a chain have distinct sums of coordinates, each in $\{0,\ldots,n\}$.
Choose a finite collection $\Lambda'\subseteq\Lambda$ with one representative for each distinct sample vector. Then}
\begin{equation}
{|\Lambda'|=|\mathcal V_n(\Lambda)|\le n+1.}\label{eq:maskchain}
\end{equation}  
{To use this finite collection, first enlarge the admissible differences to
$\mathcal D(r_j)=(\mathcal F-\mathcal F)\cap B_{r_j}(0)$.
Indeed, $G_A-F_A\in\mathcal D(r_j)$ whenever
$\|F_A-G_A\|_n^2\le2^j\eta^2/n$, so
\begin{equation}
\begin{aligned}
&\sup_{A\in\Lambda}\sup_{\substack{F_A\in\mathcal F_A:\
\|F_A-G_A\|_n^2\le2^j\eta^2/n}}
\sum_{i=1}^n\xi_i\1_{\{y_i\in A\}}
\big(G_A(x_i)-F_A(x_i)\big)\\
&\le\sup_{A\in\Lambda}\sup_{h\in\mathcal D(r_j)}
\sum_{i=1}^n\xi_i\1_{\{y_i\in A\}}h(x_i)
=\max_{A\in\Lambda'}\sup_{h\in\mathcal D(r_j)}
\sum_{i=1}^n\xi_i\1_{\{y_i\in A\}}h(x_i).
\end{aligned}\tag{145}\label{eqn:10}
\end{equation}}

{The first step is an inequality, not an identity: $G_A$ and $\mathcal F_A$ may differ for sets
with the same sample vector. }

Therefore, by~\citet[Lemma 3]{padilla2025risk}, it follows that
\begin{equation}
\begin{aligned}
&\E\Bigg[\sup_{A\in\Lambda}
\sup_{\substack{F_A\in\mathcal F_A:\
\|F_A-G_A\|_n^2\le2^j\eta^2/n}}
\sum_{i=1}^n\xi_i\1_{\{y_i\in A\}}
\big(G_A(x_i)-F_A(x_i)\big)
\,\bigg|\,\{x_i\}_{i=1}^n,\{y_i\}_{i=1}^n\Bigg]\\
&{\le}\E\Bigg[{\max_{A\in\Lambda'}\sup_{h\in\mathcal D(r_j)}}
\sum_{i=1}^n\xi_i\1_{\{y_i\in A\}}{h(x_i)}
\,\bigg|\,\{x_i\}_{i=1}^n,\{y_i\}_{i=1}^n\Bigg]\\
&\le C\int_0^{{2^{j/2}\eta}}
\sqrt{\log N\big({\varepsilon/\sqrt n,\mathcal D(r_j),\|\cdot\|_n}\big)}\,d\varepsilon
+C{2^{j/2}}\eta\sqrt{{\log(2|\Lambda'|)}}\\
&{\le C\sqrt n\,\mathcal J(r_j)
+C2^{j/2}\eta\sqrt{\log(2n+2)}.}
\end{aligned}\tag{146}\label{eqn:main_bound}
\end{equation}
for a positive constant $C$.
{Here $(h(x_i))_{i=1}^n$ has Euclidean norm at most $2^{j/2}\eta$.
For each $A\in\Lambda'$, the multipliers $(\xi_i\1_{\{y_i\in A\}})_{i=1}^n$
have independent, centered, $1$-sub-Gaussian coordinates; independence between different $A$'s is not required.
The covering-number form of the lemma follows by the packing--covering comparison at a fixed numerical scale.}

Combining {equation~\eqref{eqn:main_bound}} with
\eqref{eqn:8}, \eqref{eqn:9}, and \eqref{eqn:10},
{and using $\sum_{j=1}^J2^{-j/2}\le\sum_{j\ge1}2^{-j/2}<\infty$, gives}
\[
{
\begin{aligned}
&\Prob\left(\sup_{A\in\Lambda}\|\widehat F_A^{\mathrm{pre}}-G_A\|_n>
\frac{\eta}{\sqrt n}\,\bigg|\,\{x_i\}_{i=1}^n\right)\\
&\quad\le C\left\{
\frac{\sqrt n}{\eta^2}\sum_{j=1}^J2^{-j}\mathcal J(r_j)
+\frac{\sqrt{\log(2n+2)}}{\eta}
+\frac{\sqrt n}{\eta}\sup_{A\in\Lambda}\|G_A-F_A^*\|_\infty
\right\}.
\end{aligned}}
\]
{Finally,
$\sup_{A\in\Lambda}\|\widehat F_A^{\mathrm{pre}}-F_A^*\|_n
\le\sup_{A\in\Lambda}\|\widehat F_A^{\mathrm{pre}}-G_A\|_n+b$
by the triangle inequality, which proves the desired claim}

    \end{proof}

\subsection{Proof of Corollary \ref{cor2}}

\begin{proof}

{ All steps of Appendix~\ref{sec:theo0} remain unchanged except for the
cardinality bound on $\mathcal V_n(\Lambda)$. For each coordinate $k$,
the threshold class
\[
\left\{
\left(\mathbbm{1}_{\{Y_{1k}\le t\}},\ldots,
\mathbbm{1}_{\{Y_{nk}\le t\}}\right)
:\ t\in\mathbb R
\right\}
\]
induces at most $n+1$ distinct binary vectors on the sample, corresponding
to the $n+1$ possible threshold positions relative to the ordered values
$Y_{1k},\ldots,Y_{nk}$, including the all-zero vector obtained when
$t<\min_i Y_{ik}$. For $t=(t_1,\ldots,t_u)\in\mathbb R^u$, the sample indicator vector
associated with the lower orthant $A(t)$ satisfies
\[
\left(\mathbbm{1}_{\{Y_i\in A(t)\}}\right)_{i=1}^n
=
\bigodot_{k=1}^u
\left(\mathbbm{1}_{\{Y_{ik}\le t_k\}}\right)_{i=1}^n,
\]
where $\bigodot$ denotes coordinatewise multiplication. Consequently,
\begin{equation}
|\mathcal V_n(\Lambda)|\le (n+1)^u.
\tag{R9}\label{eq:maskorthant}
\end{equation}
Ties only decrease this count. Use equation~\eqref{eq:maskorthant} in the penultimate line of equation~\eqref{eqn:main_bound} to replace $\log(2n+2)$ by $\log2+u\log(n+1)$. This proves the desired claim.}

\end{proof}

{\subsection{Proof of Theorem~\ref{theo1}}}

{

\begin{proof}
Condition on the design, let
\[
\varepsilon_{i,j}=z_i^j-F_{A_j}^*(x_i),
\qquad
R(\bm g)
=
\sum_{j=1}^m\sum_{i=1}^n
\bigl\{F_{A_j}^*(x_i)-H_j(\bm g)(x_i)\bigr\}^2,
\]
and write $\widehat{\bm g}=(\widehat g_1,\ldots,\widehat g_m)$.
The least-squares basic inequality gives
\begin{align}
R(\widehat{\bm g})
&\leq
R(\widetilde{\bm g})
+
2\sum_{j=1}^m\sum_{i=1}^n
\varepsilon_{i,j}
\left\{
H_j(\widehat{\bm g})(x_i)
-
H_j(\widetilde{\bm g})(x_i)
\right\}\nonumber\\
&\leq
nm\{\Delta_{n,m}(\widetilde{\bm g})\}^2
+
2\sum_{j=1}^m\sum_{i=1}^n
\varepsilon_{i,j}
\left\{
H_j(\widehat{\bm g})(x_i)
-
H_j(\widetilde{\bm g})(x_i)
\right\}.
\label{eq:theo2-basic-inequality}
\end{align}

Since $R(\widehat{\bm g})\leq nm$, define
\[
\mathcal S_k
=
\left\{
2^{k-1}\eta^2
<
R(\widehat{\bm g})
\leq
2^k\eta^2
\right\},
\qquad k=1,\ldots,K.
\]
Then
\[
\{R(\widehat{\bm g})>\eta^2\}
\subseteq
\bigcup_{k=1}^K\mathcal S_k.
\]
Choose $C_1\geq2$. On $\mathcal S_k$, the condition
$nm\{\Delta_{n,m}(\widetilde{\bm g})\}^2<C_1^{-1}\eta^2$ and
\eqref{eq:theo2-basic-inequality} imply
\begin{align*}
&2\sum_{j=1}^m\sum_{i=1}^n
\varepsilon_{i,j}
\left\{
H_j(\widehat{\bm g})(x_i)
-
H_j(\widetilde{\bm g})(x_i)
\right\}\\
&\qquad\geq
R(\widehat{\bm g})
-
nm\{\Delta_{n,m}(\widetilde{\bm g})\}^2\\
&\qquad>
2^{k-1}\eta^2-C_1^{-1}\eta^2
\geq
2^{k-2}\eta^2.
\end{align*}
Moreover, for every $j$,
\begin{align*}
\|H_j(\widehat{\bm g})-H_j(\widetilde{\bm g})\|_n
&\leq
\|H_j(\widehat{\bm g})-F_{A_j}^*\|_n
+
\|H_j(\widetilde{\bm g})-F_{A_j}^*\|_n\\
&\leq
\frac{2^{k/2}\eta}{\sqrt n}
+
\Delta_{n,m}(\widetilde{\bm g})\\
&\leq
\frac{2^{k/2+1}\eta}{\sqrt n}.
\end{align*}
Therefore, by Markov's inequality,
\begin{align*}
\mathbb P(\mathcal S_k)
&\leq
\frac{C}{2^k\eta^2}
\sum_{j=1}^m
\mathbb E \left\{
\sup_{\substack{
f\in\mathcal I_{j,n}(\widetilde{\bm g})\\
\|f\|_n\leq 2^{k/2+1}\eta/\sqrt n}}
\left|
\sum_{i=1}^n\varepsilon_{i,j}f(x_i)
\right| \right\}.
\end{align*}

Let $\xi_1,\ldots,\xi_n$ be independent Rademacher variables. Conditional
symmetrization and contraction give
\begin{align*}
&\mathbb E
\left\{\sup_{\substack{
f\in\mathcal I_{j,n}(\widetilde{\bm g})\\
\|f\|_n\leq 2^{k/2+1}\eta/\sqrt n}}
\left|
\sum_{i=1}^n\varepsilon_{i,j}f(x_i)
\right|\right\}\\
&\qquad\leq
C\mathbb E_\xi
\left\{\sup_{\substack{
f\in\mathcal I_{j,n}(\widetilde{\bm g})\\
\|f\|_n\leq 2^{k/2+1}\eta/\sqrt n}}
\left|
\sum_{i=1}^n\xi_i f(x_i)
\right|\right\}\\
&\qquad\leq
C\left[
1+\sqrt n
\int_{1/n}^{C_3 2^{k/2}\eta/\sqrt n}
\sqrt{
\log N\bigl(
\epsilon,
\mathcal I_{j,n}(\widetilde{\bm g}),
\|\cdot\|_n
\bigr)}
\,d\epsilon
\right],
\end{align*}
where the last inequality follows from Dudley's entropy bound. Hence,
\begin{align*}
\mathbb P\{R(\widehat{\bm g})>\eta^2\}
&\leq
\sum_{k=1}^K\mathbb P(\mathcal S_k)\\
&\leq
C_2\sum_{j=1}^m\sum_{k=1}^K
\frac{1}{2^{k-2}\eta^2}
\left[
1+\sqrt n
\int_{1/n}^{C_3 2^{k/2}\eta/\sqrt n}
\sqrt{
\log N\bigl(
\epsilon,
\mathcal I_{j,n}(\widetilde{\bm g}),
\|\cdot\|_n
\bigr)}
\,d\epsilon
\right],
\end{align*}
which proves~\eqref{eq:growing-m-entropy}.
\end{proof}
}

{\subsection{Auxiliary results for Theorem~\ref{theo1}}}

{
The first lemma records how perturbations of the base functions propagate
through the DAG. It will also be used to relate the covering entropy of the
base class to that of each constrained output.

\begin{lemma}[Pathwise representation and stability]
\label{lem:theo2-pathwise}
For any $\bm g=(g_1,\ldots,g_m)$,
\begin{equation}
\widetilde H_j(\bm g)(x)
=
\max_{p=(i_1,\ldots,i_\ell)\in\mathcal{PR}(j)}
\left\{
g_{i_1}(x)+\sum_{r=2}^{\ell}\log\bigl(1+e^{g_{i_r}(x)}\bigr)
\right\}.
\label{eq:theo2-path-representation}
\end{equation}
Consequently, for any $\bm g,\bm h\in\mathcal F^m$,
\begin{align}
|H_j(\bm g)(x)-H_j(\bm h)(x)|
&\leq
\frac14\max_{p\in\mathcal{PR}(j)}
\sum_{k\in p}|g_k(x)-h_k(x)|,
\label{eq:theo2-pathwise-pointwise}\\
\|H_j(\bm g)-H_j(\bm h)\|_n^2
&\leq
\frac{\kappa_j}{16}
\sum_{k\in\overline{\mathcal A}(j)}
\|g_k-h_k\|_n^2.
\label{eq:theo2-pathwise-empirical}
\end{align}
\end{lemma}

\begin{proof}
Equation~\eqref{eq:theo2-path-representation} follows directly by recursively
expanding the definition of $\widetilde H_j$ along all root-to-$j$ paths.
Since the maximum map and the softplus function are one-Lipschitz,
\begin{align*}
&\left|
\widetilde H_j(\bm g)(x)-\widetilde H_j(\bm h)(x)
\right|\\
&\quad\leq
\max_{p=(i_1,\ldots,i_\ell)\in\mathcal{PR}(j)}
\left\{
|g_{i_1}(x)-h_{i_1}(x)|
+\sum_{r=2}^{\ell}
\left|
\log\bigl(1+e^{g_{i_r}(x)}\bigr)
-\log\bigl(1+e^{h_{i_r}(x)}\bigr)
\right|
\right\}\\
&\quad\leq
\max_{p\in\mathcal{PR}(j)}
\sum_{k\in p}|g_k(x)-h_k(x)|.
\end{align*}
The $1/4$-Lipschitz property of the sigmoid therefore gives
\begin{align*}
|H_j(\bm g)(x)-H_j(\bm h)(x)|
&\leq
\frac14
\left|
\widetilde H_j(\bm g)(x)-\widetilde H_j(\bm h)(x)
\right|\\
&\leq
\frac14\max_{p\in\mathcal{PR}(j)}
\sum_{k\in p}|g_k(x)-h_k(x)|,
\end{align*}
which proves~\eqref{eq:theo2-pathwise-pointwise}. Moreover, by Cauchy--Schwarz, $|p|\leq\kappa_j$, and $p\subseteq\overline{\mathcal A}(j)$,
\begin{align*}
\|H_j(\bm g)-H_j(\bm h)\|_n^2
&\leq
\frac{1}{16n}\sum_{i=1}^n
\max_{p\in\mathcal{PR}(j)}
\left\{
\sum_{k\in p}|g_k(x_i)-h_k(x_i)|
\right\}^2\\
&\leq
\frac{1}{16n}\sum_{i=1}^n
\max_{p\in\mathcal{PR}(j)}
|p|\sum_{k\in p}|g_k(x_i)-h_k(x_i)|^2\\
&\leq
\frac{\kappa_j}{16n}
\sum_{i=1}^n
\sum_{k\in\overline{\mathcal A}(j)}
|g_k(x_i)-h_k(x_i)|^2\\
&=
\frac{\kappa_j}{16}
\sum_{k\in\overline{\mathcal A}(j)}
\|g_k-h_k\|_n^2,
\end{align*}
which proves~\eqref{eq:theo2-pathwise-empirical}.
\end{proof}

\begin{lemma}[Covering entropy of a constrained output]
\label{lem:theo2-entropy-transfer}
For every $\widetilde{\bm g}\in\mathcal F^m$,
\begin{equation}
N\bigl(\epsilon,\mathcal I_{j,n}(\widetilde{\bm g}),\|\cdot\|_n\bigr)
\leq
\left[
N\left(
\frac{4\epsilon}{\sqrt{\kappa_j a_j}},
\mathcal F,\|\cdot\|_n
\right)
\right]^{a_j}.
\label{eq:theo2-entropy-transfer}
\end{equation}
Under the covering condition following Theorem~\ref{theo1}, this implies
\begin{equation}
\log N\bigl(\epsilon,\mathcal I_{j,n}(\widetilde{\bm g}),
\|\cdot\|_n\bigr)
\leq
a_jv_n\log\left(
\frac{C A_n\sqrt{\kappa_j a_j}}{\epsilon}
\right)
\label{eq:theo2-output-entropy}
\end{equation}
for a universal constant $C$.
\end{lemma}

\begin{proof}
Let $\{f_1,\ldots,f_M\}$ be a
$4\epsilon/\sqrt{\kappa_j a_j}$-cover of $\mathcal F$ under
$\|\cdot\|_n$, where
\[
M=
N\left(
\frac{4\epsilon}{\sqrt{\kappa_j a_j}},
\mathcal F,\|\cdot\|_n
\right).
\]
For each $k\in\overline{\mathcal A}(j)$, choose $h_k$ from this cover such
that
\[
\|g_k-h_k\|_n
\leq
\frac{4\epsilon}{\sqrt{\kappa_j a_j}}.
\]
By~\eqref{eq:theo2-pathwise-empirical},
\begin{align*}
\|H_j(\bm g)-H_j(\bm h)\|_n^2
&\leq
\frac{\kappa_j}{16}
\sum_{k\in\overline{\mathcal A}(j)}
\|g_k-h_k\|_n^2\\
&\leq
\frac{\kappa_j a_j}{16}
\left(
\frac{4\epsilon}{\sqrt{\kappa_j a_j}}
\right)^2
=\epsilon^2.
\end{align*}
There are at most $M^{a_j}$ such coordinatewise choices. Since translation
by $H_j(\widetilde{\bm g})$ does not change covering numbers,
\[
N\bigl(\epsilon,\mathcal I_{j,n}(\widetilde{\bm g}),
\|\cdot\|_n\bigr)
\leq M^{a_j},
\]
which proves~\eqref{eq:theo2-entropy-transfer}. Under the stated covering
condition,
\begin{align*}
\log N\bigl(\epsilon,\mathcal I_{j,n}(\widetilde{\bm g}),
\|\cdot\|_n\bigr)
&\leq
a_j\log N\left(
\frac{4\epsilon}{\sqrt{\kappa_j a_j}},
\mathcal F,\|\cdot\|_n
\right)\\
&\leq
a_jv_n\log\left(
\frac{C A_n\sqrt{\kappa_j a_j}}{\epsilon}
\right),
\end{align*}
proving~\eqref{eq:theo2-output-entropy}.
\end{proof}

The following localized bound is used to derive the explicit rate following
Theorem~\ref{theo1}. 

\begin{lemma}[Localized multiplier bound]
\label{lem:theo2-localized-multiplier}
Suppose that
\[
\log N(u,\mathcal F,\|\cdot\|_n)
\leq
v_n\log(A_n/u),
\qquad 0<u\leq B_n,
\]
where $v_n\geq1$ and $A_n\geq eB_n$. Fix
$\widetilde{\bm g}\in\mathcal F^m$ and let
\[
\varepsilon_{i,j}=z_i^j-F_{A_j}^*(x_i).
\]
For every $\delta\in(0,1/2)$, conditionally on the design, with probability
at least $1-\delta$, simultaneously for all $j\leq m$ and
$f\in\mathcal I_{j,n}(\widetilde{\bm g})$,
\begin{equation}
\left|
\frac1n\sum_{i=1}^n\varepsilon_{i,j}f(x_i)
\right|
\leq
\frac1{16}\|f\|_n^2
+
C\frac{a_jv_n}{n}
\log\left(
\frac{eA_nnm\sqrt{\kappa_m a_{\max,m}}}{\delta}
\right).
\label{eq:theo2-localized-multiplier}
\end{equation}
\end{lemma}

\begin{proof}
Condition on $x_1,\ldots,x_n$ and, for $r\in[1/n,1]$, define
\[
Z_j(r)
=
\sup_{\substack{
f\in\mathcal I_{j,n}(\widetilde{\bm g})\\
\|f\|_n\leq r}}
\left|
\frac1n\sum_{i=1}^n\varepsilon_{i,j}f(x_i)
\right|.
\]
Since $\mathbb E(\varepsilon_{i,j}\mid x_i)=0$ and
$|\varepsilon_{i,j}|\leq1$, symmetrization, contraction,
Dudley's inequality, and~\eqref{eq:theo2-output-entropy} give
\begin{align}
\mathbb E\{Z_j(r)\}
&\leq
\frac Cn
+
\frac C{\sqrt n}
\int_{1/n}^{r}
\sqrt{
a_jv_n
\log\left(
\frac{CA_n\sqrt{\kappa_j a_j}}{u}
\right)}
\,du \nonumber\\
&\leq
\frac Cn
+
Cr\sqrt{
\frac{a_jv_n}{n}
\log\left(
eA_nn\sqrt{\kappa_j a_j}
\right)}.
\label{eq:theo2-local-expectation}
\end{align}
Because $\mathcal I_{j,n}(\widetilde{\bm g})$ has envelope one, Bousquet's
inequality gives, for every $t\geq1$,
\begin{equation}
\mathbb P\left[
Z_j(r)>
C\left\{
\mathbb E Z_j(r)
+r\sqrt{\frac tn}
+\frac tn
\right\}
\right]
\leq e^{-t}.
\label{eq:theo2-bousquet}
\end{equation}
{In the usual form, the square-root term also contains $\E Z_j(r)$; Young's inequality absorbs that term into $C\E Z_j(r)+Ct/n$. No independence across event indices $j$ is assumed.}

Let
\[
S=\lceil\log_2n\rceil,\qquad
r_s=2^{-s},\qquad
t_s=\log\left\{\frac{Cm(s+2)^2}{\delta}\right\},
\qquad s=0,\ldots,S,
\]
and define
\[
q_j
=
a_jv_n
\log\left(
\frac{eA_nnm\sqrt{\kappa_m a_{\max,m}}}{\delta}
\right).
\]
Then, uniformly over $j\leq m$ and $s\leq S$,
\[
a_jv_n
\log\left(eA_nn\sqrt{\kappa_j a_j}\right)+t_s
\leq Cq_j,
\]
and
\[
\sum_{j=1}^m\sum_{s=0}^{S}e^{-t_s}
\leq\delta.
\]
Therefore, by~\eqref{eq:theo2-local-expectation},
\eqref{eq:theo2-bousquet}, and a union bound, with probability at least
$1-\delta$,
\begin{equation}
Z_j(r_s)
\leq
C\left\{
r_s\sqrt{\frac{q_j}{n}}
+\frac{q_j}{n}
\right\},
\qquad j\leq m,\quad s\leq S.
\label{eq:theo2-dyadic-bound}
\end{equation}

For any $f\in\mathcal I_{j,n}(\widetilde{\bm g})$ with
$\|f\|_n>1/n$, choose $s$ such that
$r_s/2<\|f\|_n\leq r_s$. Then~\eqref{eq:theo2-dyadic-bound} gives
\begin{align*}
\left|
\frac1n\sum_{i=1}^n\varepsilon_{i,j}f(x_i)
\right|
&\leq Z_j(r_s)\\
&\leq
C\left\{
\|f\|_n\sqrt{\frac{q_j}{n}}
+\frac{q_j}{n}
\right\}.
\end{align*}
If $\|f\|_n\leq1/n$, then
\[
\left|
\frac1n\sum_{i=1}^n\varepsilon_{i,j}f(x_i)
\right|
\leq
\|f\|_n
\leq
\frac1n
\leq
\frac{q_j}{n}.
\]
Thus, uniformly over $j$ and $f$,
\[
\left|
\frac1n\sum_{i=1}^n\varepsilon_{i,j}f(x_i)
\right|
\leq
C\left\{
\|f\|_n\sqrt{\frac{q_j}{n}}
+\frac{q_j}{n}
\right\}.
\]
Finally, Young's inequality,
\[
C\|f\|_n\sqrt{\frac{q_j}{n}}
\leq
\frac1{16}\|f\|_n^2+C\frac{q_j}{n},
\]
proves~\eqref{eq:theo2-localized-multiplier}.
\end{proof}

{
\subsection{Proof of the explicit bound~\eqref{eq:growing-m-main}}
}

{
\begin{proof}
{This proof starts from the basic inequality (\ref{eq:theo2-basic-inequality}) and Lemma~\ref{lem:theo2-localized-multiplier}.}
Fix $\widetilde{\bm g}\in\mathcal F^m$ and write
\[
\mathcal E
=
\sum_{j=1}^m
\|H_j(\widehat{\bm g})-F_{A_j}^*\|_n^2,
\qquad
\mathcal A
=
\sum_{j=1}^m
\|H_j(\widetilde{\bm g})-F_{A_j}^*\|_n^2.
\]
By the definition of $\Delta_{n,m}(\widetilde{\bm g})$,
\begin{equation}
\mathcal A
\leq
m\{\Delta_{n,m}(\widetilde{\bm g})\}^2.
\label{eq:theo2-explicit-approximation}
\end{equation}
Let
\[
f_j
=
H_j(\widehat{\bm g})-H_j(\widetilde{\bm g})
\in
\mathcal I_{j,n}(\widetilde{\bm g}).
\]
Dividing~\eqref{eq:theo2-basic-inequality} by $n$ and applying
Lemma~\ref{lem:theo2-localized-multiplier} give, with probability at least
$1-\delta$,
\begin{align}
\mathcal E
&\leq
\mathcal A
+
2\sum_{j=1}^m
\left|
\frac1n\sum_{i=1}^n
\varepsilon_{i,j}f_j(x_i)
\right|\nonumber\\
&\leq
\mathcal A
+
\frac18\sum_{j=1}^m\|f_j\|_n^2
+
C\frac{v_n}{n}\sum_{j=1}^m a_j
\log\left(
\frac{eA_nnm\sqrt{\kappa_m a_{\max,m}}}{\delta}
\right).
\label{eq:theo2-explicit-basic}
\end{align}
Moreover,
\begin{align}
\sum_{j=1}^m\|f_j\|_n^2
&\leq
2\sum_{j=1}^m
\|H_j(\widehat{\bm g})-F_{A_j}^*\|_n^2
+
2\sum_{j=1}^m
\|H_j(\widetilde{\bm g})-F_{A_j}^*\|_n^2\nonumber\\
&=
2\mathcal E+2\mathcal A.
\label{eq:theo2-explicit-difference}
\end{align}
Substituting~\eqref{eq:theo2-explicit-difference} into
\eqref{eq:theo2-explicit-basic} yields
\begin{align*}
\mathcal E
&\leq
\mathcal A+\frac14\mathcal E+\frac14\mathcal A
+
C\frac{v_n}{n}\sum_{j=1}^m a_j
\log\left(
\frac{eA_nnm\sqrt{\kappa_m a_{\max,m}}}{\delta}
\right),
\end{align*}
and hence, by~\eqref{eq:theo2-explicit-approximation},
\[
\mathcal E
\leq
C\left\{
m\{\Delta_{n,m}(\widetilde{\bm g})\}^2
+
\frac{v_n}{n}\sum_{j=1}^m a_j
\log\left(
\frac{eA_nnm\sqrt{\kappa_m a_{\max,m}}}{\delta}
\right)
\right\}.
\]
Dividing by $m$ and using
$m^{-1}\sum_{j=1}^m a_j=\bar a_m$ gives
\[
\frac1m\sum_{j=1}^m
\|\widehat F_{A_j}-F_{A_j}^*\|_n^2
\leq
C\left\{
\{\Delta_{n,m}(\widetilde{\bm g})\}^2
+
\frac{\bar a_m v_n}{n}
\log\left(
\frac{eA_nnm\sqrt{\kappa_m a_{\max,m}}}{\delta}
\right)
\right\}.
\]
For every $\gamma>0$, choose
$\widetilde{\bm g}_\gamma\in\mathcal F^m$ satisfying
\[
\Delta_{n,m}(\widetilde{\bm g}_\gamma)
\leq
\Delta_{n,m}+\gamma.
\]
Letting $\gamma\downarrow0$ proves~\eqref{eq:growing-m-main}.
\end{proof}
}}

\section{DNN Theory for the Preliminary Estimator}

    \begin{corollary}
         \label{appx:cor3}
         \textbf{[Detailed statement of Corollary~\ref{eqn:cor3}].} Consider the notation and assumptions from Theorem \ref{theo0}, and  that  for all $A \in \Lambda$ it holds that 
		$F_{A}^* \in  \mathcal{H}(l_A, \mathcal{P}_A) $ for some $ l_A \in \mathbb{N}$ and $\mathcal{P}_A\subset [1,\infty ) \times \mathbb{N}$. In addition, suppose  that each function $\mu$ in the definition of $F_{A}^*$ can have different smoothness $p_{\mu} =  q_{\mu} +u_{\mu}$, for $q_{\mu} \in \mathbb{N} $, $u_{\mu} \in (0,1]$,  of perhaps  different input dimension $d_{\mu} $, so that $(p_{\mu}, d_{\mu} ) \in \mathcal{P}_{A}$. Let $d_{\max}^{A}$ be the largest input dimension and $p_{\max}^{A}$ the largest smoothness of any function $\mu$ associated with $F_{A}^*$. We require  that all the partial derivatives of order less than or equal to $q_{\mu}$ are uniformly bounded by constant that does not depend on $A$, and each function $\mu$ is Lipschitz continuous
		with Lipschitz constant $C_{\mathrm{Lip} } \geq1 $. Also, assume that $  { \underset{A \in \Lambda }  {\sup } } \max\{p_{\max}^{A},d_{\max}^{A} \}    \,=\,  O(1).$
		%$\max\{p_{\max},K_{\max} \} =O(1)$.  
		Let  
		\[
		\phi_{n} =    { \underset{A \in \Lambda }  {\sup } }  \max_{(p, d) \in \mathcal P_{A} } n^{\frac{-2p}{ (2p+d)}}. % = \max_{\left(p, K \right) \in \mathcal P  }\left(n\bar{m}\right)^{\frac{-2p}{\left(2p+K\right)}}
		\]
		Then there exists choices of $L$ and $w$  such that 
		for 	sufficiently large positive constants $c_3$ and $c_4$  if 
		\begin{equation}
			\label{eqn:choice1_v2}
			L  =\lceil c_3 \log n \rceil\,\,\,\,\,\,\text{and}\,\,\,\,\,\, w = \left\lceil c_4   { \underset{A \in \Lambda }  {\sup } } \max_{(p, d) \in \mathcal P_{A}  } n^{\frac{d}{ 2(2p+d)}}\right\rceil
		\end{equation}
		or 
		\begin{equation}
			\label{eqn:choice2_v2}
			L  =\left\lceil c_3   { \underset{A \in \Lambda }  {\sup } } \max_{(p, d) \in \mathcal P_{A}  } n^{\frac{d}{ 2(2p+d)}} \log n\right\rceil\,\,\,\,\,\,\text{and}\,\,\,\,\,\, w = \left\lceil c_4  \right\rceil,
		\end{equation}
        	then, taking $\mathcal{F} = \{f \in  \mathcal{F}(L,w)\,:\,  \|f\|_{\infty}\leq B \}$ for a sufficiently large constant $B = O(1)$,    we obtain that 
		\[
	   \underset{A  \in \Lambda }{\sup} \,   \frac{1}{n}\sum_{i=1}^n  ( F_{A}^*(x_i)  -    \widehat{F}_{A}^{ \text{pre} }(x_i)     )^2 \,=\, O_{\mathbb{P}} \left(  \phi_n \log^4 n \right).
		\]
       \end{corollary}
\subsection{Proof of  Corollary \ref{eqn:cor3} and \ref{appx:cor3} }

	\begin{proof}

  {We give all approximation, entropy, integration, and probability-balancing steps.}

 {\textbf{Uniform approximation.} The uniform hierarchy assumptions allow one common pair of sufficiently large architecture constants in equations~\eqref{eqn:choice1_v2} or~\eqref{eqn:choice2_v2}. Kohler \& Langer's hierarchical approximation construction then gives, for every $A$, an admissible network $G_A$ with
\[
b_n:=\sup_{A\in\Lambda}\|G_A-F_A^*\|_\infty\le C\sqrt{\phi_n}.
\]
The constants depend on the common hierarchy bounds and fixed input dimension, not on $A$. }

{\textbf{Pseudodimension and covering entropy.} Fix $B\geq 2$ as a constant independent of $n$ and $A\in\Lambda$, and let
\[
\mathcal F
=
\left\{
f\in\mathcal F(L,w):\|f\|_\infty\leq B
\right\}.
\]
Since $0\leq F_A^*\leq 1$ for every $A\in\Lambda$ and the uniform
approximation error tends to zero, $B$ can be chosen sufficiently large
so that the approximating networks considered below belong to
$\mathcal F$ for all sufficiently large $n$.}

{A fully connected ReLU network in $\mathcal F(L,w)$ with input
dimension $d$, $L$ hidden layers of width $w$, and one scalar output
has
\[
W=(d+1)w+(L-1)(w+1)w+(w+1)
\]
trainable parameters and
\[
U_{\rm net}=Lw+1
\]
computational units. Let $C_{\rm pd}>0$ denote the universal constant
in the pseudo-dimension bound for fully connected ReLU networks, and define
\[
v_n
:=
\max\left\{
1,\,
C_{\rm pd}LW\log\!\bigl(eU_{\rm net}\bigr)
\right\}.
\]
By the same pseudo-dimension and bounded VC-subgraph argument as in
Lemma~\ref{lem:dnn-expanded-calibrated-entropy}, there exists a universal
constant $C_{\rm ent}\geq 1$, independent of
$n,L,w,B,u$, such that
\[
\log N\!\left(u,\F,\|\cdot\|_n\right)
\leq
C_{\rm ent}v_n
\log\!\left(\frac{C_{\rm ent}B}{u}\right),
\qquad 0<u\leq B.
\]
Moreover, the product of two $(u/2)$-covers of $\F$ induces a
$u$-cover of $\F-\F$. Since $\mathcal D(r)\subseteq\F-\F$,
\[
\begin{aligned}
\log N\!\left(u,\mathcal D(r),\|\cdot\|_n\right)
&\leq
2\log N\!\left(\frac{u}{2},\F,\|\cdot\|_n\right)\\
&\leq
2C_{\rm ent}v_n
\log\!\left(\frac{2C_{\rm ent}B}{u}\right),
\qquad
0<u\leq r\leq2\sqrt{2}\,B.
\end{aligned}
\]}

{\textbf{The integral.} For $A_0=CB$ with $A_0/r\ge e$, substitute $u=rt$:
\[
\begin{split}
\int_0^r\sqrt{\log(A_0/u)}\,du
&=r\int_0^1\sqrt{\log(A_0/r)+\log(1/t)}\,dt\\
&\le r\left\{\sqrt{\log(A_0/r)}+\int_0^1\sqrt{\log(1/t)}\,dt\right\}\\
&=r\{\sqrt{\log(A_0/r)}+\Gamma(3/2)\}
\le C r\sqrt{\log(eA_0/r)}.
\end{split}
\]}
{For Theorem~\ref{theo0}, $\eta>1$ and every shell has $r_j\ge n^{-1/2}$, while $r_J\le2\sqrt2 B$. Thus
\begin{equation}
\mathcal J(r_j)\le C r_j\sqrt{v_n\log(en)},\qquad
\frac{\sqrt n}{\eta^2}\sum_{j=1}^J2^{-j}\mathcal J(r_j)
\le\frac{C\sqrt{v_n\log(en)}}\eta.
\tag{R10}\label{eq:dnnintegral}
\end{equation}}

{\textbf{Architecture size.} Let $s_n=\sqrt{n\phi_n}=\sup_A\max_{(p,d_0)\in\mathcal P_A}n^{d_0/[2(2p+d_0)]}$. For the wide architecture, $L\asymp\log(en)$, $w\asymp s_n$, so $W\lesssim n\phi_n\log(en)$ and $v_n\lesssim n\phi_n\log^3(en)$. For the narrow architecture, $L\asymp s_n\log(en)$, $w=O(1)$, so $W\lesssim L$ and again $v_n\lesssim L^2\log(eL)\lesssim n\phi_n\log^3(en)$. Therefore
\begin{equation}
\sqrt{v_n\log(en)}+\sqrt{\log(en)}+\sqrt n\,b_n
\le C\sqrt{n\phi_n}\log^2(en).\tag{R11}\label{eq:dnnbalance}
\end{equation}}

{\textbf{Probability and rate.} Take $\eta=M\sqrt{n\phi_n}\log^2(en)$ for a fixed constant $M\ge1$. Equations~\eqref{eq:precorrect}, \eqref{eq:dnnintegral}, and~\eqref{eq:dnnbalance} give an upper-tail probability at most $C/M$, and the threshold is at most $(M+C)\sqrt{\phi_n}\log^2(en)$. For every prescribed probability tolerance, choose $M$ sufficiently large. This is precisely
\[
\sup_A\|\widehat F_A^{\rm pre}-F_A^*\|_n^2= O_{\mathbb{P}}\{\phi_n\log^4(en)\}.
\]}

	\end{proof}

\subsection{Corollary \ref{cor4}}

\begin{proof}
This follows from Corollary \ref{cor2}, using the same network-entropy calculation as in the proof of Corollary \ref{eqn:cor3} and taking
\[
		\eta\,\asymp\, \sqrt{u\log n} + \sqrt{LW \log(U) \log n}.
\]
\end{proof}
{
\section{DNN Theory for the Constrained Estimator}
       \label{sec:nn_general}
This section develops the DNN approximation theory underlying
Corollary~\ref{cor1}. We proceed from the basic definitions and the
required ReLU approximation results to the DAG-constrained network
construction and its approximation properties. With these ingredients
in place, we present the detailed statement of
Corollary~\ref{cor1} in Corollary~\ref{appx:cor1-expanded} and prove
the resulting DNN rate.

We use the DAG notation from Theorem~\ref{theo1} and the dense ReLU
class $\mathcal F(L,w)$ from Appendix~A; $L$ counts hidden layers
and $w$ is their common width. Let
\[
\ell(t)=\log\frac{t}{1-t},\quad 0<t<1,\qquad
\psi(t)=\log(e^t-1),\quad t>0,
\]
be the inverse functions of the sigmoid $\tau$ and the softplus
$\operatorname{sp}(t)=\log(1+e^t)$, respectively. For $0<s\leq1$, write
\[
\Lambda_s=\log(e/s),
\]
and, for an interval $I$, write
\[
\|f\|_{\infty,I}=\sup_{t\in I}|f(t)|.
\]

\begin{lemma}[Deep--narrow ReLU approximation of logit and inverse softplus]
\label{lem:dnn-expanded-fixed-modules}
There is a universal integer $w_{\mathrm{cal}}\geq1$ with the
following properties.

\begin{enumerate}
\item For every $0<\epsilon<1/4$ and $0<\eta<1$, there exists a scalar
ReLU network $\mathsf L_{\epsilon,\eta}$ whose weights depend only
on $(\epsilon,\eta)$ and such that
\begin{equation}
\bigl\|\mathsf L_{\epsilon,\eta}-\ell\bigr\|_{\infty,
[\epsilon,1-\epsilon]}
\leq\eta.
\label{eq:dnn-expanded-logit}
\end{equation}
Its width is at most $w_{\mathrm{cal}}$, and its number of hidden
layers, denoted by $D_{\ell}(\epsilon,\eta)$, satisfies
\begin{equation}
D_{\ell}(\epsilon,\eta)
\leq
C\Lambda_{\epsilon}\Lambda_{\eta}
(\Lambda_{\epsilon}+\Lambda_{\eta}).
\label{eq:dnn-expanded-logit-depth}
\end{equation}

\item For every $0<\delta\leq1$, $U>\delta$, and $0<\zeta<1$, there exists
a scalar ReLU network $\mathsf S_{\delta,U,\zeta}$ whose weights
depend only on $(\delta,U,\zeta)$ and such that
\begin{equation}
\bigl\|\mathsf S_{\delta,U,\zeta}-\psi\bigr\|_{\infty,
[\delta,U]}
\leq\zeta.
\label{eq:dnn-expanded-spinv}
\end{equation}
Its width is at most $w_{\mathrm{cal}}$, and its number of hidden
layers, denoted by $D_{\psi}(\delta,U,\zeta)$, satisfies
\begin{equation}
D_{\psi}(\delta,U,\zeta)
\leq
C(1+U+\Lambda_{\delta})\Lambda_{\zeta}
(1+U+\Lambda_{\delta}+\Lambda_{\zeta}).
\label{eq:dnn-expanded-spinv-depth}
\end{equation}
Moreover,
\begin{equation}
-\psi(\delta)\leq\log(1/\delta),
\qquad
\psi(U)\leq U.
\label{eq:dnn-expanded-spinv-endpoints}
\end{equation}
\end{enumerate}
All constants in this lemma are universal and independent of
$(\epsilon,\eta,\delta,U,\zeta)$.
\end{lemma}

\emph{Proof outline.}
We first show that a polynomial can be approximated by a ReLU network
of universal width by combining a fixed-width multiplication network
with Horner's recursion. We then extend this construction to continuous
piecewise polynomials by evaluating the pieces successively, so that
the number of pieces affects the depth but not the width. Finally, we
partition the domains of $\log$ and $\psi$ into intervals on which,
after rescaling to $[0,1]$, their nonlinear parts form uniformly
holomorphic families. Polynomial approximation on these intervals
then gives the stated networks.

\begin{proof}
\medskip\noindent\textbf{Polynomial evaluation.}
We begin with a fixed-width network for approximate multiplication and
use it as the basic module for evaluating a polynomial recursively. The
construction keeps the width universal, while the polynomial degree and
approximation accuracy enter through the depth.

By \citet[Lemma~4]{kohler2019rate}, for every $B\geq1$ and
$0<\omega<1$, there exists a two-input scalar ReLU network of width
$18$ and depth at most $C\log(eB^2/\omega)$,
\[
\mathsf M_{B,\omega}\in
\mathcal F\left(
C\log\frac{eB^2}{\omega},18
\right),
\]
such that
\begin{equation}
\sup_{|u|,|v|\leq B}
|\mathsf M_{B,\omega}(u,v)-uv|
\leq\omega.
\label{eq:dnn-expanded-multiplication}
\end{equation}

Then, let $q\in\mathbb N_0$ and let $b_0,\ldots,b_q\in\mathbb R$ be arbitrary.
Consider the polynomial
\[
P:[0,1]\to\mathbb R,
\qquad
P(t)=\sum_{r=0}^q b_r t^r.
\]
Define
\[
A
:=
1+\sum_{r=0}^q |b_r|.
\]

Fix a target approximation accuracy $\xi\in(0,1)$ and set
\[
\omega:=\frac{\xi}{q+1}.
\]
Then $0<\omega<1$ and $2A\geq1$, so
\eqref{eq:dnn-expanded-multiplication}, applied with $B=2A$, yields a
two-input ReLU network $\mathsf M_{2A,\omega}$ satisfying
\[
\sup_{|u|,|v|\leq 2A}
\left|
\mathsf M_{2A,\omega}(u,v)-uv
\right|
\leq \omega.
\]

For each $r\in\{0,\ldots,q\}$, define the exact Horner tail
\[
y_r:[0,1]\to\mathbb R,
\qquad
y_r(t):=\sum_{k=r}^q b_k t^{\,k-r}.
\]
Thus,
\[
y_q(t)=b_q,
\qquad
y_r(t)=b_r+t\,y_{r+1}(t),
\quad r=q-1,\ldots,0,
\]
and, in particular,
\[
y_0(t)=P(t).
\]
We define the corresponding approximate Horner iterates recursively by
\[
\widehat y_q(t):=b_q
\]
and, for $r=q-1,\ldots,0$,
\[
\widehat y_r(t)
:=
b_r+
\mathsf M_{2A,\omega}
\bigl(t,\widehat y_{r+1}(t)\bigr),
\qquad 0\leq t\leq1.
\]
Hence $\widehat y_r$ is obtained from the exact Horner recursion by
replacing each multiplication $t\,y_{r+1}(t)$ with the approximate
multiplication network $\mathsf M_{2A,\omega}$.

We next verify simultaneously that the multiplication network is always
evaluated on its prescribed domain and that the approximation errors do
not accumulate excessively.  First, for every $r\in\{0,\ldots,q\}$ and
$t\in[0,1]$,
\[
|y_r(t)|
\leq
\sum_{k=r}^q |b_k|\,t^{k-r}
\leq
\sum_{k=r}^q |b_k|
\leq
A-1.
\]
We claim that, for every $r\in\{0,\ldots,q\}$,
\begin{equation}
\sup_{t\in[0,1]}
\left|\widehat y_r(t)-y_r(t)\right|
\leq
(q-r)\omega
\label{eq:dnn-expanded-horner-error}
\end{equation}
and
\begin{equation}
\sup_{t\in[0,1]}
|\widehat y_r(t)|
\leq 2A.
\label{eq:dnn-expanded-horner-range}
\end{equation}

We prove \eqref{eq:dnn-expanded-horner-error} and
\eqref{eq:dnn-expanded-horner-range} simultaneously by backward induction
on $r$. For $r=q$,
\[
\widehat y_q(t)=y_q(t)=b_q,
\]
so the error is zero and
\[
|\widehat y_q(t)|=|b_q|\leq A-1<2A.
\]

Now let $r\in\{0,\ldots,q-1\}$ and suppose that both bounds hold at
level $r+1$. In particular,
\[
|\widehat y_{r+1}(t)|\leq2A,
\qquad t\in[0,1].
\]
Since $A\geq1$, we also have $|t|\leq1\leq2A$. Hence
\[
\bigl(t,\widehat y_{r+1}(t)\bigr)\in[-2A,2A]^2,
\]
so the approximation guarantee
\eqref{eq:dnn-expanded-multiplication} applies and gives
\[
\left|
\mathsf M_{2A,\omega}
\bigl(t,\widehat y_{r+1}(t)\bigr)
-
t\,\widehat y_{r+1}(t)
\right|
\leq\omega.
\]
Using the exact and approximate Horner recursions,
\begin{align*}
|\widehat y_r(t)-y_r(t)|
&\leq
\left|
\mathsf M_{2A,\omega}
\bigl(t,\widehat y_{r+1}(t)\bigr)
-
t\,\widehat y_{r+1}(t)
\right|\\
&\quad
+
|t|\,
|\widehat y_{r+1}(t)-y_{r+1}(t)|\\
&\leq
\omega+(q-r-1)\omega\\
&=
(q-r)\omega.
\end{align*}
Thus \eqref{eq:dnn-expanded-horner-error} holds at level $r$. Moreover,
since $q-r\leq q$ and
\[
q\omega=\frac{q}{q+1}\xi<\xi<1,
\]
the bound $|y_r(t)|\leq A-1$ yields
\[
|\widehat y_r(t)|
\leq
|y_r(t)|+|\widehat y_r(t)-y_r(t)|
\leq
(A-1)+\xi
<
A
\leq2A.
\]
Hence \eqref{eq:dnn-expanded-horner-range} also holds at level $r$,
completing the induction.

In particular, since $y_0=P$,
\[
\|\widehat y_0-P\|_{\infty,[0,1]}
\leq q\omega<\xi.
\]

It remains to identify the approximate Horner recursion with a ReLU
network. Define
\[
\widetilde P(t):=\widehat y_0(t),
\qquad t\in[0,1].
\]
By construction, $\widehat y_0$ is obtained by successively composing
the $q$ multiplication networks $\mathsf M_{2A,\omega}$ with the affine
updates
\[
v\mapsto b_r+v,
\qquad r=q-1,\ldots,0.
\]
Moreover, since
\[
z=\rho(z)-\rho(-z),
\]
the input $t$ and the current approximate Horner iterate can be carried
through the successive multiplication blocks using only a universal
number of additional neurons per layer. Consequently,
$\widetilde P=\widehat y_0$ admits a ReLU representation satisfying
\begin{equation}
\widetilde P
\in
\mathcal F\left(
C(q+1)
\log\frac{e(q+1)A}{\xi},
C
\right),
\label{eq:dnn-expanded-poly-class}
\end{equation}
where $C>0$ is a universal constant.

Since $y_0=P$ and $\widetilde P=\widehat y_0$, the preceding Horner
error bound gives
\begin{equation}
\|\widetilde P-P\|_{\infty,[0,1]}
=
\|\widehat y_0-y_0\|_{\infty,[0,1]}
\leq q\omega
<\xi.
\label{eq:dnn-expanded-poly-error}
\end{equation}

\medskip\noindent\textbf{Piecewise-polynomial approximation.} We next extend the polynomial construction in
\eqref{eq:dnn-expanded-poly-class} to continuous piecewise-polynomial
functions. The main idea is to decompose such a function into a global
piecewise-linear interpolant and polynomial residuals supported on the
individual subintervals. This will later allow the residual corrections
to be evaluated sequentially while keeping the network width uniformly
bounded.

Let $N\in\mathbb N$ and let
\[
a=x_0<x_1<\cdots<x_N=b
\]
be a partition of the interval $[a,b]$. Suppose that
$P:[a,b]\to\mathbb R$ is continuous and that, on each subinterval
$[x_{r-1},x_r]$, $r\in\{1,\ldots,N\}$, the function $P$ is a
polynomial of degree at most $q$. More precisely, introduce the local
coordinate
\[
s=\frac{t-x_{r-1}}{x_r-x_{r-1}}\in[0,1],
\qquad
t=x_{r-1}+(x_r-x_{r-1})s,
\]
and define the corresponding rescaled polynomial
\[
P_r(s)
:=
P\bigl(x_{r-1}+(x_r-x_{r-1})s\bigr),
\qquad s\in[0,1].
\]
Assume that
\begin{equation}
P_r(s)
=
\sum_{k=0}^q b_{r,k}s^k,
\qquad
\sum_{k=0}^q |b_{r,k}|
\leq A,
\qquad
r=1,\ldots,N,
\label{eq:dnn-expanded-local-polynomial}
\end{equation}
for some constant $A\geq1$ that is uniform over the subintervals.

We first construct the continuous piecewise-linear function that
interpolates $P$ at the partition points $x_0,\ldots,x_N$. For
$r=1,\ldots,N$, define the slope of the secant line over
$[x_{r-1},x_r]$ by
\[
d_r
:=
\frac{P(x_r)-P(x_{r-1})}{x_r-x_{r-1}}.
\]
Recall that $\rho(u)=\max\{u,0\}$ denotes the ReLU activation. Define
\begin{equation}
I_P(t)
:=
P(x_0)+d_1(t-x_0)
+
\sum_{r=1}^{N-1}
(d_{r+1}-d_r)\rho(t-x_r),
\qquad t\in[a,b].
\label{eq:dnn-expanded-linear-interpolant}
\end{equation}
Then $I_P$ is precisely the piecewise-linear interpolant of the data
\[
\bigl\{(x_r,P(x_r)):r=0,\ldots,N\bigr\}.
\]
Indeed, define recursively
\[
v_0(t)
:=
P(x_0)+d_1(t-x_0)
\]
and, for $r=1,\ldots,N-1$,
\[
v_r(t)
:=
v_{r-1}(t)
+
(d_{r+1}-d_r)\rho(t-x_r).
\]
By construction,
\[
v_r(t)
=
P(x_0)+d_1(t-x_0)
+
\sum_{j=1}^{r}
(d_{j+1}-d_j)\rho(t-x_j),
\]
and hence
\[
v_{N-1}(t)=I_P(t).
\]

It remains to record the ReLU-network complexity of this representation.
At step $r$, the recursion only requires the original input $t$ and the
current value $v_{r-1}(t)$. Since every scalar quantity $z$ can be
propagated through a ReLU layer using the identity
\[
z=\rho(z)-\rho(-z),
\]
the pair $(t,v_{r-1}(t))$ can be carried through the successive updates
using only a universal number of neurons per layer. Each update
\[
v_{r-1}(t)
\longmapsto
v_{r-1}(t)+(d_{r+1}-d_r)\rho(t-x_r)
\]
requires only a constant number of additional ReLU units. Consequently,
the function $I_P$ admits a scalar ReLU representation of universal
width and depth at most $CN$; equivalently,
\begin{equation}
I_P\in\mathcal F(CN,C),
\label{eq:dnn-expanded-linear-interpolant-class}
\end{equation}
where $C>0$ is a universal constant independent of
$N$, $q$, $A$, the partition $\{x_r\}_{r=0}^N$, and the coefficients
$\{b_{r,k}\}$.

\medskip\noindent\textbf{Localization of the residuals.}
We separate the piecewise polynomial into a global linear interpolant
and residual terms that vanish at the endpoints of their respective
intervals. This makes each residual local after composition with the
clipped coordinate below.

For each interval, define
\[
R_r(s)
=
P_r(s)
-
(1-s)P(x_{r-1})
-
sP(x_r).
\]
Then
\[
R_r(0)=R_r(1)=0,
\]
and, by \eqref{eq:dnn-expanded-local-polynomial},
\begin{equation}
\sum_{k=0}^{ \max\{q,1\} }|[s^k]R_r(s)|
\leq4A.
\label{eq:dnn-expanded-residual-coefficients}
\end{equation}
Also define
\begin{equation}
s_r(t)
=
\frac{
\rho(t-x_{r-1})-\rho(t-x_r)
}{
x_r-x_{r-1}
}.
\label{eq:dnn-expanded-local-coordinate}
\end{equation}
Then
\[
s_r(t)=
\begin{cases}
0, & t\leq x_{r-1},\\[1mm]
\dfrac{t-x_{r-1}}{x_r-x_{r-1}},
    & x_{r-1}<t<x_r,\\[3mm]
1, & t\geq x_r.
\end{cases}
\]

By \eqref{eq:dnn-expanded-poly-class}--%
\eqref{eq:dnn-expanded-poly-error}, for each $r$ there exists a
universal-width ReLU network $R_r'$ such that
\[
\|R_r'-R_r\|_{\infty,[0,1]}
\leq\frac{\xi}{2}
\]
and
\[
R_r'
\in
\mathcal F\left(
C(q+1)\log\frac{e(q+1)A}{\xi},
C
\right).
\]
Set
\[
\widetilde R_r(s)
=
R_r'(s)
-
(1-s)R_r'(0)
-
sR_r'(1).
\]
Then
\[
\widetilde R_r(0)
=
\widetilde R_r(1)
=
0,
\]
and
\begin{align*}
|\widetilde R_r(s)-R_r(s)|
&\leq
|R_r'(s)-R_r(s)|\\
&\quad+
(1-s)|R_r'(0)-R_r(0)|
+
s|R_r'(1)-R_r(1)|\\
&\leq\xi.
\end{align*}
Hence
\begin{equation}
\|\widetilde R_r-R_r\|_{\infty,[0,1]}
\leq\xi.
\label{eq:dnn-expanded-residual-error}
\end{equation}
Moreover, since both $R_r$ and $\widetilde R_r$ vanish at $0$ and $1$,
\begin{equation}
R_r\{s_r(t)\}
=
\widetilde R_r\{s_r(t)\}
=
0,
\qquad
t\notin(x_{r-1},x_r).
\label{eq:dnn-expanded-residual-localization}
\end{equation}

For $t\in[x_{r-1},x_r]$, by definition, 
\[
P(t)
=
I_P(t)+R_r\{s_r(t)\}.
\]
Thus \eqref{eq:dnn-expanded-residual-localization} gives
\begin{equation}
P(t)
=
I_P(t)+
\sum_{r=1}^N R_r\{s_r(t)\}.
\label{eq:dnn-expanded-piecewise-decomposition}
\end{equation}
Define
\[
\widetilde P(t)
=
I_P(t)+
\sum_{r=1}^N
\widetilde R_r\{s_r(t)\}.
\]
For every fixed $t$, at most one summand in
\eqref{eq:dnn-expanded-piecewise-decomposition} is nonzero. Therefore,
by \eqref{eq:dnn-expanded-residual-error},
\begin{equation}
\|\widetilde P-P\|_{\infty,[a,b]}
\leq\xi.
\label{eq:dnn-expanded-piecewise-error}
\end{equation}

\medskip\noindent\textbf{Serial evaluation.}
Rather than evaluating all localized residual networks in parallel, we
add them successively. This makes the number of pieces contribute to
the depth while the width remains universal.

To evaluate $\widetilde P$ with universal width, let
\[
z_0(t)=I_P(t),
\qquad
z_r(t)=
z_{r-1}(t)+\widetilde R_r\{s_r(t)\},
\quad
r=1,\ldots,N.
\]
At each step only $t$ and $z_{r-1}$ need to be carried. Combining
\eqref{eq:dnn-expanded-linear-interpolant-class} with
\eqref{eq:dnn-expanded-poly-class} therefore gives
\begin{equation}
\widetilde P
\in
\mathcal F\left(
CN(q+1)
\log\frac{e(q+1)A}{\xi},
C
\right).
\label{eq:dnn-expanded-piecewise-poly}
\end{equation}

\medskip\noindent\textbf{Analytic polynomial approximation.}
We next obtain polynomial approximants whose degree and coefficient
size can both be controlled. These are the two quantities needed in
the ReLU construction above.

We next recall a standard polynomial-approximation bound for analytic
functions.  Map $[0,1]$ to $[-1,1]$ by $x=2s-1$ and define
\[
\widetilde f(x):=f\left(\frac{x+1}{2}\right),\qquad x\in[-1,1].
\]
Suppose that, for some $\rho>1$, $\widetilde f$ extends holomorphically
to the Bernstein ellipse $E_\rho$ with foci $\{-1,1\}$ and satisfies
\[
\sup_{z\in E_\rho}|\widetilde f(z)|\leq M,
\]
enlarging the upper bound so that $M\ge1$. Then $\widetilde f$ admits a Chebyshev expansion
\[
\widetilde f(x)=\sum_{k=0}^\infty c_kT_k(x),
\]
where $T_k$ is the $k$th Chebyshev polynomial of the first kind.
Equivalently, returning to $s\in[0,1]$,
\[
f(s)=\sum_{k=0}^\infty c_kT_k(2s-1).
\]
Then \citet[Theorems 8.1--8.2]{trefethen2019approximation} gives
\begin{equation}
|c_k|\leq 2M\rho^{-k},
\qquad
\left\|
f-\sum_{k=0}^q c_kT_k(2\,\cdot-1)
\right\|_{\infty,[0,1]}
\leq
\frac{2M}{\rho-1}\rho^{-q}.
\label{eq:dnn-expanded-analytic-error}
\end{equation}

Recall that the Chebyshev polynomials of the first kind are characterized by
\[
T_k(\cos\theta)=\cos(k\theta),\qquad k\in\mathbb N_0.
\]
The trigonometric identity
\[
\cos\{(k+1)\theta\}
=
2\cos(\theta)\cos(k\theta)-\cos\{(k-1)\theta\}
\]
therefore yields the standard three-term recurrence
\[
T_{k+1}(x)=2xT_k(x)-T_{k-1}(x),
\qquad k\geq1,
\]
with
\[
T_0(x)=1,
\qquad
T_1(x)=x.
\]

We next control the coefficients after the change of variables
$x=2s-1$.  For a polynomial
\[
Q(s)=\sum_{r\geq0}a_rs^r,
\]
write
\[
[s^r]Q(s):=a_r
\]
for the coefficient of $s^r$ in $Q$,
and define its coefficient $\ell_1$-norm by
\[
\|Q\|_{\mathrm{coef},1}
:=
\sum_{r\geq0}|a_r|.
\]
Equivalently,
\[
\|Q\|_{\mathrm{coef},1}
=
\sum_{r\geq0}|[s^r]Q(s)|.
\]
Let
\[
A_k
:=
\|T_k(2\,\cdot-1)\|_{\mathrm{coef},1}
=
\sum_{r=0}^k
|[s^r]T_k(2s-1)|.
\]
Since
\[
T_0(2s-1)=1,
\qquad
T_1(2s-1)=2s-1,
\]
we have
\[
A_0=1,
\qquad
A_1=3.
\]

Substituting $x=2s-1$ into the Chebyshev recurrence gives
\[
T_{k+1}(2s-1)
=
(4s-2)T_k(2s-1)-T_{k-1}(2s-1).
\]
For every polynomial $Q$, multiplication by $s$ only shifts its
coefficients, and therefore
\[
\|sQ\|_{\mathrm{coef},1}
=
\|Q\|_{\mathrm{coef},1}.
\]
Consequently, by the triangle inequality,
\begin{align*}
A_{k+1}
&\leq
4\|sT_k(2s-1)\|_{\mathrm{coef},1}
+
2\|T_k(2s-1)\|_{\mathrm{coef},1}
+
\|T_{k-1}(2s-1)\|_{\mathrm{coef},1}\\
&=
6A_k+A_{k-1}.
\end{align*}
We now prove by induction that
\[
A_k\leq 7^k,\qquad k\geq0.
\]
The claim holds for $k=0,1$, since $A_0=1$ and $A_1=3\leq7$.
If it holds at levels $k$ and $k-1$, then
\[
A_{k+1}
\leq
6\cdot7^k+7^{k-1}
=
43\cdot7^{k-1}
\leq
49\cdot7^{k-1}
=
7^{k+1}.
\]
Hence, for every $k\geq0$,
\[
\sum_{r=0}^k
|[s^r]T_k(2s-1)|
\leq7^k.
\]

Next, define the degree-$q$ truncated Chebyshev approximation by
\[
p_q(s)
:=
\sum_{k=0}^q c_k T_k(2s-1),
\qquad s\in[0,1].
\]
 Since
\[
[s^r]p_q(s)
=
\sum_{k=r}^q c_k [s^r]T_k(2s-1),
\]
the triangle inequality and a change in the order of summation give
\begin{align*}
1+\sum_{r=0}^q |[s^r]p_q(s)|
&\leq
1+
\sum_{r=0}^q
\sum_{k=r}^q
|c_k|\,
|[s^r]T_k(2s-1)|\\
&=
1+
\sum_{k=0}^q
|c_k|
\sum_{r=0}^k
|[s^r]T_k(2s-1)|.
\end{align*}
Using the bounds
\[
|c_k|\leq 2M\rho^{-k}
\qquad\text{and}\qquad
\sum_{r=0}^k |[s^r]T_k(2s-1)|
\leq 7^k,
\]
we obtain
\[
1+\sum_{r=0}^q |[s^r]p_q(s)|
\leq
1+2M\sum_{k=0}^q\left(\frac{7}{\rho}\right)^k.
\]
Since $\rho>1$ is fixed, there exist constants $C>0$ and $C_0>1$,
depending only on $\rho$ (equivalently, only on the fixed analytic
neighborhood), such that
\[
1+2M\sum_{k=0}^q\left(\frac{7}{\rho}\right)^k
\leq
CM C_0^q.
\]

Define
\[
\vartheta:=\rho^{-1}\in(0,1).
\]
then
\[
\|f-p_q\|_{\infty,[0,1]}
\leq
\frac{2M}{\rho-1}\rho^{-q}
\leq
CM\vartheta^q,
\]
where $C>0$ depends only on the fixed value of $\rho$. Combining this
with the coefficient bound above, we obtain
\begin{equation}
\|f-p_q\|_{\infty,[0,1]}
\leq
CM\vartheta^q,
\qquad
1+\sum_{r=0}^q|[s^r]p_q(s)|
\leq
CMC_0^q,
\label{eq:dnn-expanded-analytic-poly}
\end{equation}
where $\vartheta\in(0,1)$, $C_0>1$, and $C>0$ depend only on
$\rho$, equivalently only on the fixed analytic neighborhood.

Replacing $p_q$ by
\[
p_q^0(s)
=
p_q(s)
+
(1-s)\{f(0)-p_q(0)\}
+
s\{f(1)-p_q(1)\}
\]
gives
\[
p_q^0(0)=f(0),
\qquad
p_q^0(1)=f(1),
\]
and
\begin{align*}
|p_q^0(s)-f(s)|
&\leq
|p_q(s)-f(s)|
+
(1-s)|p_q(0)-f(0)|\\
&\quad+
s|p_q(1)-f(1)|\\
&\leq
2\|p_q-f\|_{\infty,[0,1]}.
\end{align*}
The coefficient sum of $p_q^0$ is bounded by $CMC_0^q$ after changing
$C$.

\medskip\noindent\textbf{Approximation of the logarithm.}
We use geometrically growing intervals so that, after rescaling each
interval to $[0,1]$, the nonlinear part of $\log$ belongs to a
uniformly holomorphic family.

We first apply these bounds to $\log$. Let
\[
a_0=\epsilon,
\qquad
a_{r+1}=\min\{2a_r,1\},
\]
and stop at $a_N=1$. Then
\begin{equation}
N
\leq
1+\log_2(1/\epsilon)
\leq
C\Lambda_\epsilon.
\label{eq:dnn-expanded-log-pieces}
\end{equation}
For $t\in[a_r,a_{r+1}]$, write
\[
t=a_r(1+\lambda_rs),
\qquad
\lambda_r=
\frac{a_{r+1}-a_r}{a_r},
\qquad
0\leq s\leq1.
\]
Then
\[
0<\lambda_r\leq1,
\qquad
\log t
=
\log a_r+\log(1+\lambda_rs).
\]

Choose a universal constant $c\in(0,1/2]$ and define
\[
\mathcal N
=
\{z\in\mathbb C:
\operatorname{dist}(z,[0,1])<c\}.
\]
If $z\in\mathcal N$, then $\operatorname{Re}(z)\geq -c$. Hence, for
$0\leq\lambda\leq1$,
\[
\operatorname{Re}(1+\lambda z)
=
1+\lambda\operatorname{Re}(z)
\geq
1-\lambda c
\geq
1-c
\geq
\frac12.
\]
Thus $1+\lambda z$ remains uniformly inside the right half-plane, where
the principal logarithm is holomorphic. Moreover,
$|z|\leq1+c$, so
\[
\frac12
\leq
|1+\lambda z|
\leq
2+c.
\]
Therefore the family $z\mapsto\log(1+\lambda z)$ is uniformly bounded
on $\mathcal N$:
\begin{equation}
\sup_{\substack{z\in\mathcal N\\0\leq\lambda\leq1}}
|\log(1+\lambda z)|
\leq C,
\label{eq:dnn-expanded-log-uniform}
\end{equation}
for a universal constant $C>0$.

For each $r\in\{0,\ldots,N-1\}$ and
$t\in[a_r,a_{r+1}]$, introduce the rescaled coordinate
\[
s=\frac{t-a_r}{a_{r+1}-a_r}\in[0,1]
\]
and set
\[
\lambda_r
:=
\frac{a_{r+1}-a_r}{a_r}\in(0,1].
\]
Then
\[
t=a_r(1+\lambda_r s)
\]
and therefore
\begin{equation}
\log t
=
\log a_r+h_r(s),
\qquad
h_r(s):=\log(1+\lambda_r s).
\label{eq:dnn-expanded-log-local}
\end{equation}

By \eqref{eq:dnn-expanded-log-uniform}, the family
$\{h_r\}_{r=0}^{N-1}$ is holomorphic and uniformly bounded on a common
complex neighborhood of $[0,1]$.  Hence
\eqref{eq:dnn-expanded-analytic-poly} applies uniformly in $r$:
there exist constants $\vartheta\in(0,1)$ and $C_0>1$, independent of
$r$, such that for every $q\in\mathbb N_0$ there is a polynomial
$p_{r,q}$ of degree at most $q$ satisfying
\begin{equation}
\|h_r-p_{r,q}\|_{\infty,[0,1]}
\leq C\vartheta^q,
\qquad
1+\sum_{k=0}^q |[s^k]p_{r,q}(s)|
\leq CC_0^q.
\label{eq:dnn-expanded-log-local-poly}
\end{equation}

Choose
\[
q=\left\lceil C\Lambda_\xi\right\rceil
\]
with the universal constant $C>0$ sufficiently large that
$C\vartheta^q\leq\xi/4$.  Thus
\[
\|h_r-p_{r,q}\|_{\infty,[0,1]}
\leq\frac{\xi}{4}
\]
uniformly in $r$.  Define the endpoint-corrected polynomial
\[
\widetilde p_{r,q}(s)
:=
p_{r,q}(s)
+(1-s)\{h_r(0)-p_{r,q}(0)\}
+s\{h_r(1)-p_{r,q}(1)\}.
\]
Then
\[
\widetilde p_{r,q}(0)=h_r(0),
\qquad
\widetilde p_{r,q}(1)=h_r(1),
\]
and, by the same endpoint-correction argument used above,
\[
\|\widetilde p_{r,q}-h_r\|_{\infty,[0,1]}
\leq\frac{\xi}{2}.
\]
Consequently, the local polynomials
\[
Q_r(s):=\log a_r+\widetilde p_{r,q}(s)
\]
match $\log t$ at the endpoints of $[a_r,a_{r+1}]$.  Hence the
functions
\[
P^{\log}_{\epsilon,\xi}(t)
:=
Q_r\left(
\frac{t-a_r}{a_{r+1}-a_r}
\right),
\qquad
t\in[a_r,a_{r+1}],
\]
fit together to form a continuous piecewise polynomial on
$[\epsilon,1]$ satisfying
\begin{equation}
\|P^{\log}_{\epsilon,\xi}-\log\|_{\infty,[\epsilon,1]}
\leq\frac{\xi}{2}.
\label{eq:dnn-expanded-log-piecewise-error}
\end{equation}

We next bound the coefficients of the local polynomials $Q_r$. Define
\[
e_{r,0}
:=
h_r(0)-p_{r,q}(0),
\qquad
e_{r,1}
:=
h_r(1)-p_{r,q}(1).
\]
Since
\[
\|h_r-p_{r,q}\|_{\infty,[0,1]}
\leq \frac{\xi}{4},
\]
we have
\[
|e_{r,0}|\leq\frac{\xi}{4},
\qquad
|e_{r,1}|\leq\frac{\xi}{4}.
\]
The endpoint-corrected polynomial can be written as
\[
\widetilde p_{r,q}(s)
=
p_{r,q}(s)
+
e_{r,0}
+
(e_{r,1}-e_{r,0})s.
\]
Thus the correction is affine, and the sum of the absolute values of
its coefficients is bounded by
\[
|e_{r,0}|+|e_{r,1}-e_{r,0}|
\leq
2|e_{r,0}|+|e_{r,1}|
\leq
\frac{3\xi}{4}.
\]
Consequently, by the triangle inequality for polynomial coefficients
and \eqref{eq:dnn-expanded-log-local-poly},
\begin{align*}
1+\sum_{k=0}^q
|[s^k]\widetilde p_{r,q}(s)|
&\leq
1+\sum_{k=0}^q
|[s^k]p_{r,q}(s)|
+\frac{3\xi}{4}\\
&\leq
CC_0^q,
\end{align*}
after enlarging the universal constant $C$, since $\xi<1$ and
$C_0^q\geq1$.

Moreover, since $\epsilon\leq a_r\leq1$,
\[
|\log a_r|
=
\log(1/a_r)
\leq
\log(1/\epsilon)
\leq
\Lambda_\epsilon.
\]
Thus the coefficient $\ell_1$-norm of every $Q_r$ is bounded uniformly
by
\begin{equation}
A
\leq
C(1+\Lambda_\epsilon)C_0^q.
\label{eq:dnn-expanded-log-piecewise-coeff}
\end{equation}

Applying
\eqref{eq:dnn-expanded-piecewise-error}--%
\eqref{eq:dnn-expanded-piecewise-poly}
to $P^{\log}_{\epsilon,\xi}$ with network-approximation error
$\xi/2$ therefore yields a scalar ReLU network
\[
G_{\epsilon,\xi}
\in
\mathcal F\left(
CN(q+1)
\log\frac{e(q+1)A}{\xi},
C
\right)
\]
such that, by \eqref{eq:dnn-expanded-log-piecewise-error},
\begin{equation}
\|G_{\epsilon,\xi}-\log\|_{\infty,[\epsilon,1]}
\leq\xi.
\label{eq:dnn-expanded-log-network}
\end{equation}

Finally,
\eqref{eq:dnn-expanded-log-pieces},
$q+1\lesssim\Lambda_\xi$, and
\eqref{eq:dnn-expanded-log-piecewise-coeff} imply
\begin{align*}
N(q+1)
\log\frac{e(q+1)A}{\xi}
&\leq
C\Lambda_\epsilon\Lambda_\xi
\left\{
\Lambda_\xi+\log(1+\Lambda_\epsilon)
\right\}\\
&\leq
C\Lambda_\epsilon\Lambda_\xi
(\Lambda_\epsilon+\Lambda_\xi).
\end{align*}
Hence
\[
G_{\epsilon,\xi}
\in
\mathcal F\left(
C\Lambda_\epsilon\Lambda_\xi
(\Lambda_\epsilon+\Lambda_\xi),
C
\right).
\]

\medskip\noindent\textbf{Construction of the logit network.}
Since $\ell(t)=\log t-\log(1-t)$, the logit approximator is obtained
from two copies of the logarithm network evaluated in parallel.

Let $\eta \in (0,1)$ and set
\[
\mathsf L_{\epsilon,\eta}(t)
=
G_{\epsilon,\eta/2}(t)
-
G_{\epsilon,\eta/2}(1-t).
\]
For $t\in[\epsilon,1-\epsilon]$,
\begin{align*}
|\mathsf L_{\epsilon,\eta}(t)-\ell(t)|
&\leq
|G_{\epsilon,\eta/2}(t)-\log t|\\
&\quad+
|G_{\epsilon,\eta/2}(1-t)-\log(1-t)|\\
&\leq\eta.
\end{align*}
Evaluating the two universal-width networks in parallel changes the
width only by a universal constant. Since
$\Lambda_{\eta/2}\leq C\Lambda_\eta$,
\[
\mathsf L_{\epsilon,\eta}
\in
\mathcal F\left(
C\Lambda_\epsilon\Lambda_\eta
(\Lambda_\epsilon+\Lambda_\eta),
C
\right),
\]
which proves
\eqref{eq:dnn-expanded-logit}--%
\eqref{eq:dnn-expanded-logit-depth}.

\medskip\noindent\textbf{Approximation of inverse softplus below one.}
Near zero, $\psi$ has the same logarithmic singularity as $\log t$.
We therefore separate this term from a remainder that is uniformly
analytic after rescaling.

It remains to treat $\psi$. For $0<t\leq1$,
\begin{equation}
\psi(t)
=
\log t+r(t),
\qquad
r(t)
=
\log\left(\frac{e^t-1}{t}\right),
\qquad
r(0)=0.
\label{eq:dnn-expanded-psi-small-decomp}
\end{equation}
With the removable singularity at zero filled in,
\[
\frac{e^z-1}{z}=0
\quad\Longleftrightarrow\quad
z=2\pi ik,
\qquad
k\in\mathbb Z\setminus\{0\}.
\]
Thus, after decreasing the universal constant $c$ defining
$\mathcal N$ if necessary,
\[
\frac{e^z-1}{z}\neq0
\]
on a simply connected neighborhood containing
\[
\{b(1+\lambda z):
z\in\mathcal N,\;
0<b\leq1,\;
0\leq\lambda\leq1,\;
b(1+\lambda)\leq1\}.
\]
Let
\[
h(w):=
\begin{cases}
\dfrac{e^w-1}{w}, & w\neq0,\\[1mm]
1, & w=0.
\end{cases}
\]
By the preceding zero-free argument, after decreasing $c$ if necessary,
there exists a fixed simply connected open set $D$ containing
\[
K:=
\left\{
b(1+\lambda z):
z\in\overline{\mathcal N},\;
0\leq b\leq1,\;
0\leq\lambda\leq1,\;
b(1+\lambda)\leq1
\right\},
\]
such that $h$ has no zeros on $D$. Hence $h$ admits a holomorphic
logarithm $r$ on $D$, chosen to agree with
$r(t)=\log\{(e^t-1)/t\}$ for $t>0$. Since $K$ is compact and $r$ is
continuous on $K$, there is a universal constant $C>0$ such that
\begin{equation}
\sup_{\substack{z\in\mathcal N\\
0<b\leq1,\;0\leq\lambda\leq1\\
b(1+\lambda)\leq1}}
|r\{b(1+\lambda z)\}|
\leq C.
\label{eq:dnn-expanded-r-uniform}
\end{equation}

Partition $[\delta,\min\{U,1\}]$ geometrically as follows.  Set
\[
b_0=\delta,
\qquad
b_{j+1}=\min\{2b_j,\min\{U,1\}\},
\]
and stop at the first index $N_0$ for which
$b_{N_0}=\min\{U,1\}$.  Then every interval
$[b_j,b_{j+1}]$ satisfies
\[
1<\frac{b_{j+1}}{b_j}\leq2,
\]
except possibly for equality at the final endpoint. Since
$\min\{U,1\}\leq1$,
\[
N_0
\leq
1+\log_2(1/\delta)
\leq
C\Lambda_\delta,
\qquad
\Lambda_\delta=\log(e/\delta).
\]
On each interval $[b_j,b_{j+1}]$ of the geometric partition, write
\[
t=b(1+\lambda s),
\qquad
b=b_j,\qquad
\lambda=\frac{b_{j+1}-b_j}{b_j}\in(0,1],
\qquad
s\in[0,1].
\]
Using the decomposition
\[
\psi(t)=\log t+r(t),
\]
we obtain
\begin{equation}
    \label{eqn:100}
    \psi\{b(1+\lambda s)\}
=
\log b+
\log(1+\lambda s)
+
r\{b(1+\lambda s)\}.
\end{equation}
The term $\log b$ is constant in $s$.  By
\eqref{eq:dnn-expanded-log-uniform}, the functions
$s\mapsto\log(1+\lambda s)$ are holomorphic and uniformly bounded on
the same complex neighborhood, uniformly over $\lambda\in(0,1]$.
Likewise, by \eqref{eq:dnn-expanded-r-uniform}, the functions
$s\mapsto r\{b(1+\lambda s)\}$ are holomorphic and uniformly bounded
there, uniformly over all admissible $b$ and $\lambda$.  Hence the
nonconstant part of $\psi\{b(1+\lambda s)\}$ satisfies the assumptions
of \eqref{eq:dnn-expanded-analytic-poly} with constants independent of
the interval index.

\medskip\noindent\textbf{Approximation of inverse softplus above one.} For $t\geq1$, we separate the linear growth of $\psi$ from a bounded
remainder. Since
\[
e^t-1=e^t(1-e^{-t}),
\]
we may write
\begin{equation}
\psi(t)
=
\log(e^t-1)
=
t+r_+(t),
\qquad
r_+(t):=\log(1-e^{-t}).
\label{eq:dnn-expanded-psi-large-decomp}
\end{equation}
The function $r_+$ is bounded on $[1,\infty)$ and converges to zero as
$t\to\infty$.

When $U>1$, partition $[1,U]$ into consecutive intervals
$[b_j,b_j+h_j]$, $j=1,\ldots,N_1$, with
\[
b_j\geq1,
\qquad
0<h_j\leq1.
\]
For example, one may use intervals of length one, except possibly for
the final interval. Their number satisfies
\begin{equation}
N_1
\leq
\lceil U-1\rceil
\leq
1+U.
\label{eq:dnn-expanded-psi-large-pieces}
\end{equation}
On the $j$th interval, introduce the local coordinate $u\in[0,1]$ by
\[
t=b_j+h_j u.
\]
Then
\[
\psi(b_j+h_j u)
=
b_j+h_j u+r_+(b_j+h_j u),
\]
where the first two terms are affine in $u$ and only the bounded
remainder $r_+(b_j+h_j u)$ requires polynomial approximation.

For $z\in\mathcal N$, we have $\operatorname{Re}(z)\geq-c$. Hence, for
$b\geq1$ and $0<h\leq1$,
\[
\operatorname{Re}(b+hz)
=
b+h\operatorname{Re}(z)
\geq
1-c.
\]
Choosing $c\leq1/4$ gives
\[
\operatorname{Re}(b+hz)\geq\frac34,
\]
and therefore
\[
|e^{-(b+hz)}|
=
e^{-\operatorname{Re}(b+hz)}
\leq
e^{-3/4}<1.
\]
Consequently,
\[
1-e^{-(b+hz)}
\in
\{w\in\mathbb C:|w-1|\leq e^{-3/4}\},
\]
a fixed disk contained in the right half-plane and bounded away from
zero. Hence the principal logarithm is holomorphic on this disk, so
\[
z\mapsto
r_+(b+hz)
=
\log\{1-e^{-(b+hz)}\}
\]
is holomorphic on $\mathcal N$. Since the disk is compactly contained
in the domain of the logarithm, there exists a universal constant
$C>0$ such that
\begin{equation}
\sup_{\substack{z\in\mathcal N\\ b\geq1,\;0<h\leq1}}
|r_+(b+hz)|
\leq C.
\label{eq:dnn-expanded-s-uniform}
\end{equation}

\medskip\noindent\textbf{Assembly of the inverse-softplus network.}
Let $N_0$ denote the number of intervals in the geometric partition of
$[\delta,\min\{U,1\}]$ constructed above, with $N_0=0$ if this interval
is degenerate.  By construction,
\[
N_0\leq C\Lambda_\delta.
\]
Similarly, when $U>1$, let $N_1$ denote the number of intervals
$[b_j,b_j+h_j]$ in the partition of $[1,U]$ constructed above, where
$b_j\geq1$ and $0<h_j\leq1$; set $N_1=0$ when $U\leq1$.
By \eqref{eq:dnn-expanded-psi-large-pieces},
\[
N_1\leq1+U.
\]  

The two collections of intervals together form a partition of
$[\delta,U]$ (with the common endpoint $1$ shared when $U>1$).
Therefore, if $N$ denotes the total number of intervals,
\begin{equation}
N=N_0+N_1
\leq
C\Lambda_\delta+(1+U)
\leq
C(1+U+\Lambda_\delta).
\label{eq:dnn-expanded-psi-pieces}
\end{equation}

On every interval below $1$, the rescaled function has the form (\ref{eqn:100}). The preceding argument shows that
$s\mapsto\log(1+\lambda s)$ is holomorphic and uniformly bounded on a
fixed complex neighborhood of $[0,1]$, uniformly over
$0<\lambda\leq1$. Moreover, by
\eqref{eq:dnn-expanded-r-uniform},
$s\mapsto r\{b(1+\lambda s)\}$ is holomorphic and uniformly bounded on
the same neighborhood, uniformly over all admissible $b$ and $\lambda$.  On every interval above $1$, writing
$t=b+hs$ gives
\[
\psi(b+hs)
=
b+hs+r_+(b+hs),
\]
where $r_+$ is holomorphic and uniformly bounded by
\eqref{eq:dnn-expanded-s-uniform}.  Hence
\eqref{eq:dnn-expanded-analytic-poly} applies on every interval with
constants independent of the interval.

Choose
\[
q=\left\lceil C\Lambda_\zeta\right\rceil
\]
with $C>0$ sufficiently large that the degree-$q$ polynomial
approximation error on each interval is at most $\zeta/4$.
After the endpoint correction, the error on each interval is at most
$\zeta/2$, and adjacent polynomial pieces agree at their common
endpoints.  Thus they define a continuous piecewise polynomial
$P_\psi$ on $[\delta,U]$ satisfying
\begin{equation}
\|P_\psi-\psi\|_{\infty,[\delta,U]}
\leq
\frac{\zeta}{2}.
\label{eq:dnn-expanded-psi-piecewise-error}
\end{equation}

We next bound the coefficients of its local polynomial pieces.
For intervals below $1$, the constant term satisfies
\[
|\log b|
\leq
\log(1/\delta)
\leq
\Lambda_\delta,
\]
while the remaining analytic terms have coefficient $\ell_1$-norm at
most $CC_0^q$.  For intervals above $1$, the affine term
$b+hs$ has coefficient $\ell_1$-norm $b+h\leq U+1$, while the
polynomial approximation to $r_+(b+hs)$ again has coefficient
$\ell_1$-norm at most $CC_0^q$.  The endpoint corrections contribute
only $O(\zeta)$ to these coefficient sums.  Consequently, uniformly
over all local pieces,
\begin{equation}
A
\leq
C(1+U+\Lambda_\delta)C_0^q.
\label{eq:dnn-expanded-psi-piecewise-coeff}
\end{equation}

Applying
\eqref{eq:dnn-expanded-piecewise-error}--%
\eqref{eq:dnn-expanded-piecewise-poly}
to $P_\psi$ with network-approximation error $\zeta/2$ yields a
universal-width ReLU network $\mathsf S_{\delta,U,\zeta}$ satisfying
\begin{equation}
\|\mathsf S_{\delta,U,\zeta}-\psi\|_{\infty,[\delta,U]}
\leq
\zeta.
\label{eq:dnn-expanded-spinv2}
\end{equation}
Its depth satisfies
\begin{align*}
D_\psi(\delta,U,\zeta)
&\leq
CN(q+1)
\log\frac{e(q+1)A}{\zeta}\\
&\leq
C(1+U+\Lambda_\delta)\Lambda_\zeta
\log\left(
\frac{
e\Lambda_\zeta
(1+U+\Lambda_\delta)
C_0^{C\Lambda_\zeta}
}{
\zeta
}
\right)\\
&\leq
C(1+U+\Lambda_\delta)\Lambda_\zeta
(1+U+\Lambda_\delta+\Lambda_\zeta).
\end{align*}
Here the first inequality follows from the piecewise-polynomial network
bound \eqref{eq:dnn-expanded-piecewise-poly}, applied with
network-approximation error $\zeta/2$ (with the resulting numerical
factor absorbed into $C$).  The second inequality uses
\eqref{eq:dnn-expanded-psi-pieces},
the choice $q=\lceil C\Lambda_\zeta\rceil$, which implies
$q+1\leq C\Lambda_\zeta$, and the coefficient bound
\eqref{eq:dnn-expanded-psi-piecewise-coeff}; in particular,
$C_0^q\leq C_0^{C\Lambda_\zeta}$.  Finally, the last inequality follows
by expanding the logarithm and using
\[
\log(1/\zeta)\leq\Lambda_\zeta,\qquad
\log\Lambda_\zeta\leq\Lambda_\zeta,\qquad
\log(1+U+\Lambda_\delta)\leq1+U+\Lambda_\delta,
\]
together with the fact that $C_0>1$ is a fixed constant.  This proves
\eqref{eq:dnn-expanded-spinv}--%
\eqref{eq:dnn-expanded-spinv-depth}.

\medskip\noindent\textbf{Endpoint bounds.}
Finally, since $e^\delta-1\geq\delta$ and
$e^U-1\leq e^U$,
\[
-\psi(\delta)
=
\log\frac{1}{e^\delta-1}
\leq
\log(1/\delta),
\qquad
\psi(U)
=
\log(e^U-1)
\leq
U.
\]
This proves \eqref{eq:dnn-expanded-spinv-endpoints}. Enlarging the
universal width constant appearing above, if necessary, defines
$w_{\mathrm{cal}}$ and completes the proof.
\end{proof}

\begin{lemma}[Parallelization, composition, and maximum of ReLU networks]
\label{lem:relu-network-operations}
Let $f_s:\mathbb R^d\to\mathbb R$, $s=1,\ldots,r$, be ReLU networks
with depths $L_s$ and widths $w_s$. Put
\[
L_{\max}=\max_{s\leq r}L_s,
\qquad
w_{\Sigma}=\sum_{s=1}^r(w_s+1).
\]
Then for a constant $C>0$ that can change from line to line:
\begin{enumerate}
\item $(f_1,\ldots,f_r)$ has a ReLU representation of depth at most
$L_{\max}+C$ and width at most $Cw_{\Sigma}$;

\item if $q:\mathbb R^r\to\mathbb R$ has depth $L_q$ and width
$w_q$, then $q\circ(f_1,\ldots,f_r)$ has depth at most
$L_{\max}+L_q+C$ and width at most
$C\max\{w_{\Sigma},w_q\}$;

\item $\max_{s\leq r}f_s$ is represented exactly with depth at most
$L_{\max}+\lceil\log_2r\rceil+C$ and width at most
$Cw_{\Sigma}$.
\end{enumerate}
\end{lemma}

\begin{proof}
The first two statements follow from the parallelization and composition
constructions in the proof of \citet[Theorem~3]{kohler2019rate}. To
equalize the depths, use
\begin{equation}
z=\rho(z)-\rho(-z).
\label{eq:relu-identity-map}
\end{equation}
Thus, each $f_s$ can be padded to depth $L_{\max}$ with width at most
$C(w_s+1)$. Placing the padded networks in parallel gives
\[
L_{\bm f}\leq L_{\max}+C,
\qquad
w_{\bm f}
\leq
C\sum_{s=1}^r(w_s+1)
=
Cw_\Sigma,
\]
where $\bm f=(f_1,\ldots,f_r)$. Composing this representation with $q$
gives
\[
L_{q\circ\bm f}
\leq
L_{\max}+L_q+C,
\qquad
w_{q\circ\bm f}
\leq
C\max\{w_\Sigma,w_q\},
\]
which proves the first two statements.

For the third statement, let
\[
R=2^{\lceil\log_2r\rceil}
\]
and, if $R>r$, set $f_{r+1}=\cdots=f_R=f_r$. Define
\[
v_s^{(0)}=f_s,
\qquad s=1,\ldots,R,
\]
and recursively, for $t=1,\ldots,\log_2R$,
\begin{equation}
v_s^{(t)}
=
\rho\left(
v_{2s-1}^{(t-1)}-v_{2s}^{(t-1)}
\right)
+
v_{2s}^{(t-1)},
\qquad
s=1,\ldots,\frac{R}{2^t}.
\label{eq:relu-balanced-maximum}
\end{equation}
Since
\[
\rho(a-b)+b=\max\{a,b\},
\]
induction gives
\[
v_s^{(t)}
=
\max_{(s-1)2^t<k\leq s2^t}f_k,
\]
and hence
\[
v_1^{(\log_2R)}
=
\max_{1\leq s\leq r}f_s.
\]

Using
\[
b=\rho(b)-\rho(-b),
\]
each update in \eqref{eq:relu-balanced-maximum} is represented by one
ReLU layer with at most a constant number of units per active value.
At level $t$ there are $R/2^t$ active values, and therefore
\[
w_t
\leq
C\frac{R}{2^t}
\leq
CR
\leq
2Cr
\leq
Cw_\Sigma.
\]
Since
\[
\log_2R=\lceil\log_2r\rceil,
\]
the resulting network has depth at most
\[
L_{\max}+\lceil\log_2r\rceil+C
\]
and width at most $Cw_\Sigma$.
\end{proof}

\begin{definition}[Construction of the latent ReLU networks]
\label{def:dnn-expanded-latent-construction}
Let $F_{A_1}^*,\ldots,F_{A_m}^*$ be the target probability functions and
let $h_1,\ldots,h_m$ be $[0,1]$-valued ReLU networks. Fix
\[
0<\epsilon<\frac14,
\qquad
0<\eta<1,
\qquad
0<\delta\leq1,
\qquad
0<\zeta<1.
\]
Define
\begin{equation}
F_{A_j,\epsilon}^*
=
\epsilon+(1-2\epsilon)F_{A_j}^*,
\qquad
\widetilde h_j
=
\epsilon+(1-2\epsilon)h_j,
\qquad
\lambda_\epsilon
=
\log\frac{1-\epsilon}{\epsilon}.
\label{eq:dnn-expanded-regularized-target}
\end{equation}

Let $\mathsf L_{\epsilon,\eta}$ be the logit approximation from
Lemma~\ref{lem:dnn-expanded-fixed-modules}, and set
\begin{equation}
q_j
=
\mathsf L_{\epsilon,\eta}\circ\widetilde h_j.
\label{eq:dnn-expanded-q}
\end{equation}

Define the DAG level recursively by
\begin{equation}
r_j=0,
\quad j\in\mathcal R,
\qquad
r_j
=
1+\max_{k\in\mathcal P(j)}r_k,
\quad j\notin\mathcal R.
\label{eq:dnn-expanded-level}
\end{equation}
Then, with the notation of Theorem \ref{theo1},
\[
r_j\leq\kappa_j-1\leq\kappa_m-1.
\]
Define
\begin{equation}
\bar q_j=q_j,
\quad j\in\mathcal R,
\qquad
\bar q_j
=
\max\left\{
q_j,\,
\max_{k\in\mathcal P(j)}(\bar q_k+\delta)
\right\},
\quad j\notin\mathcal R.
\label{eq:dnn-expanded-ordered-q}
\end{equation}

For $j\notin\mathcal R$, set
\begin{equation}
u_j
=
\bar q_j-\max_{k\in\mathcal P(j)}\bar q_k,
\qquad
U_{\epsilon,\delta,\eta}
=
2\lambda_\epsilon+2\eta+\delta\kappa_m.
\label{eq:dnn-expanded-u}
\end{equation}
Let $\mathsf S_{\delta,U,\zeta}$ be the inverse-softplus approximation
from Lemma~\ref{lem:dnn-expanded-fixed-modules}, and define
\begin{equation}
g_j=\bar q_j,
\quad j\in\mathcal R,
\qquad
g_j
=
\mathsf S_{\delta,U_{\epsilon,\delta,\eta},\zeta}\circ u_j,
\quad j\notin\mathcal R.
\label{eq:dnn-expanded-g}
\end{equation}
\end{definition}

\begin{theorem}[Probability-scale ReLU approximation]
\label{theo:dnn-expanded-probability-approx}
Suppose that
\[
F_{A_k}^*(x)\leq F_{A_j}^*(x),
\qquad
(k,j)\in E,\quad x\in[0,1]^d,
\]
and that, for some $a\geq0$, with the notation in Definition \ref{def:dnn-expanded-latent-construction}, we have
\begin{equation}
\max_{j\leq m}
\|h_j-F_{A_j}^*\|_\infty
\leq a.
\label{eq:dnn-expanded-base-approx}
\end{equation}
Let $\bm g=(g_1,\ldots,g_m)$ be the networks constructed in
Definition~\ref{def:dnn-expanded-latent-construction}. Then
\begin{equation}
\|H_j(\bm g)-F_{A_j}^*\|_\infty
\leq
\epsilon+a+
\frac14\left(
\eta+r_j\delta+r_j\zeta
\right),
\qquad
j\in[m],
\label{eq:dnn-expanded-output-error}
\end{equation}
and therefore
\begin{equation}
\max_{j\leq m}
\|H_j(\bm g)-F_{A_j}^*\|_\infty
\leq
\epsilon+a+
\frac14
\left(
\eta+\kappa_m\delta+\kappa_m\zeta
\right).
\label{eq:dnn-expanded-output-error-max}
\end{equation}
Moreover,
\begin{equation}
\max_{j\leq m}\|g_j\|_\infty
\leq
C\left\{
\lambda_\epsilon+\eta+\delta\kappa_m
+\log(1/\delta)+\zeta+1
\right\}.
\label{eq:dnn-expanded-envelope}
\end{equation}

If the $h_j$ have depth at most $L$ and width at most $w$, and
\begin{equation}
\deg^{-}_{\max,m}
=
\max_{j\leq m}|\mathcal P(j)|,
\label{eq:dnn-expanded-max-indegree}
\end{equation}
where $\deg^{-}_{\max,m}$ is the maximum in-degree of the DAG. Then there exist integers $L_g,w_g\geq1$ such that, after padding
with identity layers and zero weights if necessary, every latent network
$g_j$, $j\in[m]$, admits a representation in the common dense ReLU class
\[
g_j\in\mathcal F(L_g,w_g),
\qquad j\in[m].
\]
Moreover, the common depth and width may be chosen so that
\begin{align}
L_g
&\leq
L+D_\ell(\epsilon,\eta)
+D_\psi(\delta,U_{\epsilon,\delta,\eta},\zeta)
+C\kappa_m
\left\lceil
\log_2(\deg^{-}_{\max,m}+2)
\right\rceil
+C,
\label{eq:dnn-expanded-g-depth}\\
w_g
&\leq
Ca_{\max,m}(\deg^{-}_{\max,m}+2)
\max\{w,w_{\mathrm{cal}},1\}.
\label{eq:dnn-expanded-g-width}
\end{align}
Here $C>0$ is a universal constant independent of the node index $j$.
\end{theorem}

\begin{proof}
By \eqref{eq:dnn-expanded-regularized-target} and
\eqref{eq:dnn-expanded-base-approx},
\begin{align}
\|\widetilde h_j-F_{A_j,\epsilon}^*\|_\infty
&=
(1-2\epsilon)
\|h_j-F_{A_j}^*\|_\infty
\leq a,
\label{eq:dnn-expanded-regularized-base-error}\\
\widetilde h_j(x)
&\in[\epsilon,1-\epsilon].
\label{eq:dnn-expanded-regularized-range}
\end{align}
Lemma~\ref{lem:dnn-expanded-fixed-modules} therefore gives 
\begin{equation}
\|q_j-\ell\circ\widetilde h_j\|_\infty
\leq\eta.
\label{eq:dnn-expanded-logit-error}
\end{equation}
Since \ref{eq:dnn-expanded-regularized-range} holds, because the logit function
\[
\ell(t)=\log\frac{t}{1-t}
\]
is increasing on $(0,1)$ and
\[
\ell(\epsilon)=-\lambda_\epsilon,
\qquad
\ell(1-\epsilon)=\lambda_\epsilon,
\]
we obtain
\[
-\lambda_\epsilon
\leq
\ell(\widetilde h_j)
\leq
\lambda_\epsilon.
\]
equation~\eqref{eq:dnn-expanded-logit-error} implies
\begin{equation}
-\lambda_\epsilon-\eta
\leq
q_j
\leq
\lambda_\epsilon+\eta.
\label{eq:dnn-expanded-q-range}
\end{equation}

We next show that
\begin{equation}
-\lambda_\epsilon-\eta
\leq
\bar q_j
\leq
\lambda_\epsilon+\eta+r_j\delta.
\label{eq:dnn-expanded-qbar-range}
\end{equation}
For $j\in\mathcal R$, the claim follows directly from
\eqref{eq:dnn-expanded-q-range}.  Now let $j\notin\mathcal R$ and
assume the bound holds for all parents $k\in\mathcal P(j)$. Then by \eqref{eq:dnn-expanded-qbar-range},
\[
\bar q_k
\leq
\lambda_\epsilon+\eta+r_k\delta.
\]
Since
\[
r_j=1+\max_{k\in\mathcal P(j)}r_k,
\]
we have $r_k+1\leq r_j$ for every parent $k$, and therefore
\begin{equation}
    \label{eqn:102}
    \bar q_k+\delta
\leq
\lambda_\epsilon+\eta+r_j\delta.
\end{equation}
Since
\[
q_j\leq \lambda_\epsilon+\eta
\leq
\lambda_\epsilon+\eta+r_j\delta,
\]
and, for every $k\in\mathcal P(j)$ we have that (\ref{eqn:102}) holds, then every term in the maximum defining $\bar q_j$ is bounded above by
$\lambda_\epsilon+\eta+r_j\delta$. Hence
\[
\bar q_j
\leq
\lambda_\epsilon+\eta+r_j\delta.
\]
Moreover, since $\bar q_j\geq q_j$, \eqref{eq:dnn-expanded-q-range}
gives
\[
\bar q_j
\geq
-\lambda_\epsilon-\eta.
\]
This completes the induction and proves
\eqref{eq:dnn-expanded-qbar-range}.

For $j\notin\mathcal R$,
\eqref{eq:dnn-expanded-ordered-q} gives
\[
\bar q_j
\geq
\max_{k\in\mathcal P(j)}\bar q_k+\delta,
\]
and hence $u_j\geq\delta.$  Moreover, by \eqref{eq:dnn-expanded-qbar-range},
\begin{align}
u_j
&=
\bar q_j-
\max_{k\in\mathcal P(j)}\bar q_k
\nonumber\\
&\leq
\lambda_\epsilon+\eta+r_j\delta
-(-\lambda_\epsilon-\eta)
\nonumber\\
&=
2\lambda_\epsilon+2\eta+r_j\delta
\leq
U_{\epsilon,\delta,\eta},
\label{eq:dnn-expanded-u-range}
\end{align}
where the last inequality follows from (\ref{eq:dnn-expanded-u}). Thus
\begin{equation}
u_j\in
[\delta,U_{\epsilon,\delta,\eta}],
\qquad
j\notin\mathcal R.
\label{eq:dnn-expanded-u-domain}
\end{equation}

Set
\[
\bar F_j=\tau(\bar q_j).
\]
Since $\tau\circ\ell$ is the identity on $(0,1)$ and
$\|\tau'\|_\infty\leq1/4$,
\begin{align}
\|\tau(q_j)-F_{A_j,\epsilon}^*\|_\infty
&\leq
\|\tau(q_j)-\tau(\ell\circ\widetilde h_j)\|_\infty
+
\|\widetilde h_j-F_{A_j,\epsilon}^*\|_\infty
\nonumber\\
&\leq
\frac14
\|q_j-\ell\circ\widetilde h_j\|_\infty+a
\nonumber\\
&\leq
\frac{\eta}{4} + a,
\label{eq:dnn-expanded-q-probability}
\end{align}
where the second inequality follows from (\ref{eq:dnn-expanded-regularized-base-error}), and the last from (\ref{eq:dnn-expanded-logit-error}).

For every edge $(k,j)\in E$,
\begin{equation}
F_{A_j,\epsilon}^*
-
F_{A_k,\epsilon}^*
=
(1-2\epsilon)
\left(
F_{A_j}^*-F_{A_k}^*
\right)
\geq0.
\label{eq:dnn-expanded-regularized-order}
\end{equation}
We now claim that  
\begin{equation}
F_{A_j,\epsilon}^*
-a-\frac{\eta}{4}
\leq
\bar F_j
\leq
F_{A_j,\epsilon}^*
+a+\frac{\eta}{4}
+\frac{r_j\delta}{4}.
\label{eq:dnn-expanded-probability-sandwich}
\end{equation}
Since $\bar q_j\geq q_j$ and $\tau$ is increasing, from (\ref{eq:dnn-expanded-q-probability}),
\[
\bar F_j
\geq
\tau(q_j)
\geq
F_{A_j,\epsilon}^*
-a-\frac{\eta}{4}.
\]
We prove the upper bound in (\ref{eq:dnn-expanded-probability-sandwich}) by induction. For a root, the upper bound in  (\ref{eq:dnn-expanded-output-error})  follows from
\eqref{eq:dnn-expanded-q-probability}. For $j\notin\mathcal R$,
monotonicity of $\tau$ gives
\begin{equation}
\bar F_j
=
\max\left\{
\tau(q_j),\,
\max_{k\in\mathcal P(j)}
\tau(\bar q_k+\delta)
\right\}.
\label{eq:dnn-expanded-ordered-probability}
\end{equation}
For every $k\in\mathcal P(j)$, the induction hypothesis for the upper
bound in \eqref{eq:dnn-expanded-probability-sandwich} gives
\[
\tau(\bar q_k)
=
\bar F_k
\leq
F_{A_k,\epsilon}^*
+a+\frac{\eta}{4}
+\frac{r_k\delta}{4}.
\]
Since $\tau$ is $1/4$-Lipschitz,
\eqref{eq:dnn-expanded-regularized-order} holds, and
$r_k+1\leq r_j$, we obtain
\begin{align*}
\tau(\bar q_k+\delta)
&\leq
\tau(\bar q_k)+\frac{\delta}{4}\\
&\leq
F_{A_k,\epsilon}^*
+a+\frac{\eta}{4}
+\frac{(r_k+1)\delta}{4}\\
&\leq
F_{A_j,\epsilon}^*
+a+\frac{\eta}{4}
+\frac{r_j\delta}{4}.
\end{align*}
Together with \eqref{eq:dnn-expanded-q-probability}, this proves
\eqref{eq:dnn-expanded-probability-sandwich}. Hence
\begin{equation}
\|\bar F_j-F_{A_j,\epsilon}^*\|_\infty
\leq
a+\frac14(\eta+r_j\delta).
\label{eq:dnn-expanded-ordered-probability-error}
\end{equation}

Let
\[
e_j
=
\|\widetilde H_j(\bm g)-\bar q_j\|_\infty.
\]
For $j\in\mathcal R$,
\[
e_j=0.
\]
For $j\notin\mathcal R$,
\[
\widetilde H_j(\bm g)
=
\max_{k\in\mathcal P(j)}
\widetilde H_k(\bm g)
+
\operatorname{sp}(g_j),
\]
whereas, by \eqref{eq:dnn-expanded-u},
\[
\bar q_j
=
\max_{k\in\mathcal P(j)}
\bar q_k+u_j.
\]

Therefore,
For $j\notin\mathcal R$, using the recursions for
$\widetilde H_j(\bm g)$ and $\bar q_j$, we have pointwise
\begin{align*}
\widetilde H_j(\bm g)-\bar q_j
&=
\left\{
\max_{k\in\mathcal P(j)}
\widetilde H_k(\bm g)
-
\max_{k\in\mathcal P(j)}
\bar q_k
\right\}
+
\{\operatorname{sp}(g_j)-u_j\}.
\end{align*}
Since, for any finite collections $(a_k)$ and $(b_k)$,
\[
\left|
\max_k a_k-\max_k b_k
\right|
\leq
\max_k|a_k-b_k|,
\]
the triangle inequality and the definition
$e_k=\|\widetilde H_k(\bm g)-\bar q_k\|_\infty$ give
\[
e_j
\leq
\max_{k\in\mathcal P(j)}e_k
+
\|\operatorname{sp}(g_j)-u_j\|_\infty.
\]
Therefore,
\begin{align}
e_j
&\leq
\max_{k\in\mathcal P(j)}e_k
+
\|\operatorname{sp}(g_j)-u_j\|_\infty
\nonumber\\
&=
\max_{k\in\mathcal P(j)}e_k
+
\left\|
\operatorname{sp}
\left\{
\mathsf S_{\delta,U_{\epsilon,\delta,\eta},\zeta}(u_j)
\right\}
-
\operatorname{sp}\{\psi(u_j)\}
\right\|_\infty
\nonumber\\
&\leq
\max_{k\in\mathcal P(j)}e_k
+
\left\|
\mathsf S_{\delta,U_{\epsilon,\delta,\eta},\zeta}(u_j)
-\psi(u_j)
\right\|_\infty
\nonumber\\
&\leq
\max_{k\in\mathcal P(j)}e_k+\zeta.
\label{eq:dnn-expanded-recursion-error}
\end{align}
Here, the equality uses \eqref{eq:dnn-expanded-g} and
$\operatorname{sp}\{\psi(u_j)\}=u_j$, since $\psi$ is the inverse of
$\operatorname{sp}$.  The next inequality follows from
$\|\operatorname{sp}'\|_\infty\leq1$, so $\operatorname{sp}$ is
$1$-Lipschitz.  Finally, \eqref{eq:dnn-expanded-u-domain} ensures that
$u_j\in[\delta,U_{\epsilon,\delta,\eta}]$, and hence
\eqref{eq:dnn-expanded-spinv} gives
\[
\left\|
\mathsf S_{\delta,U_{\epsilon,\delta,\eta},\zeta}(u_j)
-\psi(u_j)
\right\|_\infty
\leq\zeta.
\]

 Induction gives
\begin{equation}
e_j\leq r_j\zeta.
\label{eq:dnn-expanded-recursion-error-bound}
\end{equation}
Since
\[
H_j(\bm g)
=
\tau\{\widetilde H_j(\bm g)\},
\qquad
\bar F_j
=
\tau(\bar q_j),
\]
from (\ref{eq:dnn-expanded-recursion-error-bound}), we obtain
\begin{equation}
\|H_j(\bm g)-\bar F_j\|_\infty
\leq
\frac14e_j
\leq
\frac{r_j\zeta}{4}.
\label{eq:dnn-expanded-sp-error}
\end{equation}

Finally,
\begin{equation}
\|F_{A_j,\epsilon}^*-F_{A_j}^*\|_\infty
=
\epsilon
\|1-2F_{A_j}^*\|_\infty
\leq\epsilon.
\label{eq:dnn-expanded-boundary-error}
\end{equation}
Combining
\eqref{eq:dnn-expanded-ordered-probability-error},
\eqref{eq:dnn-expanded-sp-error}, and
\eqref{eq:dnn-expanded-boundary-error} gives
\begin{align*}
\|H_j(\bm g)-F_{A_j}^*\|_\infty
&\leq
\|H_j(\bm g)-\bar F_j\|_\infty
+
\|\bar F_j-F_{A_j,\epsilon}^*\|_\infty\\
&\quad+
\|F_{A_j,\epsilon}^*-F_{A_j}^*\|_\infty\\
&\leq
\epsilon+a+
\frac14
\left(
\eta+r_j\delta+r_j\zeta
\right),
\end{align*}
which proves \eqref{eq:dnn-expanded-output-error}.
Since $r_j\leq\kappa_m$,
\eqref{eq:dnn-expanded-output-error-max} follows.

For $j\in\mathcal R$,
\eqref{eq:dnn-expanded-g} and
\eqref{eq:dnn-expanded-q-range} give
\[
\|g_j\|_\infty
\leq
\lambda_\epsilon+\eta.
\]
For $j\notin\mathcal R$, Lemma~\ref{lem:dnn-expanded-fixed-modules} and
\eqref{eq:dnn-expanded-u-domain} give
\begin{align*}
\|g_j\|_\infty
&=
\left\|
S_{\delta,U_{\epsilon,\delta,\eta},\zeta}(u_j)
\right\|_\infty\\
&\leq
\|\psi(u_j)\|_\infty
+
\left\|
S_{\delta,U_{\epsilon,\delta,\eta},\zeta}(u_j)
-\psi(u_j)
\right\|_\infty\\
&\leq
\sup_{u\in[\delta,U_{\epsilon,\delta,\eta}]}
|\psi(u)|+\zeta\\
&=
\max\left\{
-\psi(\delta),
\psi(U_{\epsilon,\delta,\eta})
\right\}
+\zeta\\
&\leq
\log(1/\delta)
+
U_{\epsilon,\delta,\eta}
+\zeta,
\end{align*}
where the second inequality uses
$u_j\in[\delta,U_{\epsilon,\delta,\eta}]$ and the uniform approximation
bound
\[
\left\|
S_{\delta,U_{\epsilon,\delta,\eta},\zeta}
-\psi
\right\|_{\infty,[\delta,U_{\epsilon,\delta,\eta}]}
\leq\zeta.
\]
Since $\psi(u)=\log(e^u-1)$ is increasing, the maximum of
$|\psi(u)|$ over the interval is attained at one of its endpoints.
Moreover,
\[
-\psi(\delta)
=
\log\frac{1}{e^\delta-1}
\leq
\log(1/\delta),
\qquad
\psi(U_{\epsilon,\delta,\eta})
=
\log\!\left(e^{U_{\epsilon,\delta,\eta}}-1\right)
\leq
U_{\epsilon,\delta,\eta},
\]
where we used $e^\delta-1\geq\delta$ and
$e^{U_{\epsilon,\delta,\eta}}-1
\leq e^{U_{\epsilon,\delta,\eta}}$.
Setting
\[
U_{\epsilon,\delta,\eta}
=
2\lambda_\epsilon+2\eta+\delta\kappa_m
\]
proves \eqref{eq:dnn-expanded-envelope}.

It remains to establish the architecture bounds. Indeed,
\[
q_j=\mathsf L_{\epsilon,\eta}\circ\widetilde h_j,
\qquad
\widetilde h_j=\epsilon+(1-2\epsilon)h_j.
\]
The affine transformation defining $\widetilde h_j$ does not change the
architecture order of $h_j$. Since $h_j$ has depth at most $L$ and width
at most $w$, while $\mathsf L_{\epsilon,\eta}$ has depth
$D_\ell(\epsilon,\eta)$ and width at most $w_{\mathrm{cal}}$,
Lemma~\ref{lem:relu-network-operations} yields
\[
q_j
\in
\mathcal F\left(
L+D_\ell(\epsilon,\eta)+C,\,
C\max\{w,w_{\mathrm{cal}},1\}
\right).
\]

Fix $j\in[m]$. The networks
\[
q_s,
\qquad
s\in\overline{\mathcal A}(j),
\]
can be evaluated in parallel with width at most
\begin{equation}
Ca_j\max\{w,w_{\mathrm{cal}},1\}.
\label{eq:dnn-expanded-ancestor-width}
\end{equation}
At each DAG level, every maximum in
\eqref{eq:dnn-expanded-ordered-q} contains at most
$\deg^{-}_{\max,m}+1$ inputs. Lemma~\ref{lem:relu-network-operations}
therefore adds at most
\[
C
\left\lceil
\log_2(\deg^{-}_{\max,m}+2)
\right\rceil
\]
hidden layers per graph level and increases the common width to at most
\[
Ca_j(\deg^{-}_{\max,m}+2)
\max\{w,w_{\mathrm{cal}},1\}.
\]
Iterating this construction over the $r_j$ graph levels yields
\begin{align}
L_{\bar q_j}
&\leq
L+D_\ell(\epsilon,\eta)
+
Cr_j
\left\lceil
\log_2(\deg^{-}_{\max,m}+2)
\right\rceil
+C,
\label{eq:dnn-expanded-qbar-depth}\\
w_{\bar q_j}
&\leq
Ca_j(\deg^{-}_{\max,m}+2)
\max\{w,w_{\mathrm{cal}},1\}.
\label{eq:dnn-expanded-qbar-width}
\end{align}

For $j\notin\mathcal R$, Lemma~\ref{lem:relu-network-operations}
shows that computing $\max_{k\in\mathcal P(j)}\bar q_k$ adds at most
\[
C\left\lceil\log_2(\deg^{-}_{\max,m}+2)\right\rceil
\]
hidden layers. Since $r_j\geq1$, this additional depth is absorbed into
the term in \eqref{eq:dnn-expanded-qbar-depth}. The subtraction defining
$u_j$ is then affine. For $j\notin\mathcal R$,
composition with
$\mathsf S_{\delta,U_{\epsilon,\delta,\eta},\zeta}$ gives
\begin{align*}
L_{g_j}
&\leq
L+D_\ell(\epsilon,\eta)
+D_\psi(\delta,U_{\epsilon,\delta,\eta},\zeta)\\
&\quad+
Cr_j
\left\lceil
\log_2(\deg^{-}_{\max,m}+2)
\right\rceil
+C,\\
w_{g_j}
&\leq
Ca_j(\deg^{-}_{\max,m}+2)
\max\{w,w_{\mathrm{cal}},1\}.
\end{align*}
The same bounds hold for roots after padding. Taking the maximum over
$j$, using
\[
r_j\leq\kappa_m,
\qquad
a_j\leq a_{\max,m},
\]
gives
\eqref{eq:dnn-expanded-g-depth}--%
\eqref{eq:dnn-expanded-g-width}.
\end{proof}

\begin{definition}[Constructed functions and clipped dense ReLU class]
\label{def:dnn-expanded-calibrated-class}
Let $\mathcal B$ be a class of $[0,1]$-valued ReLU networks. For fixed
\[
0<\epsilon<\frac14,\qquad
0<\eta<1,\qquad
0<\delta\leq1,\qquad
0<\zeta<1,
\]
 let
$\mathcal T_j(\mathcal B;\epsilon,\eta,\delta,\zeta)$ be the set of all
functions $g_j$ obtained from
Definition~\ref{def:dnn-expanded-latent-construction} with
\[
h_s\in\mathcal B,
\qquad
s\in\overline{\mathcal A}(j),
\]
and define
\begin{equation}
\mathcal T_m(\mathcal B;\epsilon,\eta,\delta,\zeta)
=
\bigcup_{j=1}^m
\mathcal T_j(\mathcal B;\epsilon,\eta,\delta,\zeta).
\label{eq:dnn-expanded-calibrated-class}
\end{equation}

For $R\geq1$, define
\begin{equation}
\pi_R(t)
=
-R+\rho(t+R)-\rho(t-R)
=
\max\{-R,\min(t,R)\},
\label{eq:dnn-expanded-latent-clipping-map}
\end{equation}
and, for integers $L,w\geq1$, let
\begin{equation}
\mathcal F^{\mathrm{cl}}(L,w,R)
=
\{\pi_R\circ f:f\in\mathcal F(L,w)\}.
\label{eq:dnn-expanded-ambient-class}
\end{equation}
All weights of $f$ are trainable, whereas the final clipping map
$\pi_R$ is fixed.
\end{definition}

\begin{lemma}[Embedding and entropy of the ambient class]
\label{lem:dnn-expanded-calibrated-entropy}
Suppose that every $h\in\mathcal B$ has depth at most
$L_{\mathcal B}$ and width at most $w_{\mathcal B}$.  With the notation of
Theorem~\ref{theo1}, \eqref{eq:dnn-expanded-max-indegree},
\eqref{eq:dnn-expanded-u}, and
Lemma~\ref{lem:dnn-expanded-fixed-modules},  define
\begin{align}
L_{\mathrm{amb}}
&=
\left\lceil
L_{\mathcal B}
+D_{\ell}(\epsilon,\eta)
+D_{\psi}(\delta,U_{\epsilon,\delta,\eta},\zeta)
+C\kappa_m
\left\lceil
\log_2(\deg^{-}_{\max,m}+2)
\right\rceil
+C
\right\rceil,
\label{eq:dnn-expanded-ambient-depth}\\
w_{\mathrm{amb}}
&=
\left\lceil
Ca_{\max,m}(\deg^{-}_{\max,m}+2)
\max\{w_{\mathcal B},w_{\mathrm{cal}},1\}
\right\rceil,
\label{eq:dnn-expanded-ambient-width}
\end{align}
and
\begin{equation}
R_{\mathrm{amb}}
=
\max\left[
1,\,
C\left\{
\lambda_\epsilon+\eta+\delta\kappa_m
+\log(1/\delta)+\zeta+1
\right\}
\right].
\label{eq:dnn-expanded-ambient-radius}
\end{equation}
Then
\begin{equation}
\mathcal T_m(\mathcal B;\epsilon,\eta,\delta,\zeta)
\subseteq
\mathcal F^{\mathrm{cl}}
(L_{\mathrm{amb}},w_{\mathrm{amb}},R_{\mathrm{amb}}).
\label{eq:dnn-expanded-ambient-embedding}
\end{equation}

Moreover, uniformly over probability measures $Q$ on $[0,1]^d$ and
$0<u\leq R_{\mathrm{amb}}$,
\begin{equation}
\log N\left(
u,\,
\mathcal F^{\mathrm{cl}}
(L_{\mathrm{amb}},w_{\mathrm{amb}},R_{\mathrm{amb}}),
\|\cdot\|_{L^2(Q)}
\right)
\leq
v_{\mathrm{amb}}
\log\left(
\frac{CR_{\mathrm{amb}}}{u}
\right),
\label{eq:dnn-expanded-ambient-entropy}
\end{equation}
where
\[
\overline L=L_{\mathrm{amb}}+1,
\qquad
\overline w=\max\{w_{\mathrm{amb}},2\},
\]
and
\begin{align}
W_{\mathrm{amb}}
&=
(d+1)\overline w
+(\overline L-1)(\overline w+1)\overline w
+(\overline w+1),
\label{eq:dnn-expanded-ambient-parameters}\\
U_{\mathrm{amb}}
&=
\overline L\,\overline w+1,
\qquad
v_{\mathrm{amb}}
=
C\overline L W_{\mathrm{amb}}
\log(eU_{\mathrm{amb}}).
\label{eq:dnn-expanded-ambient-vc-index}
\end{align}
\end{lemma}

\begin{proof}
By Theorem~\ref{theo:dnn-expanded-probability-approx}, every
\[
g\in
\mathcal T_m(\mathcal B;\epsilon,\eta,\delta,\zeta)
\]
admits a representation in
\[
\mathcal F(L_{\mathrm{amb}},w_{\mathrm{amb}})
\]
and satisfies
\[
\|g\|_\infty
\leq
R_{\mathrm{amb}}.
\]
Hence, by \eqref{eq:dnn-expanded-latent-clipping-map},
\[
\pi_{R_{\mathrm{amb}}}\{g(x)\}
=
g(x),
\qquad
x\in[0,1]^d,
\]
and therefore
\[
g
=
\pi_{R_{\mathrm{amb}}}\circ g
\in
\mathcal F^{\mathrm{cl}}
(L_{\mathrm{amb}},w_{\mathrm{amb}},R_{\mathrm{amb}}).
\]
This proves \eqref{eq:dnn-expanded-ambient-embedding}.

The representation
\[
\pi_R(t)
=
-R+\rho(t+R)-\rho(t-R)
\]
adds one hidden layer with two ReLU units. After padding the hidden
layers to common width
\[
\overline w
=
\max\{w_{\mathrm{amb}},2\},
\]
the clipped class is therefore contained in a dense ReLU architecture
with
\[
\overline L=L_{\mathrm{amb}}+1
\]
hidden layers and width $\overline w$. Its number of parameters is at
most
\begin{align*}
&(d+1)\overline w
+
\sum_{\ell=2}^{\overline L}
(\overline w+1)\overline w
+
(\overline w+1)\\
&\qquad=
(d+1)\overline w
+
(\overline L-1)(\overline w+1)\overline w
+
(\overline w+1)
=
W_{\mathrm{amb}},
\end{align*}
and its number of computational units is at most
\begin{equation}
\overline L\,\overline w+1
=
U_{\mathrm{amb}}.
\label{eq:dnn-expanded-ambient-unit-count}
\end{equation}

Since the clipping weights are fixed,
$\mathcal F^{\mathrm{cl}}$ is a subclass of this dense ReLU
architecture. Hence \citet[Theorem~7]{bartlett2019nearly} gives
\begin{align}
\operatorname{Pdim}\left\{
\mathcal F^{\mathrm{cl}}
(L_{\mathrm{amb}},w_{\mathrm{amb}},R_{\mathrm{amb}})
\right\}
&\leq
C\overline L W_{\mathrm{amb}}
\log(eU_{\mathrm{amb}})
\nonumber\\
&=
v_{\mathrm{amb}}.
\label{eq:dnn-expanded-ambient-pdim}
\end{align}

Set
\[
\mathcal G
=
\left\{
\frac{f}{R_{\mathrm{amb}}}:
f\in
\mathcal F^{\mathrm{cl}}
(L_{\mathrm{amb}},w_{\mathrm{amb}},R_{\mathrm{amb}})
\right\}.
\]
By construction,
\begin{equation}
\sup_{g\in\mathcal G}\|g\|_\infty
\leq1,
\qquad
\operatorname{Pdim}(\mathcal G)
=
\operatorname{Pdim}
\left\{
\mathcal F^{\mathrm{cl}}
(L_{\mathrm{amb}},w_{\mathrm{amb}},R_{\mathrm{amb}})
\right\}
\leq
v_{\mathrm{amb}}.
\label{eq:dnn-expanded-normalized-class}
\end{equation}
The bounded VC-subgraph covering inequality
\citet[Theorem~2.6.7]{vandervaart1996weak} therefore gives, for
$0<\varepsilon\leq1$,
\begin{equation}
\log N\left(
\varepsilon,\mathcal G,\|\cdot\|_{L^2(Q)}
\right)
\leq
v_{\mathrm{amb}}
\log\left(\frac{C}{\varepsilon}\right).
\label{eq:dnn-expanded-normalized-cover}
\end{equation}
Taking
\[
\varepsilon
=
\frac{u}{R_{\mathrm{amb}}}
\]
and using
\[
N\left(
u,\,
\mathcal F^{\mathrm{cl}}
(L_{\mathrm{amb}},w_{\mathrm{amb}},R_{\mathrm{amb}}),
\|\cdot\|_{L^2(Q)}
\right)
=
N\left(
\frac{u}{R_{\mathrm{amb}}},
\mathcal G,
\|\cdot\|_{L^2(Q)}
\right)
\]
proves \eqref{eq:dnn-expanded-ambient-entropy}.
\end{proof}

\begin{corollary}[Detailed statement of Corollary~\ref{cor1}]
\label{appx:cor1-expanded}
Suppose that
\[
F_{A_j}^*\in\mathcal H(l_j,\mathcal P_j),
\qquad
j=1,\ldots,m,
\]
where $l_j\in\mathbb N$ and
$\mathcal P_j\subset[1,\infty)\times\mathbb N$. Each constituent
function $\mu$ may have smoothness
$p_\mu=q_\mu+u_\mu$, where $q_\mu\in\mathbb N_0$ and
$u_\mu\in(0,1]$, and input dimension $d_\mu$, with
$(p_\mu,d_\mu)\in\mathcal P_j$. Assume uniformly over $j\leq m$ that
the hierarchy depths, { constituent counts and ranges, smoothness constants}, the relevant derivative bounds, and the Lipschitz constants
are bounded, and that
\begin{equation}
\max_{j\leq m}\max_{(p,d_0)\in\mathcal P_j}
\max\{p,d_0\}
=
O(1).
\label{eq:dnn-expanded-hc-uniformity}
\end{equation}

Define
\begin{equation}
\phi_n
=
\max_{j\leq m}\max_{(p,d_0)\in\mathcal P_j}
n^{-2p/(2p+d_0)}.
\label{eq:dnn-expanded-phi}
\end{equation}
For sufficiently large constants $c_3,c_4$, set
\begin{equation}
L_{0,n}
=
\left\lceil
c_3
\max_{j\leq m}\max_{(p,d_0)\in\mathcal P_j}
n^{d_0/[2(2p+d_0)]}\log(en)
\right\rceil,
\qquad
w_{0,n}
=
\lceil c_4\rceil.
\label{eq:dnn-expanded-deep}
\end{equation}

For a sufficiently large positive constant $C_0$, set
\begin{equation}
\alpha_n=C_0\sqrt{\phi_n}.
\label{eq:dnn-expanded-alpha}
\end{equation}
For all sufficiently large $n$ such that $\alpha_n<1/4$, let
\begin{equation}
\epsilon_n=\eta_n=\alpha_n,
\qquad
\delta_n=\zeta_n=\frac{\alpha_n}{\kappa_m},
\label{eq:dnn-expanded-tuning}
\end{equation}
and define
\begin{equation}
U_n
=
2\log\frac{1-\alpha_n}{\alpha_n}
+3\alpha_n.
\label{eq:dnn-expanded-U}
\end{equation}

With the notation of Theorem~\ref{theo1},
\eqref{eq:dnn-expanded-max-indegree},
\eqref{eq:dnn-expanded-deep},
\eqref{eq:dnn-expanded-alpha}, and
\eqref{eq:dnn-expanded-U}, and with
$D_{\ell}$ and $D_{\psi}$ defined in
\eqref{eq:dnn-expanded-logit-depth} and
\eqref{eq:dnn-expanded-spinv-depth}, respectively, set
\begin{align}
L_n
&=
\left\lceil
L_{0,n}+1
+D_\ell(\alpha_n,\alpha_n)
+D_\psi\left(
\frac{\alpha_n}{\kappa_m},
U_n,
\frac{\alpha_n}{\kappa_m}
\right)
\right.
\nonumber\\
&\hspace{35mm}\left.
+C\kappa_m
\left\lceil
\log_2(\deg^{-}_{\max,m}+2)
\right\rceil
+C
\right\rceil,
\label{eq:dnn-expanded-training-depth}\\
w_n
&=
\left\lceil
Ca_{\max,m}(\deg^{-}_{\max,m}+2)
\max\{w_{0,n},w_{\mathrm{cal}},2\}
\right\rceil,
\label{eq:dnn-expanded-training-width}\\
R_n
&=
\max\left\{
1,\,
C\left[
\log\frac{1-\alpha_n}{\alpha_n}
+\log\frac{\kappa_m}{\alpha_n}
+1
\right]
\right\},
\label{eq:dnn-expanded-training-radius}
\end{align}
for $C>0$ large enough. In the constrained least-squares estimator in equations~(3)--(4),
take the network class to be
\[
\mathcal F^{\mathrm{cl}}(L_n,w_n,R_n).
\]

Let $v_n$ denote the quantity in
\eqref{eq:dnn-expanded-ambient-vc-index} evaluated at
\[
(L_{\mathrm{amb}},w_{\mathrm{amb}})
=
(L_n,w_n).
\]
Then, conditionally on the design, for every
$\gamma\in(0,1/2)$, with probability at least $1-\gamma$,
\begin{equation}
\frac1m\sum_{j=1}^m
\|\widehat F_{A_j}-F_{A_j}^*\|_n^2
\leq
C\left[
\phi_n+
\frac{\bar a_m v_n}{n}
\log\left\{
\frac{
CR_nnm\sqrt{\kappa_m a_{\max,m}}
}{
\gamma
}
\right\}
\right].
\label{eq:dnn-expanded-full-rate}
\end{equation}

If $m$ grows no faster than polynomially in $n$, then
\begin{equation}
\frac1m\sum_{j=1}^m
\|\widehat F_{A_j}-F_{A_j}^*\|_n^2
=
O_{\mathbb P}\left(
\phi_n+
\frac{\bar a_m v_n\log(en)}{n}
\right).
\label{eq:dnn-expanded-simplified-rate}
\end{equation}

For a fixed DAG,
\begin{equation}
L_n
\lesssim
\sqrt{n\phi_n}\log(en),
\qquad
w_n=O(1),
\qquad
R_n\lesssim\log(en),
\qquad
v_n\lesssim n\phi_n\log^3(en),
\label{eq:dnn-expanded-fixed-complexity}
\end{equation}
and hence
\begin{equation}
\frac1m\sum_{j=1}^m
\|\widehat F_{A_j}-F_{A_j}^*\|_n^2
=
O_{\mathbb P}\{\phi_n\log^4(en)\}.
\label{eq:dnn-expanded-fixed-dag-rate}
\end{equation}
The constant in \eqref{eq:dnn-expanded-full-rate} depends only on the
fixed input dimension and the uniform hierarchical-composition bounds,
and not on $n$ or the DAG.
\end{corollary}

\begin{proof}
For brevity, write
\[
D_{\ell,n}
=
D_\ell(\alpha_n,\alpha_n),
\qquad
D_{\psi,n}
=
D_\psi\left(
\frac{\alpha_n}{\kappa_m},
U_n,
\frac{\alpha_n}{\kappa_m}
\right),
\]
and, within this proof, let
\begin{equation}
\mathcal F
=
\mathcal F^{\mathrm{cl}}(L_n,w_n,R_n),
\qquad
A_n=CeR_n.
\label{eq:dnn-expanded-proof-notation}
\end{equation}

\emph{Approximation.}
By \citet[Theorem~3(b)]{kohler2019rate}, for each $j\in[m]$ there exists
\[
f_j\in\mathcal F(L_{0,n},w_{0,n})
\]
here $L_{0,n}$ and $w_{0,n}$ are defined in
\eqref{eq:dnn-expanded-deep}, and such that
\begin{equation}
\|f_j-F_{A_j}^*\|_\infty^2
\leq
C
\max_{(p,d_0)\in\mathcal P_j}
n^{-2p/(2p+d_0)}
\leq
C\phi_n.
\label{eq:dnn-expanded-kohler-base-approx}
\end{equation}

Define
\[
c(t)=\rho(t)-\rho(t-1)
=
\max\{0,\min(t,1)\},
\]
and, only for this proof, set
\begin{equation}
\mathcal B_n
=
\{c\circ f:f\in\mathcal F(L_{0,n},w_{0,n})\}.
\label{eq:dnn-expanded-proof-base-class}
\end{equation}
Since
\[
|c(t)-y|\leq |t-y|,
\qquad
t\in\mathbb R,\quad y\in[0,1],
\]
the networks
\[
h_j=c\circ f_j\in\mathcal B_n
\]
satisfy
\begin{equation}
\max_{j\leq m}
\|h_j-F_{A_j}^*\|_\infty
\leq
\max_{j\leq m}
\|f_j-F_{A_j}^*\|_\infty
\leq
C\sqrt{\phi_n}
\leq
\alpha_n,
\label{eq:dnn-expanded-clipped-base-approx}
\end{equation}
where $C_0$ is chosen sufficiently large.

Applying
Theorem~\ref{theo:dnn-expanded-probability-approx} with
$a=\alpha_n$ and
\[
\epsilon_n=\eta_n=\alpha_n,
\qquad
\delta_n=\zeta_n=\frac{\alpha_n}{\kappa_m},
\]
gives
\begin{align}
\max_{j\leq m}
\|H_j(\bm g)-F_{A_j}^*\|_\infty
&\leq
\alpha_n+\alpha_n
+\frac14
\left(
\alpha_n
+\kappa_m\frac{\alpha_n}{\kappa_m}
+\kappa_m\frac{\alpha_n}{\kappa_m}
\right)
\nonumber\\
&=
\frac{11}{4}\alpha_n.
\label{eq:dnn-expanded-oracle-approx}
\end{align}

Every $h\in\mathcal B_n$ has depth at most $L_{0,n}+1$ and width at
most
\[
\max\{w_{0,n},2\}.
\]

Hence Lemma~\ref{lem:dnn-expanded-calibrated-entropy}, applied with
\[
L_{\mathcal B}=L_{0,n}+1,
\qquad
w_{\mathcal B}=\max\{w_{0,n},2\},
\]
gives
\begin{equation}
    \label{eqn:300}
    L_{\mathrm{amb}}\leq L_n,
\qquad
w_{\mathrm{amb}}\leq w_n,
\end{equation}
because $L_n$ and $w_n$ in
\eqref{eq:dnn-expanded-training-depth}--%
\eqref{eq:dnn-expanded-training-width}
are defined as the corresponding ambient architecture bounds after these substitutions.

Moreover, by the tuning parameters,
\begin{equation}
    \label{eqn:301}
\begin{array}{lll}
 \displaystyle   L_{\epsilon_n}
+\eta_n
+\kappa_m\delta_n
+\log(1/\delta_n)
+\zeta_n+1
&=& \displaystyle  
\log\frac{1-\alpha_n}{\alpha_n}
+2\alpha_n
+\log\frac{\kappa_m}{\alpha_n}
+\frac{\alpha_n}{\kappa_m}
+1\\
&\leq& \displaystyle  
C\left\{
\log\frac{1-\alpha_n}{\alpha_n}
+\log\frac{\kappa_m}{\alpha_n}
+1
\right\},
\end{array}
\end{equation}
for a postive constant $C$. The first equality follows by substituting the tuning choices
$\epsilon_n=\eta_n=\alpha_n$ and
$\delta_n=\zeta_n=\alpha_n/\kappa_m$, together with the definition
$L_{\epsilon}=\log\{(1-\epsilon)/\epsilon\}$ in
\eqref{eq:dnn-expanded-regularized-target}. The second inequality follows
from $\alpha_n<1/4$ and $\kappa_m\geq1$.

%so the ambient radius in
%Lemma~\ref{lem:dnn-expanded-calibrated-entropy} is at most $R_n$.

By \eqref{eq:dnn-expanded-proof-base-class}, the functions
$h_s=c\circ f_s$  satisfy $h_s\in\mathcal B_n$.
Hence, by Definition~\ref{def:dnn-expanded-calibrated-class}, the
corresponding latent functions satisfy
\[
g_j\in
\mathcal T_m(\mathcal B_n;\epsilon_n,\eta_n,\delta_n,\zeta_n),
\qquad j\in[m].
\]
Moreover, from \eqref{eq:dnn-expanded-ambient-radius},  \eqref{eq:dnn-expanded-training-radius}, (\ref{eqn:301}) we obtain that $R_{\mathrm{amb}}\leq R_n$. Hence, from \eqref{eq:dnn-expanded-ambient-embedding} and  definition of
$\mathcal F$ in \eqref{eq:dnn-expanded-proof-notation}, yields
\begin{equation}
g_j
\in
\mathcal T_m(
\mathcal B_n;
\epsilon_n,\eta_n,\delta_n,\zeta_n)
\subseteq
\mathcal F,
\qquad
j\in[m].
\label{eq:dnn-expanded-oracle-membership}
\end{equation}

Hence $\bm g\in\mathcal F^m$. By the definition of
$\Delta_{n,m}$ following Theorem~\ref{theo1}, its infimum over
$\mathcal F^m$ is no larger than the approximation error of this
particular $\bm g$. Therefore, using
\eqref{eq:dnn-expanded-oracle-approx} and
$\alpha_n=C_0\sqrt{\phi_n}$ from \eqref{eq:dnn-expanded-alpha},
\begin{equation}
\Delta_{n,m}^2
\leq
\max_{j\leq m}
\|H_j(\bm g)-F_{A_j}^*\|_\infty^2
\leq
\frac{121}{16}\alpha_n^2
\lesssim
\phi_n.
\label{eq:dnn-expanded-direct-approx}
\end{equation}

\emph{Estimation.}

By \eqref{eq:dnn-expanded-proof-notation}, $\mathcal F
=
\mathcal F^{\mathrm{cl}}(L_n,w_n,R_n).
$  Hence, the argument in 
Lemma~\ref{lem:dnn-expanded-calibrated-entropy}  yields, uniformly over all probability measures
$Q$ on $[0,1]^d$,
\begin{equation}
\log N\left(
u,\mathcal F,\|\cdot\|_{L^2(Q)}
\right)
\leq
v_n\log\left(\frac{A_n}{u}\right),
\qquad
0<u\leq R_n,
\label{eq:dnn-expanded-final-entropy}
\end{equation}
where $v_n$ is defined by \eqref{eq:dnn-expanded-ambient-vc-index}
with $(L_{\mathrm{amb}},w_{\mathrm{amb}})$ replaced by
$(L_n,w_n)$ and $A_n=CeR_n\geq eR_n$.

Thus the approximation condition and the entropy condition in the
explicit bound following Theorem~\ref{theo1} are satisfied by
\eqref{eq:dnn-expanded-direct-approx} and
\eqref{eq:dnn-expanded-final-entropy}, respectively. Applying that
bound with $B_n=R_n$ therefore gives
\eqref{eq:dnn-expanded-full-rate}, after enlarging the universal
constant $C$ if necessary.

Suppose now that
\[
m=O(n^{C_m})
\]
for some fixed $C_m>0$. Then
\begin{equation}
1\leq\bar a_m\leq a_{\max,m}\leq m,
\qquad
1\leq\kappa_m\leq m,
\qquad
0\leq \deg^{-}_{\max,m}\leq m-1.
\label{eq:dnn-expanded-polynomial-dag-bounds}
\end{equation}

By \eqref{eq:dnn-expanded-phi}, \eqref{eq:dnn-expanded-alpha}, and
\eqref{eq:dnn-expanded-hc-uniformity},
\[
\log(1/\alpha_n)=O(\log(en)).
\]
Hence, using \eqref{eq:dnn-expanded-training-radius},
$m=O(n^{C_m})$, (\ref{eq:dnn-expanded-training-radius}), and
\eqref{eq:dnn-expanded-polynomial-dag-bounds},
\begin{equation}
R_n\lesssim\log(en),
\qquad
\log\left\{
\frac{
eA_nnm\sqrt{\kappa_m a_{\max,m}}
}{
\gamma
}
\right\}
=
O_\gamma(\log(en)),
\label{eq:dnn-expanded-log-factor}
\end{equation}
where we also used $A_n=CeR_n$ from
\eqref{eq:dnn-expanded-proof-notation}. Substituting
\eqref{eq:dnn-expanded-log-factor} into
\eqref{eq:dnn-expanded-full-rate} proves
\eqref{eq:dnn-expanded-simplified-rate}.

\emph{Fixed DAG.}
By the definition of $\phi_n$,
\begin{align}
\max_{j\leq m}
\max_{(p,d_0)\in\mathcal P_j}
n^{d_0/[2(2p+d_0)]}
&=
\max_{j\leq m}
\max_{(p,d_0)\in\mathcal P_j}
\left\{
n\,n^{-2p/(2p+d_0)}
\right\}^{1/2}
\nonumber\\
&=
\sqrt{n\phi_n}.
\label{eq:dnn-expanded-depth-identity}
\end{align}
Thus
\begin{equation}
L_{0,n}
\asymp
\sqrt{n\phi_n}\log(en).
\label{eq:dnn-expanded-base-depth-rate}
\end{equation}

Since $p$ is uniformly bounded and $d_0\geq1$, there exists
$c_*>0$ such that, for all sufficiently large $n$,
\begin{equation}
n\phi_n
=
\max_{j\leq m}
\max_{(p,d_0)\in\mathcal P_j}
n^{d_0/(2p+d_0)}
\geq
n^{c_*}.
\label{eq:dnn-expanded-nphi-lower}
\end{equation}
For a fixed DAG,
\[
\kappa_m=O(1),
\qquad
\deg^{-}_{\max,m}=O(1),
\qquad
a_{\max,m}=O(1),
\]
and
\begin{equation}
U_n
+\Lambda_{\epsilon_n}
+\Lambda_{\eta_n}
+\Lambda_{\delta_n}
+\Lambda_{\zeta_n}
\lesssim
\log(en).
\label{eq:dnn-expanded-calibration-log}
\end{equation}
Here $U_n$ is defined in \eqref{eq:dnn-expanded-U},
$(\epsilon_n,\eta_n,\delta_n,\zeta_n)$ are defined in
\eqref{eq:dnn-expanded-tuning}, and
$\Lambda_s=\log(e/s)$ is defined at the beginning of
Section~\ref{sec:nn_general}. Lemma~\ref{lem:dnn-expanded-fixed-modules} therefore gives
\begin{equation}
D_{\ell,n}+D_{\psi,n}
\lesssim
\log^3(en).
\label{eq:dnn-expanded-calibration-depth}
\end{equation}
By \eqref{eq:dnn-expanded-base-depth-rate} and
\eqref{eq:dnn-expanded-nphi-lower},
\[
\log^3(en)=o(L_{0,n}),
\]
and hence
\begin{equation}
D_{\ell,n}+D_{\psi,n}
=
o(L_{0,n}).
\label{eq:dnn-expanded-calibration-negligible}
\end{equation}

It follows from the definitions of $L_n$, $w_n$, and $R_n$ in
\eqref{eq:dnn-expanded-training-depth}--%
\eqref{eq:dnn-expanded-training-radius} that
\begin{equation}
L_n
\lesssim
\sqrt{n\phi_n}\log(en),
\qquad
w_n=O(1),
\qquad
R_n\lesssim\log(en).
\label{eq:dnn-expanded-training-complexity}
\end{equation}
Since $w_n=O(1)$, the parameter and unit counts in
\eqref{eq:dnn-expanded-ambient-parameters}--%
\eqref{eq:dnn-expanded-ambient-vc-index}, replacing $(L_{\mathrm{amb}},w_{\mathrm{amb}})$ with $(L_n,w_n)$,
satisfy
\[
W_{\mathrm{amb}}\lesssim L_n,
\qquad
U_{\mathrm{amb}}\lesssim L_n.
\]
Recall from the statement of Corollary~\ref{appx:cor1-expanded} that
$v_n$ is defined by \eqref{eq:dnn-expanded-ambient-vc-index} replacing 
$(L_{\mathrm{amb}},w_{\mathrm{amb}})$ with $(L_n,w_n)$. Hence, since
$\overline L=L_n+1$, $W_{\mathrm{amb}}\lesssim L_n$, and
$U_{\mathrm{amb}}\lesssim L_n$,
\begin{align}
v_n
&\lesssim
L_n^2\log(eL_n)
\nonumber\\
&\lesssim
n\phi_n\log^3(en).
\label{eq:dnn-expanded-fixed-vc}
\end{align}where the second inequality follows from
\eqref{eq:dnn-expanded-training-complexity}. This proves \eqref{eq:dnn-expanded-fixed-complexity}. Finally, since
$\bar a_m=O(1)$ for a fixed DAG,
\begin{align*}
\phi_n+\frac{\bar a_m v_n\log(en)}{n}
&\lesssim
\phi_n+
\frac{
n\phi_n\log^3(en)\log(en)
}{n}\\
&\lesssim
\phi_n\log^4(en).
\end{align*}
Equation~\eqref{eq:dnn-expanded-fixed-dag-rate} follows from
\eqref{eq:dnn-expanded-simplified-rate}.
\end{proof}

{
\section{Trend filtering and orthogonal series regression}
\label{app:additional-estimators}

Theorems~\ref{theo0} and~\ref{theo1} apply once the approximation
error and the covering entropy of the function class are controlled.
In this section, we derive results for trend filtering and orthogonal
series regression, including concrete rates for both the preliminary
and constrained estimators.

\subsection{Trend filtering}
\label{subsec:trend-filtering-both-estimators}

Trend filtering estimates a univariate regression function by controlling
the total variation of one of its derivatives, allowing the fitted function
to adapt to spatially varying smoothness. We consider the corresponding
bounded function class below.

Let $d=1$ and fix an integer $r\geq1$. We first define the ambient
smoothness class
\begin{equation}
\mathcal S_r
=
\left\{f:[0,1]\to\mathbb R:
\operatorname{TV}\bigl(f^{(r-1)}\bigr)<\infty\right\},
\label{eq:tf-space}
\end{equation}
consisting of functions whose $(r-1)$st derivative has finite total
variation. For $r=1$, $\mathcal S_r$ is understood in the usual pointwise bounded-variation sense; for $r\ge2$, we take $f\in W^{r-1,1}(0,1)$ with weak derivative $f^{(r-1)}\in BV(0,1)$ and use the continuous representative of $f$ on $[0,1]$. 

For $V\geq0$ and $B\geq1$, define the bounded trend-filtering
class
\begin{equation}
\mathcal G_r(V,B)
=
\left\{f\in\mathcal S_r:
\operatorname{TV}\bigl(f^{(r-1)}\bigr)\leq V,
\ \|f\|_\infty\leq B\right\}.
\label{eq:tf-class}
\end{equation}
\subsubsection*{Preliminary estimator}

For $A\in\Lambda$, define
\begin{equation}
\widehat F_A^{\mathrm{pre,TF}}
\in
\arg\min_{f\in\mathcal G_r(V_n,B)}
\sum_{i=1}^n\{z_i^A-f(x_i)\}^2,
\label{eq:pre-tf-estimator}
\end{equation}
and let
\begin{equation}
\phi_{n,\mathrm{TF}}^{\mathrm{pre}}
=
\sup_{A\in\Lambda}
\inf_{g\in\mathcal G_r(V_n,B)}
\|F_A^*-g\|_\infty^2.
\label{eq:pre-tf-approx}
\end{equation}

\begin{corollary}[Preliminary trend-filtering estimator]
\label{cor:pre-tf}
Suppose Assumption~\ref{as0} holds, $B=O(1)$, and
\begin{equation}
(1\vee V_n)^{\frac{2}{2r+1}}
n^{-\frac{2r}{2r+1}}
+
\frac{\log n}{n}
+
\phi_{n,\mathrm{TF}}^{\mathrm{pre}}
=o(1).
\label{eq:pre-tf-vanishing}
\end{equation}
Then
\begin{equation}
\sup_{A\in\Lambda}
\|\widehat F_A^{\mathrm{pre,TF}}-F_A^*\|_n^2
=
O_{\mathbb P}\!\left(
(1\vee V_n)^{\frac{2}{2r+1}}
n^{-\frac{2r}{2r+1}}
+
\frac{\log n}{n}
+
\phi_{n,\mathrm{TF}}^{\mathrm{pre}}
\right).
\label{eq:pre-tf-rate}
\end{equation}
If, for some constant $V<\infty$,
\begin{equation}
\sup_{A\in\Lambda}
\operatorname{TV}\bigl((F_A^*)^{(r-1)}\bigr)
\leq V,
\label{eq:pre-tf-membership}
\end{equation}
then, for $V \leq V_n \leq \bar{V} < \infty$, with $\bar{V}$ fixed, and $B\geq1$,
\begin{equation}
\sup_{A\in\Lambda}
\|\widehat F_A^{\mathrm{pre,TF}}-F_A^*\|_n^2
=
O_{\mathbb P}\!\left(
n^{-\frac{2r}{2r+1}}
\right).
\label{eq:pre-tf-concrete-rate}
\end{equation}
\end{corollary}

\begin{proof}
We use the standard entropy bound
\begin{equation}
\sup_{q\geq1}
\sup_{x_1,\ldots,x_q\in[0,1]}
\log N\!\left(
u,\mathcal G_r(1,1),\|\cdot\|_q
\right)
\leq
C_r u^{-1/r},
\qquad
0<u\leq1,
\label{eq:tf-unit-entropy}
\end{equation}
where $C_r$ depends only on $r$; see the proof of
\citet[Corollary~1]{sadhanala2019additive, sadhanala2019supplement}.

Let
\[
\mathcal D_{r,n}
=
\mathcal G_r(V_n,B)-\mathcal G_r(V_n,B).
\]
For $h\in\mathcal D_{r,n}$,
\[
\operatorname{TV}\bigl(h^{(r-1)}\bigr)\leq2V_n,
\qquad
\|h\|_\infty\leq2B.
\]
Hence, after rescaling \eqref{eq:tf-unit-entropy},
\begin{equation}
\log N(u,\mathcal D_{r,n},\|\cdot\|_n)
\leq
C_r(1+B+V_n)^{1/r}u^{-1/r}.
\label{eq:pre-tf-difference-entropy}
\end{equation}
Therefore, for $t>0$,
\begin{equation}
\int_0^t
\sqrt{
\log N(u,\mathcal D_{r,n},\|\cdot\|_n)
}\,du
\leq
C_r
(1+B+V_n)^{1/(2r)}
t^{1-1/(2r)}.
\label{eq:pre-tf-entropy-integral}
\end{equation}

Substituting \eqref{eq:pre-tf-entropy-integral} into
Theorem~\ref{theo0}, with
$\rho_j=2^{j/2}\eta/\sqrt n$, gives
\[
\frac{\sqrt n}{\eta^2}
\sum_{j=1}^J\frac{\mathcal J(\rho_j)}{2^j}
\lesssim
(1+B+V_n)^{1/(2r)}
n^{1/(4r)}
\eta^{-(2r+1)/(2r)},
\]
where we used
$\sum_{j\ge1}2^{-j(2r+1)/(4r)}<\infty$.
Consequently,
\begin{equation}
\sup_{A\in\Lambda}
\|\widehat F_A^{\mathrm{pre,TF}}-F_A^*\|_n
=
O_{\mathbb P}\!\left(
(1+B+V_n)^{\frac1{2r+1}}
n^{-\frac r{2r+1}}
+
\sqrt{\frac{\log n}{n}}
+
\sqrt{\phi_{n,\mathrm{TF}}^{\mathrm{pre}}}
\right).
\label{eq:pre-tf-norm-rate}
\end{equation}
Since $B=O(1)$, squaring yields
\eqref{eq:pre-tf-rate}. Under \eqref{eq:pre-tf-membership},
$F_A^*\in\mathcal G_r(V_n,B)$ because
$0\leq F_A^*\leq1$. Hence
\[
\phi_{n,\mathrm{TF}}^{\mathrm{pre}}=0.
\]
Since
\[
\frac{\log n}{n}
=
o\!\left(
n^{-2r/(2r+1)}
\right),
\]
\eqref{eq:pre-tf-concrete-rate} follows.
\end{proof}

\subsubsection*{Constrained estimator}

Take $\mathcal F=\mathcal G_r(V_n,B)$ and define
\begin{equation}
\widehat{\bm g}^{\mathrm{TF}}
\in
\arg\min_{\bm g\in\mathcal G_r(V_n,B)^m}
\sum_{j=1}^m\sum_{i=1}^n
\{z_i^j-H_j(\bm g)(x_i)\}^2,
\label{eq:con-tf-latent-estimator}
\end{equation}
with
\begin{equation}
\widehat F_{A_j}^{\mathrm{con,TF}}
=
H_j(\widehat{\bm g}^{\mathrm{TF}}),
\qquad
j=1,\ldots,m.
\label{eq:con-tf-estimator}
\end{equation}
Define
\begin{equation}
\Delta_{n,m}^{\mathrm{TF}}
=
\inf_{\bm g\in\mathcal G_r(V_n,B)^m}
\max_{j\leq m}
\|F_{A_j}^*-H_j(\bm g)\|_\infty.
\label{eq:con-tf-direct-approx}
\end{equation}

\begin{corollary}[Constrained trend-filtering estimator]
\label{cor:con-tf}
Suppose
\begin{align}
(\Delta_{n,m}^{\mathrm{TF}})^2
&+
(1\vee V_n)^{\frac{2}{2r+1}}
n^{-\frac{2r}{2r+1}}
\frac1m
\left\{
\sum_{j=1}^m
a_j^{\frac{2r+1}{4r}}
\kappa_j^{\frac1{4r}}
\right\}^{\frac{4r}{2r+1}}
=o(1).
\label{eq:con-tf-direct-condition}
\end{align}
Then
\begin{align}
\frac1m\sum_{j=1}^m
\|\widehat F_{A_j}^{\mathrm{con,TF}}-F_{A_j}^*\|_n^2
=
O_{\mathbb P}\Bigg\{&
(\Delta_{n,m}^{\mathrm{TF}})^2
\nonumber\\
&+
(1\vee V_n)^{\frac{2}{2r+1}}
n^{-\frac{2r}{2r+1}}
\frac1m
\left(
\sum_{j=1}^m
a_j^{\frac{2r+1}{4r}}
\kappa_j^{\frac1{4r}}
\right)^{\frac{4r}{2r+1}}
\Bigg\}.
\label{eq:con-tf-direct-rate}
\end{align}
\end{corollary}

\begin{proof}
By \eqref{eq:tf-unit-entropy} and scaling,
\begin{equation}
\log N\bigl(
u,\mathcal G_r(V_n,B),\|\cdot\|_n
\bigr)
\leq
C(1\vee V_n)^{1/r}u^{-1/r}.
\label{eq:con-tf-base-entropy}
\end{equation}
Hence, by Lemma~\ref{lem:theo2-entropy-transfer}, for every
$\widetilde{\bm g}\in\mathcal G_r(V_n,B)^m$,
\begin{equation}
\log N\bigl(
u,\mathcal I_{j,n}(\widetilde{\bm g}),\|\cdot\|_n
\bigr)
\leq
C
a_j^{1+\frac1{2r}}
\kappa_j^{\frac1{2r}}
(1\vee V_n)^{1/r}
u^{-1/r}.
\label{eq:con-tf-output-entropy}
\end{equation}
Therefore,
\begin{align}
&\sqrt n
\int_{1/n}^{C_3 2^{k/2}\eta/\sqrt n}
\sqrt{
\log N\bigl(
u,\mathcal I_{j,n}(\widetilde{\bm g}),\|\cdot\|_n
\bigr)}
\,du
\nonumber\\
&\qquad\leq
C
a_j^{\frac{2r+1}{4r}}
\kappa_j^{\frac1{4r}}
(1\vee V_n)^{\frac1{2r}}
n^{\frac1{4r}}
2^{\frac{k(2r-1)}{4r}}
\eta^{1-\frac1{2r}}.
\label{eq:con-tf-entropy-integral}
\end{align}
Substitution into Theorem~\ref{theo1} and
\[
\sum_{k\geq1}
2^{-k(2r+1)/(4r)}
<\infty
\]
give
\begin{align}
\mathbb P\!\left\{
\sum_{j=1}^m\sum_{i=1}^n
\bigl(
F_{A_j}^*(x_i)
-\widehat F_{A_j}^{\mathrm{con,TF}}(x_i)
\bigr)^2
>\eta^2
\right\}
\leq C\Bigg\{&
\frac{m}{\eta^2}
\nonumber\\
&+
(1\vee V_n)^{\frac1{2r}}
n^{\frac1{4r}}
\eta^{-\frac{2r+1}{2r}}
\sum_{j=1}^m
a_j^{\frac{2r+1}{4r}}
\kappa_j^{\frac1{4r}}
\Bigg\}.
\label{eq:con-tf-tail}
\end{align}
Let $T_{n,m}$ denote the stochastic term in
\eqref{eq:con-tf-direct-rate}, and set
\[
R_{n,m}
=
(\Delta_{n,m}^{\mathrm{TF}})^2+T_{n,m}.
\]
By the definition of the infimum, choose a comparator
$\widetilde{\bm g}\in\mathcal G_r(V_n,B)^m$ with squared
approximation error at most $2R_{n,m}$. Taking
\[
\eta^2=MnmR_{n,m}
\]
for a sufficiently large fixed constant $M$, condition
\eqref{eq:con-tf-direct-condition} ensures that the requirements on
$\eta$ in Theorem~\ref{theo1} hold for all sufficiently large $n$.

Since
\[
\sum_{j=1}^m
a_j^{\frac{2r+1}{4r}}
\kappa_j^{\frac1{4r}}
\ge m,
\]
we have
\[
nT_{n,m}\ge n^{1/(2r+1)}.
\]
Hence the first term in \eqref{eq:con-tf-tail} satisfies
\[
\frac{m}{\eta^2}
\le
M^{-1}n^{-1/(2r+1)},
\]
while direct substitution into the second term gives an upper bound
of order $M^{-(2r+1)/(4r)}$. Taking $M$ sufficiently large proves
\eqref{eq:con-tf-direct-rate}.
\end{proof}

To simplify the derivation of the concrete constrained rates, we use
the following bounded latent representation condition.

\begin{assumption}[Bounded latent representation]
\label{as:tf-os-latent}
For every $j\leq m$ and $x\in[0,1]^d$,
\begin{equation}
0<F_{A_j}^*(x)<1,
\label{eq:tf-os-interior}
\end{equation}
and, for every edge $(k,j)\in E$,
\begin{equation}
F_{A_k}^*(x)<F_{A_j}^*(x).
\label{eq:tf-os-strict-inclusion}
\end{equation}
For $j\in\mathcal R$, define
\begin{equation}
g_j^*
=
\tau^{-1}\circ F_{A_j}^*,
\label{eq:tf-os-root-latent}
\end{equation}
and, for $j\notin\mathcal R$,
\begin{equation}
g_j^*(x)
=
\log\!\left[
\exp\!\left\{
\tau^{-1}\!\bigl(F_{A_j}^*(x)\bigr)
-
\max_{k\in\mathcal P(j)}
\tau^{-1}\!\bigl(F_{A_k}^*(x)\bigr)
\right\}
-1
\right].
\label{eq:tf-os-nonroot-latent}
\end{equation}
Suppose
\begin{equation}
\max_{j\leq m}\|g_j^*\|_\infty\leq C_0
\label{eq:tf-os-latent-envelope}
\end{equation}
for a constant $C<\infty$.
\end{assumption}

For $j\notin\mathcal R$,
\begin{equation}
\log(1+e^{g_j^*})
=
\tau^{-1}(F_{A_j}^*)
-
\max_{k\in\mathcal P(j)}
\tau^{-1}(F_{A_k}^*).
\label{eq:tf-os-latent-increment}
\end{equation}
Together with \eqref{eq:tf-os-root-latent}, induction over a
topological ordering of the DAG gives
\begin{equation}
F_{A_j}^*=H_j(\bm g^*),
\qquad
j=1,\ldots,m.
\label{eq:tf-os-exact-representation}
\end{equation}

\subsubsection*{Constrained estimator: concrete trend-filtering rate}

\begin{corollary}[Constrained trend filtering: bounded latent functions]
\label{cor:con-tf-concrete}
Suppose Assumption~\ref{as:tf-os-latent} holds and, for a constant
$V<\infty$,
\begin{equation}
\max_{j\leq m}
\operatorname{TV}\bigl((g_j^*)^{(r-1)}\bigr)
\leq V.
\label{eq:con-tf-latent-membership}
\end{equation}
Set $V_n=V$ and take a fixed $B\geq \max\{1,C\}$. If
\begin{equation}
n^{-\frac{2r}{2r+1}}
\frac1m
\left\{
\sum_{j=1}^m
a_j^{\frac{2r+1}{4r}}
\kappa_j^{\frac1{4r}}
\right\}^{\frac{4r}{2r+1}}
=o(1),
\label{eq:con-tf-concrete-condition}
\end{equation}
then
\begin{equation}
\frac1m\sum_{j=1}^m
\|\widehat F_{A_j}^{\mathrm{con,TF}}-F_{A_j}^*\|_n^2
=
O_{\mathbb P}\!\left[
n^{-\frac{2r}{2r+1}}
\frac1m
\left\{
\sum_{j=1}^m
a_j^{\frac{2r+1}{4r}}
\kappa_j^{\frac1{4r}}
\right\}^{\frac{4r}{2r+1}}
\right].
\label{eq:con-tf-concrete-rate}
\end{equation}
For a fixed DAG, \eqref{eq:con-tf-concrete-rate} reduces to
$O_{\mathbb P}(n^{-2r/(2r+1)})$.
\end{corollary}

\begin{proof}
By Assumption~\ref{as:tf-os-latent},
\eqref{eq:tf-os-exact-representation}, and
\eqref{eq:con-tf-latent-membership},
\begin{equation}
\bm g^*\in\mathcal G_r(V,B)^m,
\qquad
\Delta_{n,m}^{\mathrm{TF}}=0.
\label{eq:con-tf-zero-approx}
\end{equation}
The result follows from Corollary~\ref{cor:con-tf}. For a fixed DAG,
$m$, $a_j$, and $\kappa_j$ are constants.
\end{proof}

\subsection{Orthogonal series regression}
\label{subsec:orthogonal-series-both-estimators}

Let $\{\varphi_k\}_{k\ge1}$ be a fixed orthonormal basis of
$L^2([0,1]^d)$, with the basis functions defined pointwise on $[0,1]^d$. For an integer $J\geq1$, define
\begin{equation}
\mathcal L_J
=
\operatorname{span}\{\varphi_1,\ldots,\varphi_J\},
\qquad
\mathcal O_J(B)
=
\{f\in\mathcal L_J:\|f\|_\infty\leq B\},
\label{eq:os-class}
\end{equation}
where $B\geq1$ is fixed. The empirical quotient of $\mathcal L_J$ has dimension at most $J$.
Hence, for $0<u\leq B$, the standard volumetric bound gives
\begin{equation}
N(u,\mathcal O_J(B),\|\cdot\|_n)
\leq
(1+2B/u)^J.
\label{eq:os-base-cover}
\end{equation}

\subsubsection*{Preliminary estimator}

For $A\in\Lambda$, define
\begin{equation}
\widehat F_{A,J}^{\mathrm{pre,OS}}
\in
\arg\min_{f\in\mathcal O_J(B)}
\sum_{i=1}^n\{z_i^A-f(x_i)\}^2
\label{eq:pre-os-estimator}
\end{equation}
and let
\begin{equation}
\phi_{n,J}^{\mathrm{pre,OS}}
=
\sup_{A\in\Lambda}
\inf_{g\in\mathcal O_J(B)}
\|F_A^*-g\|_\infty^2.
\label{eq:pre-os-approx}
\end{equation}

\begin{corollary}[Preliminary orthogonal-series estimator]
\label{cor:pre-os}
Suppose Assumption~\ref{as0} holds, $J=J_n$, and
\begin{equation}
\frac{J+\log n}{n}
+
\phi_{n,J}^{\mathrm{pre,OS}}
=o(1).
\label{eq:pre-os-vanishing}
\end{equation}
Then
\begin{equation}
\sup_{A\in\Lambda}
\|\widehat F_{A,J}^{\mathrm{pre,OS}}-F_A^*\|_n^2
=
O_{\mathbb P}\!\left(
\frac{J+\log n}{n}
+
\phi_{n,J}^{\mathrm{pre,OS}}
\right).
\label{eq:pre-os-rate}
\end{equation}
If, for some $s>0$,
\begin{equation}
F_A^*
=
\sum_{k=1}^\infty\theta_{A,k}\varphi_k,
\qquad
\sup_{A\in\Lambda}
\sum_{k=1}^\infty
k^{s/d}|\theta_{A,k}|\,\|\varphi_k\|_\infty
\leq C_s,
\label{eq:pre-os-coefficient-condition}
\end{equation}
then, for $B\geq1+C_s$ and
$J_n\asymp n^{d/(2s+d)}$,
\begin{equation}
\sup_{A\in\Lambda}
\|\widehat F_{A,J_n}^{\mathrm{pre,OS}}-F_A^*\|_n^2
=
O_{\mathbb P}\!\left(
n^{-\frac{2s}{2s+d}}
\right).
\label{eq:pre-os-concrete-rate}
\end{equation}
\end{corollary}

\begin{proof}
For $t>0$, let
\[
\mathcal D_J(t)
=
\{f-g:f,g\in\mathcal O_J(B),\ \|f-g\|_n\leq t\}.
\]
The volumetric bound in the empirical quotient space gives
\[
N(u,\mathcal D_J(t),\|\cdot\|_n)
\leq
(1+2t/u)^J,
\]
and hence
\begin{equation}
\int_0^t
\sqrt{\log N(u,\mathcal D_J(t),\|\cdot\|_n)}\,du
\leq
Ct\sqrt J.
\label{eq:pre-os-entropy-integral}
\end{equation}
For the shell radii
$\rho_\ell=2^{\ell/2}\eta/\sqrt n$ in
Theorem~\ref{theo0}, \eqref{eq:pre-os-entropy-integral} gives
\[
\frac{\sqrt n}{\eta^2}
\sum_{\ell}
\frac{\mathcal J(\rho_\ell)}{2^\ell}
\lesssim
\frac{\sqrt J}{\eta}
\sum_{\ell\geq1}2^{-\ell/2}
\lesssim
\frac{\sqrt J}{\eta}.
\]
Taking $\eta$ to be a sufficiently large constant multiple of
\[
\sqrt{J+\log n}
+
\sqrt{n\phi_{n,J}^{\mathrm{pre,OS}}}
\]
in Theorem~\ref{theo0} yields
\[
\sup_{A\in\Lambda}
\|\widehat F_{A,J}^{\mathrm{pre,OS}}-F_A^*\|_n
=
O_{\mathbb P}\!\left(
\sqrt{\frac{J+\log n}{n}}
+
\sqrt{\phi_{n,J}^{\mathrm{pre,OS}}}
\right),
\]
which proves \eqref{eq:pre-os-rate}.

Under \eqref{eq:pre-os-coefficient-condition}, let
\[
G_{A,J}
=
\sum_{k=1}^J\theta_{A,k}\varphi_k.
\]
Then
\begin{equation}
\|F_A^*-G_{A,J}\|_\infty
\leq
J^{-s/d}
\sum_{k>J}
k^{s/d}|\theta_{A,k}|\,\|\varphi_k\|_\infty
\leq
C_sJ^{-s/d}.
\label{eq:pre-os-tail}
\end{equation}
Since $\|F_A^*\|_\infty\leq1$,
$\|G_{A,J}\|_\infty\leq1+C_s\leq B$, and therefore
\[
\phi_{n,J}^{\mathrm{pre,OS}}
\leq
C_s^2J^{-2s/d}.
\]
Balancing $J/n$ and $J^{-2s/d}$ proves
\eqref{eq:pre-os-concrete-rate}.
\end{proof}

\subsubsection*{Constrained estimator}

Take $\mathcal F=\mathcal O_J(B)$ and define
\begin{equation}
\widehat{\bm g}^{\mathrm{OS}}
\in
\arg\min_{\bm g\in\mathcal O_J(B)^m}
\sum_{j=1}^m\sum_{i=1}^n
\{z_i^j-H_j(\bm g)(x_i)\}^2,
\label{eq:con-os-latent-estimator}
\end{equation}
with
\begin{equation}
\widehat F_{A_j}^{\mathrm{con,OS}}
=
H_j(\widehat{\bm g}^{\mathrm{OS}}),
\qquad
j=1,\ldots,m.
\label{eq:con-os-estimator}
\end{equation}
Define
\begin{equation}
\Delta_{n,m}^{\mathrm{OS}}
=
\inf_{\bm g\in\mathcal O_J(B)^m}
\max_{j\leq m}
\|F_{A_j}^*-H_j(\bm g)\|_\infty.
\label{eq:con-os-direct-approx}
\end{equation}

\begin{corollary}[Constrained orthogonal-series estimator]
\label{cor:con-os}
For every $\delta\in(0,1/2)$, with probability at least $1-\delta$,
\begin{equation}
\frac1m\sum_{j=1}^m
\|\widehat F_{A_j}^{\mathrm{con,OS}}-F_{A_j}^*\|_n^2
\leq
C\left\{
(\Delta_{n,m}^{\mathrm{OS}})^2
+
\frac{\bar a_mJ}{n}
\log\!\left(
\frac{eBnm\sqrt{\kappa_ma_{\max,m}}}{\delta}
\right)
\right\}.
\label{eq:con-os-direct-rate}
\end{equation}
\end{corollary}

\begin{proof}
By \eqref{eq:os-base-cover},
\begin{equation}
\log N(u,\mathcal O_J(B),\|\cdot\|_n)
\leq
J\log(3B/u),
\qquad
0<u\leq B.
\label{eq:con-os-log-cover}
\end{equation}
Thus the covering condition following Theorem~\ref{theo1} holds with
$v_n=J$ and $A_n=3B$. Substitution into
\eqref{eq:growing-m-main} proves \eqref{eq:con-os-direct-rate}.
\end{proof}

\subsubsection*{Constrained estimator: concrete orthogonal-series rate}

\begin{corollary}[Constrained orthogonal series: bounded latent functions]
\label{cor:con-os-concrete}
Suppose Assumption~\ref{as:tf-os-latent} holds and, for some $s>0$,
\begin{equation}
g_j^*
=
\sum_{k=1}^\infty\gamma_{j,k}\varphi_k,
\qquad
\max_{j\leq m}
\sum_{k=1}^\infty
k^{s/d}|\gamma_{j,k}|\,\|\varphi_k\|_\infty
\leq C_s.
\label{eq:con-os-coefficient-condition}
\end{equation}
If $B\ge \max\{1,C_0+C_s\}$ is fixed, then
\begin{equation}
\frac1m\sum_{j=1}^m
\|\widehat F_{A_j}^{\mathrm{con,OS}}-F_{A_j}^*\|_n^2
=
O_{\mathbb P}\!\left\{
\kappa_m^2J^{-2s/d}
+
\frac{\bar a_mJ}{n}
\log\!\left(
eBnm\sqrt{\kappa_ma_{\max,m}}
\right)
\right\}.
\label{eq:con-os-concrete-oracle-rate}
\end{equation}
Choosing the integer truncation level
\[
J_n
=
\max\left\{
1,\,
\left\lceil
\left[
\frac{\kappa_m^2 n}
{\bar a_m
\log\!\left(
eBnm\sqrt{\kappa_ma_{\max,m}}
\right)}
\right]^{\frac{d}{2s+d}}
\right\rceil
\right\},
\]
we further obtain
\begin{equation}
\frac1m\sum_{j=1}^m
\|\widehat F_{A_j}^{\mathrm{con,OS}}-F_{A_j}^*\|_n^2
=
O_{\mathbb P}\!\left\{
\kappa_m^{\frac{2d}{2s+d}}
\left[
\frac{\bar a_m}{n}
\log\!\left(
eBnm\sqrt{\kappa_ma_{\max,m}}
\right)
\right]^{\frac{2s}{2s+d}}
+
\frac{\bar a_m}{n}
\log\!\left(
eBnm\sqrt{\kappa_ma_{\max,m}}
\right)
\right\}.
\label{eq:con-os-balanced-rate}
\end{equation}

In particular, if
$\kappa_m,\bar a_m,a_{\max,m}=O(1)$
and $m$ grows at most polynomially in $n$, then, ignoring logarithmic
factors,
\[
\frac1m\sum_{j=1}^m
\|\widehat F_{A_j}^{\mathrm{con,OS}}-F_{A_j}^*\|_n^2
=
O_{\mathbb P}\!\left(
n^{-\frac{2s}{2s+d}}
\right).
\]
\end{corollary}

\begin{proof}
Let
\[
g_{j,J}
=
\sum_{k=1}^J\gamma_{j,k}\varphi_k.
\]
By \eqref{eq:con-os-coefficient-condition},
\begin{equation}
\max_{j\leq m}
\|g_j^*-g_{j,J}\|_\infty
\leq
C_sJ^{-s/d}.
\label{eq:con-os-latent-tail}
\end{equation}
Moreover,
\[
\|g_{j,J}\|_\infty
\leq
C_0+C_s
\leq B,
\]
so $\bm g_J\in\mathcal O_J(B)^m$. By
\eqref{eq:tf-os-exact-representation} and
\eqref{eq:theo2-pathwise-pointwise},
\begin{equation}
\max_{j\leq m}
\|H_j(\bm g_J)-F_{A_j}^*\|_\infty
\leq
C\kappa_mJ^{-s/d}.
\label{eq:con-os-joint-approx}
\end{equation}
Hence
\begin{equation}
(\Delta_{n,m}^{\mathrm{OS}})^2
\leq
C\kappa_m^2J^{-2s/d}.
\label{eq:con-os-direct-approx-bound}
\end{equation}
For the stated choice of $J_n$, put
\[
t=
\left\{
\frac{\kappa_m^2 n}
{\bar a_m
\log\!\left(eBnm\sqrt{\kappa_m a_{\max,m}}\right)}
\right\}^{\frac{d}{2s+d}}.
\]
Since $t\leq J_n\leq t+1$,
\[
\begin{aligned}
&\kappa_m^2J_n^{-2s/d}
+
\frac{\bar a_mJ_n}{n}
\log\!\left(eBnm\sqrt{\kappa_m a_{\max,m}}\right)
\\
&\qquad\lesssim
\kappa_m^{\frac{2d}{2s+d}}
\left[
\frac{\bar a_m}{n}
\log\!\left(eBnm\sqrt{\kappa_m a_{\max,m}}\right)
\right]^{\frac{2s}{2s+d}}+
\frac{\bar a_m}{n}
\log\!\left(eBnm\sqrt{\kappa_m a_{\max,m}}\right).
\end{aligned}
\]
This proves \eqref{eq:con-os-balanced-rate}. If
$\kappa_m,\bar a_m,a_{\max,m}=O(1)$ and $m$ grows at most
polynomially in $n$, ignoring logarithmic factors gives
$O_{\mathbb P}(n^{-2s/(2s+d)})$.
\end{proof}
}}

\newpage

\section{Representative DAG structures}
\label{app:dag-examples}

This section illustrates how the topology-dependent quantities
$\kappa_m$ and $\bar a_m$ vary across several representative DAG
structures. Recall that $\kappa_m$ is the maximum number of vertices
on a root-to-node path, and $\bar a_m = \frac{1}{m}\sum_{j=1}^m a_j,$ where $a_j$ denotes the number of ancestors of node $j$, including
the node itself. Ignoring logarithmic factors, \eqref{eq:growing-m-main} yields an estimation term of order
$\bar a_m v_n/n$.

We consider the following DAGs.

\begin{itemize}
    \item \textbf{Isolated roots.}
    The graph contains $m$ nodes and no edges. Hence every node is a
    root and has only itself as an ancestor.

    \item \textbf{Chain.}
    The nodes are ordered as
    $1 \to 2 \to \cdots \to m.$ 
    Thus node $j$ has ancestor set $\{1,\ldots,j\}$.

    \item \textbf{Rooted tree of depth $D$.}
    The DAG is a rooted directed tree with all edges oriented away
    from the root, and the longest root-to-node path contains $D+1$
    vertices. A node at depth $r$ has $r+1$ ancestors.

    \item \textbf{Star with $m-1$ roots and one sink.}
    There are $m-1$ root nodes, each with an edge to a single sink
    node. Thus the sink has all $m$ nodes as ancestors, whereas every
    root has only itself as an ancestor.

    \item \textbf{Complete two-layer DAG.}
    The nodes are partitioned into two layers, with every node in the
    first layer being a root and every possible edge directed from
    the first layer to the second layer. We take both layers to have
    order $m$, up to constant factors. Hence each node in the second
    layer has order $m$ ancestors, while root nodes have only
    themselves as ancestors.
\end{itemize}

The resulting quantities and rates are summarized below.

\begin{center}
\small
\begin{tabular}{lccc}
\toprule
DAG structure
& $\kappa_m$
& $\bar a_m$
& Average MSE bound \\
\midrule

$m$ isolated roots
& $1$
& $1$
& $\Delta_{n,m}^2 + v_n/n$ \\

Chain of length $m$
& $m$
& $(m+1)/2$
& $\Delta_{n,m}^2 + mv_n/n$ \\

Rooted tree of depth $D$
& $D+1$
& $\leq D+1$
& $\Delta_{n,m}^2 + Dv_n/n$ \\

Star with $m-1$ roots and one sink
& $2$
& $2-1/m$
& $\Delta_{n,m}^2 + v_n/n$ \\

Complete two-layer DAG
& $2$
& $\asymp m$
& $\Delta_{n,m}^2 + mv_n/n$ \\

\bottomrule
\end{tabular}
\end{center}

These examples show that the number of nodes $m$ alone does not
determine the estimation rate. For isolated roots and the star,
$\bar a_m$ remains bounded, so the estimation term is of order
$v_n/n$ up to logarithmic factors. For a rooted tree, the dependence
is controlled by its depth; in particular, a balanced tree has depth
of order $\log m$. In contrast, both the chain and the complete
two-layer DAG have $\bar a_m$ of order $m$, leading to an estimation
term of order $mv_n/n$.

\section{Comparison with prior theory}
\label{app:compare-theory}
Table~\ref{tab:comparison} summarizes recent theoretical results most closely
related to our work. Existing results primarily study distributional regression
for Euclidean responses, and only \citet{padilla2025risk} provides an explicit
convergence rate for DNN-based distributional regression. Moreover,
\citet{padilla2025risk} does not provide monotonicity estimation with respect to set
inclusion. Methods such as \citet{cevid2022distributional} and \citet{shen2025engression} may preserve
monotonicity from an algorithmic perspective; however, rather than addressing
general set-indexed conditional event probabilities, their theoretical
convergence guarantees apply to particular distributional functionals
(e.g., CDFs and quantiles), and these guarantees do not provide explicit
convergence rates in terms of $n$. In contrast, our framework estimates set-indexed conditional event probabilities on a general response space $\mathcal{Y}$ that allows non-Euclidean responses, and enforces monotonicity with respect to set inclusion. Specifically, we provide explicit convergence rates for both the uniform MSE of the preliminary estimator and the aggregated MSE of the constrained estimator.

\begin{table}[H]
\centering
\caption{Comparison of recent theoretical results in distributional regression. Here, $o_p(1)$ denotes convergence to zero in probability.}
\label{tab:comparison}
\scriptsize
\renewcommand{\arraystretch}{1.15}
\setlength{\tabcolsep}{4pt}
\begin{tabular}{c >{\centering\arraybackslash}m{2.4cm} >{\raggedright\arraybackslash}m{4.2cm} c >{\centering\arraybackslash}m{4.4cm}}
\toprule
& \textsc{Method} 
& \multicolumn{1}{c}{\textsc{Assumptions / Target}} 
& \textsc{Monotonicity} 
& \textsc{Convergence Rate} \\
\midrule
\multirow{4}{*}{\raisebox{-17ex}{\textsc{\shortstack{Prior \\ Work}}}}
& ~\cite{kohler2019rate} DNN non-parametric regression (not distributional)
& $Y \in \mathbb{R}$ \newline
  $f^*(x) = \mathbb{E}(Y|X=x)$ \newline
  $f^*(\cdot) \in \mathcal{H}(l, \mathcal{P})$
& No 
& $\phi_n$  (\textsc{MSE}) \\
\cmidrule(lr){2-5}
& ~\cite{cevid2022distributional} Forest-weighted distributional regression
& $Y \in \mathbb{R}^u$ \newline
  $F_A^*(x) = \mathbb{P}(Y \in A \mid X=x)$ \newline
  $F_A^*(\cdot) \in \mathcal{H}(1,\{1,d_0\})$
& Yes 
& $o_p(1)$ \\
\cmidrule(lr){2-5}
& ~\cite{shen2025engression} Generative DNN-based sampling
& $Y\in\mathbb{R}$ for theory; extends to\newline
   $Y\in\mathbb{R}^u$ methodologically \newline
  $F_A^*(x) = \mathbb{P}(Y \in A \mid X=x)$ \newline
  $F_A^*(\cdot) \in \mathcal{H}(1,\{1,d_0\})$
& Yes 
& $o_p(1)$ \\
\cmidrule(lr){2-5}
& \cite{padilla2025risk} DNN-based distributional regression for scalar CDF
& $Y \in \mathbb{R}$ \newline
  $F_t^*(x) = \mathbb{P}(Y \leq t \mid X=x)$ \newline
  $F_t^*(\cdot) \in \mathcal{H}(l,\mathcal{P})$
& No 
& $\dfrac{\log n}{n} + \phi_n \log^4 n$ \newline
  (\textsc{CRPS}~\citep{gneiting2007strictly}, 
  {Worst-case MSE}) \\
\midrule
\multirow{2}{*}{\raisebox{-8ex}{\textsc{\shortstack{This \\ Work}}}}
& Preliminary estimator
& $Y\in\mathcal{Y}$: general, non-Euclidean \newline
  $\Lambda$: totally ordered, potentially uncountable \newline
  $F_A^*(x)=\mathbb{P}(Y\in A\mid X=x)$, $A\in\Lambda$ \newline
  $F_A^*(\cdot)\in\mathcal{H}(l_A,\mathcal{P}_A)$
& No
& $\phi_n \log^4 n$ \newline
  (Uniform MSE over $\Lambda$) \\
\cmidrule(lr){2-5}
& Constrained estimator
& $Y\in\mathcal{Y}$: general, non-Euclidean \newline
  {$\Lambda=\{A_j\}_{j=1}^{m}\subset\mathcal{Y}$ with inclusion DAG} \newline
  $F_{A_j}^*(x)=\mathbb{P}(Y\in A_j\mid X=x)$ \newline
  {$F_A^*(\cdot)\in\mathcal{H}(l_A,\mathcal{P}_A)$}
& Yes
& {$\phi_n \log^4 n$} \newline
  {(Average MSE over $\Lambda$ with fixed sets)} \\
\bottomrule
\end{tabular}
\end{table}

\section{Compute resources}

\label{app:compute}

All experiments are conducted on a single laptop equipped with an Intel Core Ultra 9 275HX CPU, $32$GB of RAM, and an NVIDIA RTX 5070 Ti laptop GPU with $12$GB of VRAM. 

\section{Detailed configuration and results for numerical experiments}
\label{app:additional_scenarios}
In this section, we provide the detailed settings for all eight scenarios, including the data-generating mechanisms and the design of the sets (Appendix~\ref{app:num-setting}), full experimental results (Appendix~\ref{app:num-results}), and the configurations of the estimators used (Appendix~\ref{app:num-config-estim}).

\subsection{Detailed design for the eight scenarios}
\label{app:num-setting}
Here we present detailed design for the three scenarios (S1, S2, S3) mentioned in the main text and the additional five scenarios (S4, S5, S6, S7, S8).

\subsubsection{Scenario 1 (Swiss roll manifold).} 
We consider $m=20$ sets for a response $Y\in\mathbb{R}^3$ lying on a Swiss roll manifold. Let $X\sim\mathrm{Unif}([0,2]^{50})$ and define $X_{\mathrm{noisy}}=X+\epsilon$, where $\epsilon\sim\mathcal{N}(0,\sigma^2(X)I_{50})$ with heteroscedastic variance $\sigma^2(X)=1+0.25\|X\|^2/50$. We generate two latent coordinates from $X_{\mathrm{noisy}}$ as
\[
t=1.5\pi+3\pi\tau(w_t^\top X_{\mathrm{noisy}}+b_t),
\qquad
h=\tau(w_h^\top X_{\mathrm{noisy}}+b_h),
\]
where $\tau(u)=(1+e^{-u})^{-1}$, the entries of $w_t,w_h\in\mathbb{R}^{50}$ are drawn i.i.d.\ from $\mathcal{N}(0,0.25)$, and $b_t,b_h\sim\mathcal{N}(0,0.01)$. Thus, $t\in[1.5\pi,4.5\pi]$ and $h\in[0,1]$. The response is then obtained through the Swiss roll parametrization
\[
Y=(t\cos t,\,h,\,t\sin t).
\]

The $m = 20$ sets are defined in the $(t, h)$ parameter space using two families of quantile thresholds:
\[
q_t^{(l_t)} = Q_{\alpha_t^{(l_t)}}^{(t)}, \quad \alpha_t \in \{15\%, 30\%, 50\%, 70\%, 85\%\}, \quad l_t = 1,\ldots,5,
\]
\[
q_h^{(l_h)} = Q_{\alpha_h^{(l_h)}}^{(h)}, \quad \alpha_h \in \{20\%, 40\%, 60\%, 80\%\}, \quad l_h = 1,\ldots,4,
\]
where $Q_\alpha^{(\cdot)}$ denotes the $\alpha$-th percentile of the respective marginal distribution, estimated from an independent sample of size $10{,}000$. The 20 sets form a full $5 \times 4$ grid:
\begin{itemize}
    \item $A_{(l_t, l_h)} = \{(t,h) : t \leq q_t^{(l_t)},\, h \leq q_h^{(l_h)}\}$, \quad $l_t \in \{1,\ldots,5\}$,  \quad $l_h \in \{1,\ldots,4\}$.
\end{itemize}
The DAG structure is induced by componentwise ordering: $A_{(l_t, l_h)} \subset A_{(l_t', l_h')}$ if and only if $l_t \leq l_t'$ and $l_h \leq l_h'$.

\subsubsection{S2 (Random dot product graph).} 
Let $Y \in \{0,1\}^{N \times N}$ be the adjacency matrix with $N = 50$, where each node $i$ carries a unit-length embedding vector $z_i \in \mathbb{S}^{d-1}$ with $d = 5$, and the probability of an edge between nodes $i$ and $j$ is $\tfrac{1}{2}(\langle z_i, z_j \rangle + 1)$. The RDPG is endowed with a community structure of $K = 3$ communities, where both the community assignment probabilities $\pi_k(X)$ and the unit-length community embedding centers $a_k(X) \in \mathbb{S}^{d-1}$ are non-linearly parametrized by the covariate $X \sim \mathrm{Unif}([0,2]^{50})$ for $k\in\{1,2,3\}$. 

Specifically, we use a two-layer MLP to map $X$ to a vector $Q(X) \in \mathbb{R}^{18}$ with $Q(X) := W^{(2)} \tanh\bigl(W^{(1)} X\bigr)$, where $W^{(1)} \in \mathbb{R}^{16 \times 50}$ and $W^{(2)} \in \mathbb{R}^{18 \times 16}$ have entries independently drawn from $\mathcal{N}(0, 2/50)$ and $\mathcal{N}(0, 2/16)$, respectively. The output $Q(X)$ is split into two parts indexed by community $k \in \{1, 2, 3\}$. The first $15$ components are reshaped into $3$ vectors $\{\tilde{a}_k(X)\}_{k=1}^{3} \in \mathbb{R}^5$, each normalized to obtain the community embedding center $a_k(X) = \tilde{a}_k(X) / \|\tilde{a}_k(X)\| \in \mathbb{S}^{d-1}$. The remaining $3$ components, denoted $\{\ell_k(X)\}_{k=1}^{3}$, are treated as logits and passed through a softmax to obtain the community assignment probability $\pi_k(X) = \exp(\ell_k(X)) / \sum_{k'=1}^{3} \exp(\ell_{k'}(X))$.

Each node $i$ is then independently assigned $c_i \sim \mathrm{Cat}(\pi(X))$, with embedding $z_i = (a_{c_i}(X) + \varepsilon_i) / \|a_{c_i}(X) + \varepsilon_i\|$ where $\varepsilon_i \sim \mathcal{N}(0, 0.3^2 I_d)$, yielding the random graph $Y$. 

We estimate the distribution of the network statistic transitivity. For a graph $Y$ with adjacency matrix $A \in \{0,1\}^{N \times N}$, the transitivity is defined as $T(Y) = \frac{3 \cdot \mathrm{tri}(Y)}{\mathrm{tp}(Y)},$ where $\mathrm{tri}(Y) = \sum_{i < j < k} A_{ij} A_{jk} A_{ik}$ counts the number of triangles in $Y$, and $\mathrm{tp}(Y) = \sum_{i} \binom{d_i}{2}$ with $d_i = \sum_{j \neq i} A_{ij}$ counts the number of length-two paths of node $i$. 

The $m = 10$ sets form a chain $A_1 \subset A_2 \subset \cdots \subset A_{10}$, defined as sublevel sets of the transitivity:
\[
A_j = \bigl\{Y : T(Y) \leq q_{\alpha_j}\bigr\}, \qquad q_{\alpha_j} = Q_{\alpha_j}\bigl(T(Y)\bigr), \qquad j = 1, \ldots, 10,
\]
where $Q_\alpha$ denotes the $\alpha$-th percentile of the marginal distribution of $T(Y)$, and $\alpha_j = 5 + 10(j-1)$ for $j = 1, \ldots, 10$ (i.e.\ evenly spaced from $5\%$ to $95\%$). 

\subsubsection{S3 (Pixel density estimation on single noisy natural image).} We consider $X \in [0,1]^2$ a pixel coordinate and $Y = (Y_R, Y_G, Y_B) \in [0,1]^3$ a noisy observation of the clean pixel value $J(X) = (J_R, J_G, J_B)$ at $X$, with the goal of estimating $Y \mid X$ for a single given noisy image. Since we use $96 \times 96$ images from STL-10 (see Figure~\ref{fig:stl10} for illustration), $X$ takes values on the discrete grid $\{k/96 : k = 1, \ldots, 96\}^2$. The noise follows a Poisson-Gaussian model: $Y_c = \operatorname{Poisson}(50 \cdot J_c(X))/50 + \varepsilon_c$ with $\varepsilon_c \sim \mathcal{N}(0, 0.05^2)$ i.i.d. across $c \in \{R,G,B\}$.

The $m = 13$ sets are upper-orthants in $\mathbb{R}^3$, where each channel admits three nested intervals:
\[
L_1^{(c)} = (-\infty, Q_{30}^{(c)}], \quad L_2^{(c)} = (-\infty, Q_{60}^{(c)}], \quad L_3^{(c)} = (-\infty, Q_{90}^{(c)}],
\]
with $Q_\alpha^{(c)}$ denoting the $\alpha$-th percentile of the marginal distribution of $Y_c$, estimated from $50000$ noisy samples. The 13 sets are organized as follows:
\begin{itemize}
    \item $A_1 = L_1^{(R)} \times L_1^{(G)} \times L_1^{(B)}$,
    \item $A_2 = L_2^{(R)} \times L_1^{(G)} \times L_1^{(B)}$, \quad $A_3 = L_1^{(R)} \times L_2^{(G)} \times L_1^{(B)}$, \quad $A_4 = L_1^{(R)} \times L_1^{(G)} \times L_2^{(B)}$,
    \item $A_5 = L_2^{(R)} \times L_2^{(G)} \times L_1^{(B)}$, \quad $A_6 = L_2^{(R)} \times L_1^{(G)} \times L_2^{(B)}$, \quad $A_7 = L_1^{(R)} \times L_2^{(G)} \times L_2^{(B)}$,
    \item $A_8 = L_2^{(R)} \times L_2^{(G)} \times L_2^{(B)}$,
    \item $A_9 = L_3^{(R)} \times L_2^{(G)} \times L_2^{(B)}$, \quad $A_{10} = L_2^{(R)} \times L_3^{(G)} \times L_2^{(B)}$, \quad $A_{11} = L_2^{(R)} \times L_2^{(G)} \times L_3^{(B)}$,
    \item $A_{12} = L_3^{(R)} \times L_3^{(G)} \times L_2^{(B)}$,
    \item $A_{13} = L_3^{(R)} \times L_3^{(G)} \times L_3^{(B)}$.
\end{itemize}
The DAG structure is induced by componentwise inclusion: $A_i \subset A_j$ if and only if $L_{\ell_c^{(i)}}^{(c)} \subseteq L_{\ell_c^{(j)}}^{(c)}$ for all $c \in \{R,G,B\}$.
Also, since the per-pixel data are abundant, we follow~\citet{foi2008practical} and use the heteroskedastic Gaussian approximation to obtain the true conditional probability in closed form:
\[
\mathbb{P}(Y \in A_j \mid X) \approx \prod_{c \in \{R,G,B\}} \Phi\left( \frac{Q_{\alpha_c^{(j)}}^{(c)} - J_c(X)}{\sqrt{J_c(X)/\lambda + \sigma^2}} \right),
\]
where $\alpha_c^{(j)} \in \{30, 60, 90\}$ is the percentile level associated with channel $c$ in $A_j$.

\begin{figure}[H]
\label{fig:stl10}
    \centering
    \includegraphics[width=0.75\linewidth]{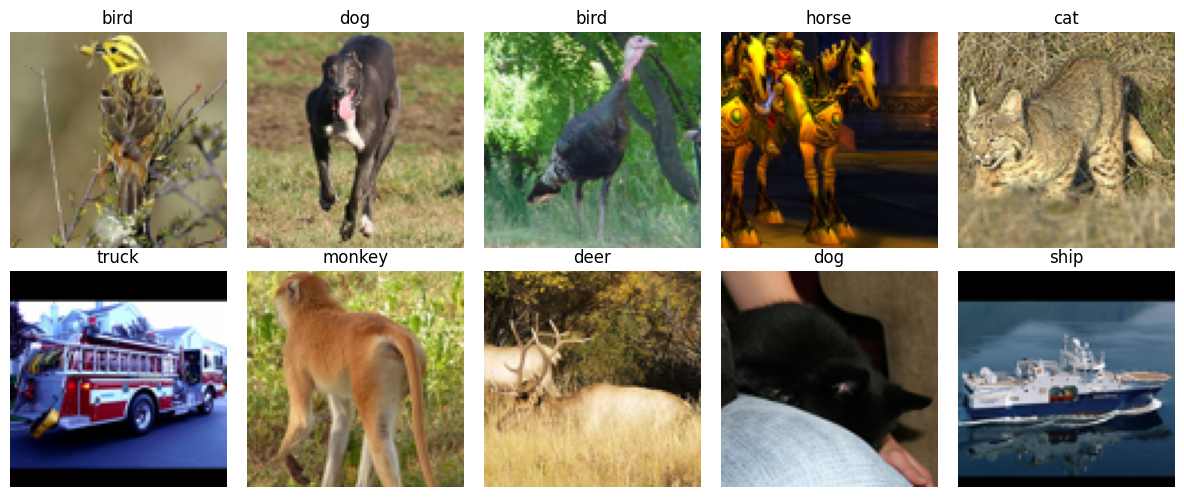}
    \caption{Example noiseless images from the STL-10 dataset across 10 categories.}
    \label{fig:placeholder}
\end{figure}

\subsubsection{S4 (Chain DAG).}
We consider $m = 10$ nested sets in a chain structure
$A_1 \subset A_2 \subset \cdots \subset A_{10}$, representing
the simplest DAG topology. Covariates are $X \sim \mathrm{Unif}([0,2]^{10})$. The response $Y \in \mathbb{R}^2$ follows a heteroscedastic Gaussian model $Y \mid X \sim \mathcal{N}(\mu(X), \sigma^2(X) I_2)$, where $\mu(X) = \Theta_1 X + \Theta_2 X^{\odot 2}$ with $\Theta_1$ and $\Theta_2$ having entries drawn i.i.d.\ from $\mathrm{Unif}([-1,1])$ and $\mathrm{Unif}([-0.5, 0.5])$, respectively, and $\sigma^2(X) = 1 + 0.25\|X\|^2$.

The $m = 10$ sets are Euclidean balls $A_j = \{y :
\|y\|^2 \leq r_j^2\}$ with radii given by the 10th through
95th percentiles of $\|Y\|^2$. Specifically, the sets are defined as
\[
A_j = \bigl\{ y \in \mathbb{R}^2 : \|y\|^2 \leq r_j^2 \bigr\}, \quad j = 1, \ldots, 10,
\]
with radii $r_1 < r_2 < \cdots < r_{10}$ defined by
\[
r_j = \sqrt{q_{\alpha_j}}, \quad q_{\alpha_j} = Q_{\alpha_j}(\|Y\|^2),
\]
where $Q_\alpha$ denotes the $\alpha$-th percentile of the marginal distribution of $\|Y\|^2$, and $\alpha_j = 10 + \tfrac{85}{9}(j-1)$ for $j = 1, \ldots, 10$ (i.e.\ evenly spaced from $10\%$ to $95\%$). 

\subsubsection{S5 (General DAG).}
We use the same covariate and response distributions as
Scenario~4, but consider $m = 7$ sets with a richer DAG
structure involving multiple roots and shared descendants.
Sets are axis-aligned rectangles $A_j = \{y : y_k \in
[q_{k,l_j}, q_{k,u_j}],\, k=1,2\}$, where $q_{k,\alpha}$
are the $\alpha$-quantile of $Y_k$.

The $m = 7$ sets are axis-aligned rectangles $A_j = I_1^{(j)} \times I_2^{(j)}$, where $Q_\alpha^{(k)}$ denotes the $\alpha$-th percentile of the marginal distribution of $Y_k$. The 7 sets take:
\begin{itemize}
    \item $A_1 = [Q_{30}^{(1)}, Q_{70}^{(1)}] \times [Q_{30}^{(2)}, Q_{70}^{(2)}]$,
    \item $A_2 = [Q_{10}^{(1)}, Q_{30}^{(1)}] \times [Q_{25}^{(2)}, Q_{75}^{(2)}]$,
    \item $A_3 = [Q_{25}^{(1)}, Q_{75}^{(1)}] \times [Q_{25}^{(2)}, Q_{75}^{(2)}]$,
    \item $A_4 = [Q_{70}^{(1)}, Q_{90}^{(1)}] \times [Q_{25}^{(2)}, Q_{75}^{(2)}]$,
    \item $A_5 = [Q_{10}^{(1)}, Q_{75}^{(1)}] \times [Q_{10}^{(2)}, Q_{90}^{(2)}]$,
    \item $A_6 = [Q_{25}^{(1)}, Q_{90}^{(1)}] \times [Q_{10}^{(2)}, Q_{90}^{(2)}]$,
    \item $A_7 = [Q_{5}^{(1)},  Q_{95}^{(1)}]  \times [Q_{5}^{(2)},  Q_{95}^{(2)}]$.
\end{itemize}
The DAG structure follows from set inclusion and is illustrated in Figure~\ref{fig:dag_structure_s2}.
\begin{figure}[h]
\label{fig:dag_structure_s2}
\centering
\begin{tikzpicture}[
    node distance=1.2cm and 0.8cm,
    every node/.style={circle, draw, minimum size=0.8cm, font=\small},
    arrow/.style={->, >=stealth, thick}
]
% Bottom level - roots
\node (A1) {$A_1$};
\node (A2) [left=of A1, xshift=-0.5cm] {$A_2$};
\node (A4) [right=of A1, xshift=0.5cm] {$A_4$};
% Middle level
\node (A3) [above=of A1] {$A_3$};
\node (A5) [above=of A2, xshift=0.8cm] {$A_5$};
\node (A6) [above=of A4, xshift=-0.8cm] {$A_6$};
% Top level
\node (A7) [above=of A3, yshift=0.3cm] {$A_7$};
% Edges (from subset to superset)
\draw[arrow] (A1) -- (A3);
\draw[arrow] (A2) -- (A5);
\draw[arrow] (A3) -- (A5);
\draw[arrow] (A3) -- (A6);
\draw[arrow] (A4) -- (A6);
\draw[arrow] (A5) -- (A7);
\draw[arrow] (A6) -- (A7);
\end{tikzpicture}
\caption{DAG structure for S5. An arrow from $A_i$ to $A_j$ indicates $A_i \subset A_j$.}
\end{figure}
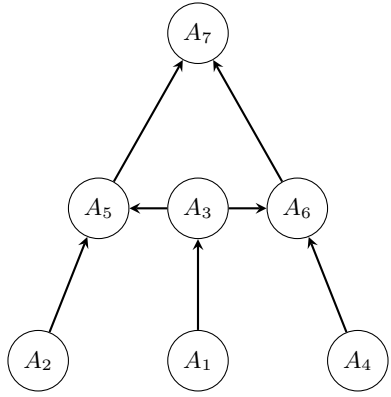

\subsubsection{S6 (gaussian mixture).}
Consider the covariates follow $X \sim \mathrm{Unif}([0,2]^{30})$. We generate $Y \in \mathbb{R}^5$ from a mixture
\[
Y \mid X \sim \sum_{k=1}^{4} \pi_k(X)\, \mathcal{N}(\mu_k(X),\, \sigma_k^2(X)\, I_5),
\]
where component means follow $\mu_k(X) =\Theta_{k,1} X + \Theta_{k,2} X^{\odot 2}$ with $\Theta_{k,1} \in \mathbb{R}^{5 \times 30}$ having entries from $\mathrm{Unif}([-1,1])$ and $\Theta_{k,2} \in \mathbb{R}^{5 \times 30}$ from $\mathrm{Unif}([-0.5, 0.5])$, and $X^{\odot 2}$ denotes element-wise square of $X$. The mixing weights are $\pi(X) = \mathrm{softmax}(W_\pi X + b_\pi)$ with $W_\pi \in \mathbb{R}^{4 \times 30}$. The variance is heteroscedastic: $\sigma_k^2(X) = \sigma_{k,0}^2 + \sigma_{k,1}^2 \|\mu_k(X)\|^2$, where $\sigma_{k,0} \sim \mathrm{Unif}([1, 2])$ and $\sigma_{k,1} \sim \mathrm{Unif}([0.3, 0.5])$.

The $m = 19$ sets are axis-aligned hyperrectangles $A_j = \prod_{k=1}^{5} I_k^{(j)}$, where each interval $I_k^{(j)}$ is drawn from four nested levels:
\[
L_1^{(k)} = [Q_{20}^{(k)}, Q_{80}^{(k)}], \quad
L_2^{(k)} = [Q_{15}^{(k)}, Q_{85}^{(k)}], \quad
L_3^{(k)} = [Q_{10}^{(k)}, Q_{90}^{(k)}], \quad
L_4^{(k)} = [Q_{5}^{(k)},  Q_{95}^{(k)}],
\]
where $Q_\alpha^{(k)}$ denotes the $\alpha$-th percentile of the marginal distribution of $Y_k$, estimated from an independent sample of size $10{,}000$. The 19 sets are organized as follows:
\begin{itemize}
    \item $A_1 = \prod_{k=1}^{5} L_1^{(k)}$,
    \item $A_{i+1} = L_2^{(i)} \times \prod_{k \neq i} L_1^{(k)}$, \quad $i = 1, \ldots, 5$,
    \item $A_{i+6} = L_3^{(i)} \times \prod_{k \neq i} L_1^{(k)}$, \quad $i = 1, \ldots, 5$,
    \item $A_{i+11} = L_3^{(i)} \times L_3^{(i \bmod 5 + 1)} \times \prod_{k \neq i, i\bmod 5+1} L_1^{(k)}$, \quad $i = 1, \ldots, 5$,
    \item $A_{17} = \prod_{k=1}^{5} L_2^{(k)}$,
    \item $A_{18} = \prod_{k=1}^{5} L_3^{(k)}$,
    \item $A_{19} = \prod_{k=1}^{5} L_4^{(k)}$.
\end{itemize}
The DAG structure follows from set inclusion and is illustrated in Figure~\ref{fig:dag_structure_s3}.

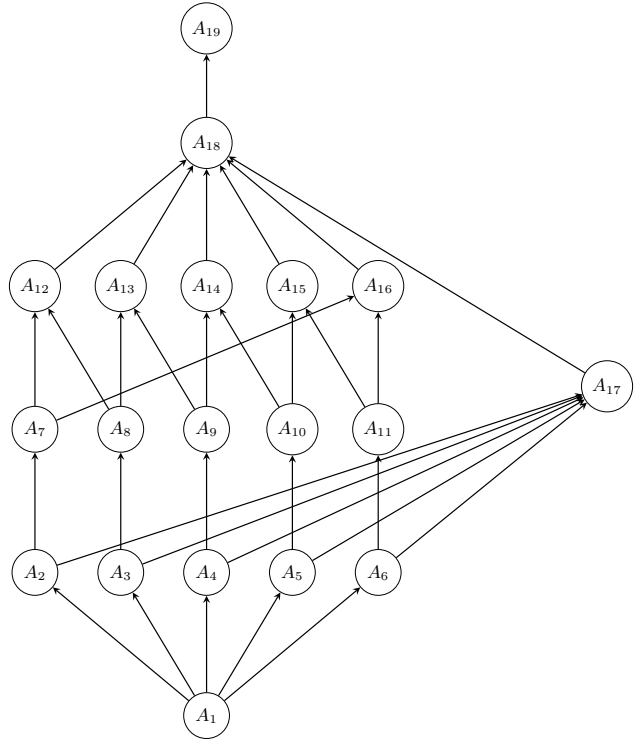
\begin{figure}[H]
\centering
\resizebox{0.5\textwidth}{!}{%
\begin{tikzpicture}[
    every node/.style={circle, draw, minimum size=0.8cm, font=\small},
    arrow/.style={->, >=stealth, semithick}
]

% Level 0
\node (A1)  at ( 0,   0) {$A_1$};
% Level 1: A2--A6
\node (A2)  at (-3,  2.5) {$A_2$};
\node (A3)  at (-1.5,  2.5) {$A_3$};
\node (A4)  at ( 0,  2.5) {$A_4$};
\node (A5)  at ( 1.5,  2.5) {$A_5$};
\node (A6)  at ( 3,  2.5) {$A_6$};
% Level 2: A7--A11
\node (A7)  at (-3,  5) {$A_7$};
\node (A8)  at (-1.5,  5) {$A_8$};
\node (A9)  at ( 0,  5) {$A_9$};
\node (A10) at ( 1.5,  5) {$A_{10}$};
\node (A11) at ( 3,  5) {$A_{11}$};
% Level 3: A12--A16
\node (A12) at (-3,  7.5) {$A_{12}$};
\node (A13) at (-1.5,  7.5) {$A_{13}$};
\node (A14) at ( 0,  7.5) {$A_{14}$};
\node (A15) at ( 1.5,  7.5) {$A_{15}$};
\node (A16) at ( 3,  7.5) {$A_{16}$};
% A17: all dims L1 -- placed right
\node (A17) at ( 7,  5.75) {$A_{17}$};
% A18: all dims L2
\node (A18) at ( 0,  10) {$A_{18}$};
% A19: all dims L3
\node (A19) at ( 0,  12) {$A_{19}$};
% A1 -> A2..A6
\foreach \i in {2,3,4,5,6}
    \draw[arrow] (A1) -- (A\i);
% A2..A6 -> A7..A11
\draw[arrow] (A2) -- (A7);
\draw[arrow] (A3) -- (A8);
\draw[arrow] (A4) -- (A9);
\draw[arrow] (A5) -- (A10);
\draw[arrow] (A6) -- (A11);
% A2..A6 -> A17
\foreach \i in {2,3,4,5,6}
    \draw[arrow] (A\i) -- (A17);
% A7..A11 -> A12..A16 (adjacent pairs, cyclic)
\draw[arrow] (A7)  -- (A12);
\draw[arrow] (A8)  -- (A12);
\draw[arrow] (A8)  -- (A13);
\draw[arrow] (A9)  -- (A13);
\draw[arrow] (A9)  -- (A14);
\draw[arrow] (A10) -- (A14);
\draw[arrow] (A10) -- (A15);
\draw[arrow] (A11) -- (A15);
\draw[arrow] (A11) -- (A16);
\draw[arrow] (A7)  -- (A16);
% A12..A16 -> A18
\foreach \i in {12,13,14,15,16}
    \draw[arrow] (A\i) -- (A18);
% A17 -> A18
\draw[arrow] (A17) -- (A18);
% A18 -> A19
\draw[arrow] (A18) -- (A19);

\end{tikzpicture}%
}
\caption{DAG structure for S1. An arrow from $A_i$ to $A_j$ indicates $A_i \subset A_j$.
}
\label{fig:dag_structure_s3}
\end{figure}

% REVISION START: functional- and network-valued response experiments
{
\subsubsection{Scenario 7 (Functional DAG).}
We consider a functional-valued response $Y:[0,1]\to\mathbb{R}$ with
covariates $X\sim\operatorname{Unif}([-1,1]^{10})$. The response is
generated through the finite Fourier expansion
\[
Y(t)=\sum_{k=1}^6 \xi_k\varphi_k(t), \qquad t\in[0,1],
\]
where $(\varphi_1,\ldots,\varphi_6)
=
\bigl(
1,\sqrt{2}\sin(2\pi t),\sqrt{2}\cos(2\pi t),
\sqrt{2}\sin(4\pi t),\sqrt{2}\cos(4\pi t),
\sqrt{2}\sin(6\pi t)
\bigr).$ Conditionally on $X=x$, the coefficient vector
$\xi=(\xi_1,\ldots,\xi_6)^\top$ follows
\[
\xi\mid X=x\sim\mathcal{N}_6\bigl(0,\Sigma(x)\bigr),
\qquad
\Sigma(x)=D(x)R D(x),
\]
where $R_{k\ell}=0.4^{|k-\ell|}$ and $D(x)=\operatorname{diag}\{s_1(x),\ldots,s_6(x)\}.$ Let $b=(1,0.9,0.7,0.6,0.5,0.4)$. The scale functions are defined by
\[
u(x)
=
1.2\sin(\pi x_1x_2)
+0.8(x_3^2-1/3)
+0.6x_4,
\]
and
\[
s_k(x)
=
b_k\{0.4+0.8\tau(u(x))\}
\{0.8+0.4\tau(1.5x_{k+4})\},
\qquad k=1,\ldots,6,
\]
where $\tau(x)=1/(1+e^{-x})$. Thus, although the conditional mean of $Y$ is zero, its conditional
distribution varies nonlinearly with $X$ through the covariance
structure.
To construct the target sets, we define two functionals
\[
D_1(Y)=\int_0^1 |Y(t)|\,dt,
\qquad
D_\infty(Y)=\sup_{t\in[0,1]}|Y(t)|.
\]
Let $Q_\alpha^{(D_1)}$ and $Q_\alpha^{(D_\infty)}$ denote the
$\alpha$-th percentiles of the marginal distributions of $D_1(Y)$ and
$D_\infty(Y)$, respectively. We use the quantile levels
\[
(\alpha_1^{(1)},\alpha_1^{(2)},\alpha_1^{(3)})
=
(\alpha_\infty^{(1)},\alpha_\infty^{(2)},\alpha_\infty^{(3)})
=
(40\%,60\%,80\%),
\]
and define
\[
q_1^{(l_1)}=Q_{\alpha_1^{(l_1)}}^{(D_1)},
\qquad
q_\infty^{(l_\infty)}
=Q_{\alpha_\infty^{(l_\infty)}}^{(D_\infty)},
\qquad l_1,l_\infty\in\{1,2,3\}.
\]
We consider $m=9$ sets,
\[
A_{(l_1,l_\infty)}
=
\left\{
Y:
D_1(Y)\le q_1^{(l_1)},
\quad
D_\infty(Y)\le q_\infty^{(l_\infty)}
\right\},
\qquad l_1,l_\infty\in\{1,2,3\}.
\]
The DAG structure is then induced by componentwise ordering: 
$A_{(l_1,l_\infty)}\subseteq A_{(l_1',l_\infty')}$ if 
$l_1\le l_1'$ and $l_\infty\le l_\infty'$. Thus, the DAG is a full $3\times3$ grid, induced by the componentwise
ordering of the two quantile indices, with edges directed from each set
to its immediate supersets.

\subsubsection{Scenario 8 (Network balls).}
We consider an additional experiment for network responses. Each
response is the adjacency matrix $Y\in\{0,1\}^{50\times50}$ of a
50-node RDPG generated using the same mechanism as in S2. For each replication, we generate a reference network
$Y_0\in\{0,1\}^{50\times50}$ once from the same model with fixed input covariates $x_0=(1,\ldots,1)\in[0,2]^{50}$. Define the Frobenius distance
\[
D(Y,Y_0)=\|Y-Y_0\|_F.
\]
For aligned undirected binary networks, $D(Y,Y_0)^2$ is proportional to
the fraction of edges on which $Y$ and $Y_0$ disagree. We define ten
nested network balls centered at $Y_0$:
\[
A_j=\{Y:D(Y,Y_0)\le q_{\alpha_j}\},
\qquad j=1,\ldots,10,
\]
where $q_{\alpha_j}=Q_{\alpha_j}(D(Y,Y_0))$ is the $\alpha_j$-th
percentile of the marginal distribution of $D(Y,Y_0)$, with
$\alpha_j=5\%+10\%(j-1)$. We refer to this setting as network balls.

}
\subsection{Full experiment results}
\label{app:num-results}
We present the full experimental results for all scenarios in this section. Consistent with the conclusions in the main text, \textbf{Constrain-NN} achieves the best performance in most cases, while \textbf{Pre-NN} also converges and surpasses \textbf{DRF} and \textbf{Engression} as the sample $n$ grows.

\begin{table}[H]
\scriptsize
\centering
\caption{Mean Squared Error across all simulation scenarios. Best result in each row is bolded.}
\label{apptab:all_scenarios}
\begin{tabular}{clccccc}
\toprule
Scenario & $n_{\text{train}}$ & Constrain-NN & Pre-NN & Pre-NN-Clip & DRF & Engression \\
\midrule
\multirow{3}{*}{S1 (Swiss roll)}
& 1000 & \textbf{0.0096 (0.0007)} & 0.0307 (0.0015) & 0.0249 (0.0016) & 0.0128 (0.0013) & 0.0493 (0.0255) \\
& 2000 & \textbf{0.0072 (0.0004)} & 0.0191 (0.0010) & 0.0156 (0.0011) & 0.0122 (0.0012) & 0.0466 (0.0244) \\
& 4000 & \textbf{0.0057 (0.0004)} & 0.0104 (0.0005) & 0.0088 (0.0005) & 0.0116 (0.0011) & 0.0436 (0.0239) \\
\midrule
\multirow{3}{*}{S2 (RDPG)}
& 1000 & \textbf{0.0348 (0.0011)} & 0.0464 (0.0010) & 0.0439 (0.0010) & 0.0713 (0.0146) & 0.0420 (0.0135)\\
& 2000 & \textbf{0.0305 (0.0007)} & 0.0372 (0.0007) & 0.0355 (0.0007) & 0.0694 (0.0143)& 0.0395 (0.0134)\\
& 4000 & \textbf{0.0281 (0.0006)} & 0.0309 (0.0006) & 0.0299 (0.0006) & 0.0630 (0.0121)& 0.0383 (0.0132)\\
\midrule
\multirow{3}{*}{S4 (Chain)}
& 1000 & 0.0061 (0.0009) & 0.0095 (0.0011) & 0.0091 (0.0011) & 0.0128 (0.0046) & \textbf{0.0055 (0.0018)}\\
& 2000 & \textbf{0.0043 (0.0003)} & 0.0056 (0.0004) & 0.0055 (0.0004) & 0.0108 (0.0038) & 0.0049 (0.0015)\\
& 4000 & \textbf{0.0035 (0.0003)} & 0.0038 (0.0002) & 0.0037 (0.0002) & 0.0089 (0.0031) & 0.0047 (0.0015)\\
\midrule
\multirow{3}{*}{S5 (General DAG)}
& 1000 & \textbf{0.0055 (0.0009)} & 0.0094 (0.0008) & 0.0092 (0.0008) & 0.0072 (0.0023) & \textbf{0.0055 (0.0019)}\\
& 2000 & \textbf{0.0036 (0.0004)} & 0.0058 (0.0004) & 0.0058 (0.0004) & 0.0064 (0.0020) & 0.0052 (0.0018)\\
& 4000 & \textbf{0.0026 (0.0004)}& 0.0037 (0.0003) & 0.0037 (0.0003) & 0.0055 (0.0018) & 0.0049 (0.0016)\\
\midrule
\multirow{3}{*}{S6 (Gaussian mixture)}
& 1000 & \textbf{0.0100 (0.0023)} & 0.0145 (0.0025) & 0.0134 (0.0023) & 0.0118 (0.0028) & 0.0114 (0.0037) \\
& 2000 & \textbf{0.0085 (0.0023)} & 0.0105 (0.0022) & 0.0101 (0.0022) & 0.0114 (0.0027) & 0.0108 (0.0037)\\
& 4000 & \textbf{0.0077 (0.0022)} & 0.0082 (0.0022) & 0.0081 (0.0022) & 0.0109 (0.0027) & 0.0102 (0.0036)\\
\midrule
\multirow{3}{*}{S7 (Functional DAG)}
& 1000 & 0.0085 (0.0007) & 0.0173 (0.0021) & 0.0156 (0.0026) & \textbf{0.0075 (0.0006)} & 0.0675 (0.0243) \\
& 2000 & \textbf{0.0053 (0.0004)} & 0.0090 (0.0010) & 0.0083 (0.0011) & 0.0068 (0.0002) & 0.0633 (0.0185) \\
& 4000 & \textbf{0.0047 (0.0005)} & 0.0053 (0.0004) & 0.0050 (0.0004) & 0.0058 (0.0006) & 0.0623 (0.0182) \\
\midrule
\multirow{3}{*}{S8 (Network balls)}
& 1000 & \textbf{0.0434 (0.0022)} & 0.0645 (0.0018) & 0.0612 (0.0025) & 0.0691 (0.0011) & 0.0463 (0.0009) \\
& 2000 & \textbf{0.0382 (0.0021)} & 0.0519 (0.0018) & 0.0492 (0.0017) & 0.0668 (0.0013) & 0.0438 (0.0011) \\
& 4000 & \textbf{0.0362 (0.0015)} & 0.0417 (0.0015) & 0.0402 (0.0013) & 0.0644 (0.0009) & 0.0428 (0.0010) \\
\bottomrule
\end{tabular}
\end{table}

\begin{table}[H]
\centering
\caption{Test MSE ($\times 10^2$) for S3 by object class. Standard deviations are shown in parentheses. Results are averaged over 30 images per class.}
\scriptsize
\label{tab:stl10_results_std}
\begin{tabular}{lccccc}
\toprule
Class & Pre-NN & Constrain-NN & Pre-NN-Clip & DRF & Engression\\
\midrule
Airplane & 2.68 (0.95) & \textbf{2.62 (0.93)} & 2.66 (0.95) & 3.26 (1.07) & 6.09 (2.21)\\
Bird     & 3.75 (1.34) & \textbf{3.52 (1.12)} & 3.73 (1.32) & 4.68 (1.53) & 7.91 (2.29) \\
Car      & 3.70 (1.12) & \textbf{3.45 (1.05)} & 3.69 (1.11) & 4.24 (1.22) & 8.91 (2.36) \\
Cat      & 3.70 (1.60) & \textbf{3.50 (1.62)} & 3.69 (1.60) & 4.41 (1.75) & 7.47 (2.27) \\
Deer     & 4.17 (1.18) & \textbf{3.99 (1.14)} & 4.15 (1.18) & 4.87 (1.21) & 7.97 (1.84) \\
Dog      & 3.72 (1.43) & \textbf{3.56 (1.32)} & 3.71 (1.43) & 4.55 (1.58) & 7.92 (2.56) \\
Horse    & 4.50 (1.20) & \textbf{4.10 (1.09)} & 4.49 (1.19) & 5.37 (1.52) & 9.23 (2.39) \\
Monkey   & 3.42 (1.11) & \textbf{3.18 (1.10)} & 3.41 (1.12) & 4.34 (1.25) & 8.12 (2.06) \\
Ship     & 2.83 (0.81) & \textbf{2.35 (0.87)} & 2.86 (0.82) & 2.73 (0.98) & 7.08 (1.69) \\
Truck    & 3.79 (1.12) & \textbf{3.38 (0.99)} & 3.76 (1.10) & 4.14 (1.36) & 8.92 (2.44) \\
\midrule
Overall  & 3.63 (1.29) & \textbf{3.37 (1.24)} & 3.61 (1.29) & 4.26 (1.53) & 7.96 (2.37) \\
\bottomrule
\end{tabular}
\end{table}

\subsection{Configuration for estimators}
\label{app:num-config-estim}
In each simulation scenario except the special scenario using real images, each neural network for \textbf{Pre-NN} and \textbf{Constrain-NN} is a four-layer MLP with $64$ ReLU units per layer, trained for $200$ epochs using Adam with initial learning rate $10^{-3}$ and a cosine annealing schedule decaying to $1\%$ of the initial learning rate. For the special scenario, in order to model the more complex data, we use a five-layer MLP with $256$ ReLU units per layer, trained for $500$ epochs with the same optimizer and schedule. For \textbf{Engression}, we use the same architecture as the corresponding \textbf{Pre-NN}/\textbf{Constrain-NN} and draw $500$ samples per query point to estimate set probabilities with the \texttt{Engression} R package. For \textbf{DRF}, we fit $2000$ trees using the default hyperparameters of the \texttt{drf} R package, with the same covariate inputs as the neural network methods.

% REVISION START: set-family construction guideline

% REVISION END: set-family construction guideline

\section{Real data experiments}
\label{app:real-experiment}
In this section, we provide additional details and complete results for the two real-data experiments. Appendix~\ref{app:beijing} covers the Beijing Multi-Site Air Quality experiment, including target set construction, input processing, estimator configurations, and per-station results, while Appendix~\ref{app:chicago} reports the Chicago 2015 Crime experiment in full.

\subsection{Real data: Beijing multi-Site air quality prediction}
\label{app:beijing}
In this section, we provide the experimental details omitted from the main text. In Subsection~\ref{app:beijing-sets}, we provide a summary of the response variables (the pollutants) and describe the construction of the target sets. In Subsection~\ref{app:beijing-vars}, we list the variables contained in the dataset and describe how the model inputs are processed. In Subsection~\ref{app:beijing-config}, we provide the configurations of all estimators. In Subsection~\ref{app:beijing-result}, we provide full results and detailed results broken down by monitoring station.

\subsubsection{Set construction}
\label{app:beijing-sets}
Here we describe the construction of the target sets. Specifically,  an initial set of candidates is defined based on the regulatory document, and the collection is then pruned with several steps according to the empirical data distribution to retain only the meaningful sets.

\paragraph{Pollutant grading and empirical distribution.}
The Chinese Ambient Air Quality Standards (HJ~633-2012) regulate each pollutant under multiple threshold schemes. The starting point of our threshold design draws from this document, adopting the grading scheme that corresponds to the strictest available setting per pollutant; this choice is without loss of generality in comparing our experimental estimators in a distributional regression regime. Table~\ref{tab:aqi_breakpoints} summarizes the original grade breakpoints we use and their empirical prevalence across all $420768$ hourly observations in the dataset.

\begin{table}[htbp]
\centering
\caption{AQI grade breakpoints we used and empirical prevalence (\%) for each pollutant. All concentration units in $\mu$g/m$^3$. For each pollutant, the first row (``upper'') gives the upper concentration boundary of each grade, and the second row (``\%'') gives the empirical prevalence of that grade in the dataset. The eight columns correspond to the eight pollution grades: Excellent, Good, Light pollution, Moderate pollution, Heavy pollution, Severe pollution, Hazardous, and Beyond Index.}
\label{tab:aqi_breakpoints}
\begin{tabular}{lcccccccc}
\toprule
 & Excellent & Good & Light & Moderate & Heavy & Severe & Hazardous & Beyond \\
\midrule
PM2.5 upper & 35   & 75   & 115  & 150   & 250   & 350   & 500   & $\infty$ \\
PM2.5 \%    & 37.34 & 23.58 & 15.48 & 8.39 & 10.74 & 3.21 & 1.04 & 0.23 \\
\midrule
PM10 upper  & 50   & 150  & 250  & 350   & 420   & 500   & 600   & $\infty$ \\
PM10 \%     & 33.87 & 42.66 & 16.41 & 4.88 & 1.15 & 0.54 & 0.29 & 0.19 \\
\midrule
SO$_2$ upper & 50  & 150  & 475  & 800   & 1600  & 2100  & 2620  & $\infty$ \\
SO$_2$ \%   & 92.81 & 6.99 & 0.20 & 0.00 & 0.00 & 0.00 & 0.00 & 0.00 \\
\midrule
NO$_2$ upper & 40  & 80   & 180  & 280   & 565   & 750   & 940   & $\infty$ \\
NO$_2$ \%   & 46.99 & 34.07 & 18.59 & 0.35 & 0.00 & 0.00 & 0.00 & 0.00 \\
\midrule
CO upper    & 2000 & 4000 & 14000 & 24000 & 36000 & 48000 & 60000 & $\infty$ \\
CO \%       & 84.29 & 12.32 & 3.39 & 0.00 & 0.00 & 0.00 & 0.00 & 0.00 \\
\midrule
O$_3$ upper & 160  & 200  & 300  & 400   & 800   & 1000  & 1200  & $\infty$ \\
O$_3$ \%    & 93.40 & 3.38 & 3.06 & 0.14 & 0.01 & 0.00 & 0.00 & 0.00 \\
\bottomrule
\end{tabular}
\end{table}
\paragraph{Grade merging.}
After merging sparse grades into the next lower grade from the top down (i.e., if the highest remaining grade has prevalence $<2\%$, it is merged with the grade immediately below, repeating until the top grade exceeds $2\%$), and removing the excellent/good boundary to focus on pollution events, we retain thresholds from the light pollution grade upward. For SO$_2$, although exceedances of the light pollution threshold (SO$_2 \geq 150\ \mu$g/m$^3$) occur in only $0.20\%$ of observations, we retain this pollutant to ensure all six dimensions are represented in the joint exceedance structure. The resulting thresholds used to define exceedance sets are summarized in Table~\ref{tab:final_thresholds}.

\begin{table}[htbp]
\centering
\caption{Exceedance thresholds after grade merging. Each value defines an exceedance event $\{\text{pollutant} \geq \text{value}\}$; the value $0$ indicates the dimension is unconstrained.}
\label{tab:final_thresholds}
\begin{tabular}{ll}
\toprule
Pollutant & Thresholds ($\mu$g/m$^3$) \\
\midrule
PM2.5  & $0,\ 75,\ 115,\ 150,\ 250$ \\
PM10   & $0,\ 150,\ 250,\ 350$      \\
SO$_2$ & $0,\ 150$                  \\
NO$_2$ & $0,\ 80$                   \\
CO     & $0,\ 4{,}000$              \\
O$_3$  & $0,\ 200$                  \\
\bottomrule
\end{tabular}
\end{table}

\paragraph{Set construction and pruning.}
For each pollutant $k \in \{1, \ldots, 6\}$, let $\mathcal{T}_k = \{t_0^{(k)}, t_1^{(k)}, \ldots, t_{L_k}^{(k)}\}$ denote the set of available thresholds, with $t_0^{(k)} = 0$ indicating no constraint. Each exceedance set is then specified by a level tuple $(\ell_1, \ldots, \ell_6)$ with $\ell_k \in \{0, 1, \ldots, L_k\}$, defining
\[
A_{(\ell_1, \ldots, \ell_6)} = \bigl\{ y \in \mathbb{R}^6 : y_k \geq t_{\ell_k}^{(k)} \text{ for all } k = 1, \ldots, 6 \bigr\}.
\]
A level $\ell_k = 0$ imposes no constraint on pollutant $k$. Taking the full Cartesian product of all per-pollutant level sets yields $5 \times 4 \times 2 \times 2 \times 2 \times 2 = 320$ candidate sets, which we reduce in 2 steps:
\begin{enumerate}
    \item \textbf{Prevalence filtering.} We retain only those sets whose empirical prevalence---the fraction of observations satisfying the joint exceedance condition---is at least $0.5\%$. This reduces the $320$ candidates to $86$.
    \item \textbf{Gap pruning.} We construct DAG of the retained sets under the componentwise partial ordering: $A_{(\ell_1, \ldots, \ell_6)} \subseteq A_{(\ell_1', \ldots, \ell_6')}$ if and only if $\ell_k \geq \ell_k'$ for all $k$ (higher levels correspond to smaller sets). Starting from the deepest level of the diagram, a set is removed if its empirical prevalence differs from that of its closest parent by less than $0.001$ percentage points, indicating that the two sets are nearly indistinguishable. This step removes $9$ sets.
\end{enumerate}
Finally, we exclude the unconstrained full set ($\ell_k = 0$ for all $k$), whose probability is trivially $1$, leaving $76$ nested sets. Table~\ref{tab:sets_extremes} reports the ten sets with the highest and the ten sets with the lowest empirical prevalences.

\begin{table}[htbp]
\centering
\caption{Exceedance sets with the ten highest and ten lowest empirical prevalences (after pruning). The full collection contains $76$ sets.}
\label{tab:sets_extremes}
\begin{tabular}{cl r}
\toprule
Rank & Exceedance condition & Prevalence (\%) \\
\midrule
\multicolumn{3}{l}{\textit{Highest prevalence}} \\
1  & $\text{PM}_{2.5} \geq 75$                                          & 39.43 \\
2  & $\text{PM}_{2.5} \geq 115$                                         & 23.80 \\
3  & $\text{PM}_{10} \geq 150$                                          & 23.79 \\
4  & $\text{PM}_{2.5} \geq 75 \;\wedge\; \text{PM}_{10} \geq 150$      & 21.71 \\
5  & $\text{NO}_2 \geq 80$                                              & 19.53 \\
6  & $\text{PM}_{2.5} \geq 115 \;\wedge\; \text{PM}_{10} \geq 150$     & 18.73 \\
7  & $\text{PM}_{2.5} \geq 75 \;\wedge\; \text{NO}_2 \geq 80$          & 16.16 \\
8  & $\text{PM}_{2.5} \geq 150$                                         & 15.34 \\
9  & $\text{PM}_{2.5} \geq 150 \;\wedge\; \text{PM}_{10} \geq 150$     & 14.91 \\
10 & $\text{PM}_{10} \geq 150 \;\wedge\; \text{NO}_2 \geq 80$          & 13.07 \\
\midrule
\multicolumn{3}{l}{\textit{Lowest prevalence}} \\
67 & $\text{PM}_{2.5} \geq 250 \;\wedge\; \text{PM}_{10} \geq 350$                                                                      & 1.74 \\
68 & $\text{PM}_{10} \geq 350 \;\wedge\; \text{NO}_2 \geq 80$                                                                           & 1.73 \\
69 & $\text{PM}_{2.5} \geq 75 \;\wedge\; \text{PM}_{10} \geq 350 \;\wedge\; \text{NO}_2 \geq 80$                                       & 1.73 \\
70 & $\text{PM}_{10} \geq 350 \;\wedge\; \text{NO}_2 \geq 80 \;\wedge\; \text{CO} \geq 4000$                                           & 1.15 \\
71 & $\text{PM}_{2.5} \geq 250 \;\wedge\; \text{PM}_{10} \geq 350 \;\wedge\; \text{CO} \geq 4000$                                      & 1.14 \\
72 & $\text{PM}_{2.5} \geq 250 \;\wedge\; \text{PM}_{10} \geq 350 \;\wedge\; \text{NO}_2 \geq 80 \;\wedge\; \text{CO} \geq 4000$       & 1.13 \\
73 & $\text{PM}_{2.5} \geq 115 \;\wedge\; \text{O}_3 \geq 200$                                                                          & 0.92 \\
74 & $\text{PM}_{10} \geq 150 \;\wedge\; \text{O}_3 \geq 200$                                                                           & 0.91 \\
75 & $\text{PM}_{2.5} \geq 75 \;\wedge\; \text{PM}_{10} \geq 150 \;\wedge\; \text{O}_3 \geq 200$                                       & 0.84 \\
76 & $\text{PM}_{2.5} \geq 115 \;\wedge\; \text{PM}_{10} \geq 150 \;\wedge\; \text{O}_3 \geq 200$                                      & 0.62 \\
\bottomrule
\end{tabular}
\end{table}

\subsubsection{Dataset variables and input processing}
\label{app:beijing-vars}

The Beijing Multi-Site Air Quality dataset contains hourly observations from 12 monitoring stations across Beijing, covering March 2013 to February 2017. For each station and each timestamp, the dataset records six pollutant concentrations $(\text{PM}_{2.5}, \text{PM}_{10}, \text{SO}_2, \text{NO}_2, \text{CO}, \text{O}_3)$, five meteorological variables (temperature, pressure, dew point temperature, wind speed, precipitation), and the wind direction. 

We construct the model inputs from these raw variables in three steps. First, the hour-of-day, day-of-week, and month-of-year are each mapped to a $(\sin, \cos)$ pair to preserve cyclic continuity, and the 16-direction wind label is similarly converted to its angle (e.g., $\text{N} \to 0^\circ$, $\text{NNE} \to 22.5^\circ$, etc.) and then to $(\sin, \cos)$. Second, given a forecast horizon $h$ and an input window length $W$, we collect the measurements at lags $\{h, h+1, \ldots, h+W-1\}$ relative to the target time; for each lag, we include all five meteorological variables, the two wind-direction components, and the six pollutants, yielding $13$ variables per lag. Third, rows containing any missing pollutant value at the target timestamp or any missing variable in the input window are removed.

%Note that this encoding is not one-to-one with individual data points, so it does not introduce any index-based information leakage between training and test sets. 

For each target time $s$, the input is the concatenation of the cyclical time encodings at $s-h$ (the earliest end of the input window) with the $W$ lagged variable blocks aggregated over the input window.

For each monitoring station, we construct the training and test sets separately 
within each year of the 4-year study period. In each year, 3 consecutive 
months are held out for testing, and all observations within the held out period 
are assigned to the test set, with no intervening observations included in the 
training set. The remaining 9 months are then used for training, yielding an 
overall 75/25 train/test split.

\subsubsection{Estimator configurations}
\label{app:beijing-config}
Each neural network instance consists of a three-layer MLP with ReLU activations. Hidden layer width is set to $64$ for all instances. All input features are standardized to zero mean and unit variance using training set statistics. Models are trained for $40$ epochs with the Adam optimizer at a learning rate of $3\times 10^{-4}$ and a batch size of \(100\). For Distributional Random Forests, we fit $2000$ trees per station using the default hyperparameters provided by the \texttt{drf} R package, with the same covariate inputs as the neural network methods.

\subsubsection{Detailed results}

\label{app:beijing-result}

In this subsection, we report both the aggregated test MSE with standard deviation and the per-station test MSE for the five estimators: \textbf{Pre-NN}, \textbf{Constrain-NN}, \textbf{Pre-NN-Clip}, \textbf{DRF}, and \textbf{Engression}. Results are presented across all $3 \times 3 = 9$ configurations of the forecast horizon $h$ and input window length $W$. Table~\ref{tab:beijing-mse-expanded} reports the aggregated results. Table~\ref{tab:beijing-mse-per-station} presents the results for all 12 monitoring stations, with each cell reporting the average MSE over the 76 target sets. Consistent with the results in the main text, \textbf{Constrain-NN} achieves the lowest MSE in the vast majority of cells across all stations, and the other findings reported in the main text also hold at the per-station level.

\begin{table}[htbp]
\centering
\caption{Test MSE ($\times 10^{2}$) across 76 sets and aggregated over the 12 monitoring stations. Each entry reports $\text{mean},(\text{std})$ over stations. The best performance is bolded.}
\label{tab:beijing-mse-expanded}
\renewcommand{\arraystretch}{1.15}
\setlength{\tabcolsep}{5pt}
\begin{tabular}{c l c c c}
\toprule
Horizon $h$ & Method & $W=3$ & $W=6$ & $W=9$ \\
\midrule
\multirow{5}{*}{$h=3$}
& Pre-NN
& 3.43 (0.68)
& 3.41 (0.67)
& 3.35 (0.66) \\
& Constrain-NN
& \textbf{2.28 (0.48)}
& \textbf{2.24 (0.46)}
& \textbf{2.23 (0.46)} \\
& Pre-NN-Clip
& 3.54 (0.71)
& 3.51 (0.70)
& 3.46 (0.69) \\
& DRF
& 2.29 (0.48)
& 2.33 (0.50)
& 2.36 (0.50) \\
& Engression
& 3.79 (0.79)
& 3.94 (0.87)
& 3.92 (0.87) \\
\midrule
\multirow{5}{*}{$h=6$}
& Pre-NN
& 3.78 (0.74)
& 3.71 (0.74)
& 3.65 (0.74) \\
& Constrain-NN
& \textbf{2.84 (0.59)}
& \textbf{2.81 (0.57)}
& \textbf{2.79 (0.57)} \\
& Pre-NN-Clip
& 3.88 (0.74)
& 3.79 (0.77)
& 3.73 (0.77) \\
& DRF
& 2.90 (0.62)
& 2.91 (0.62)
& 2.92 (0.62) \\
& Engression
& 4.12 (0.90)
& 4.24 (0.95)
& 4.18 (0.94) \\
\midrule
\multirow{5}{*}{$h=9$}
& Pre-NN
& 3.97 (0.74)
& 3.87 (0.76)
& 3.79 (0.77) \\
& Constrain-NN
& \textbf{3.15 (0.61)}
& \textbf{3.10 (0.60)}
& \textbf{3.09 (0.61)} \\
& Pre-NN-Clip
& 4.06 (0.75)
& 3.95 (0.76)
& 3.87 (0.79) \\
& DRF
& 3.21 (0.67)
& 3.20 (0.67)
& 3.17 (0.66) \\
& Engression
& 4.31 (0.94)
& 4.40 (0.99)
& 4.30 (0.96) \\
\bottomrule
\end{tabular}
\end{table}

\begin{table}
\scriptsize
\centering
\caption{Test MSE ($\times 10^{2}$) across 76 sets, reported per station and per (horizon $h$, input window $W$) configuration. Each cell reports results for \textbf{Pre-NN} / \textbf{Constrain-NN} / \textbf{Pre-NN-Clip} / \textbf{DRF} / \textbf{Engression}. The best performance in each cell is bolded.}
\label{tab:beijing-mse-per-station}
\renewcommand{\arraystretch}{1.2}
\begin{tabular}{cc ccc}
\toprule
 & & \multicolumn{3}{c}{Input Window $W$ (hours)} \\
\cmidrule(lr){3-5}
Station & Horizon $h$ & $W=3$ & $W=6$ & $W=9$ \\
\midrule
\multirow{3}{*}{Aotizhongxin}
& $h=3$ & 3.86 / \textbf{2.72} / 3.99 / \textbf{2.72} / 4.36 & 3.90 / \textbf{2.73} / 4.02 / 2.83 / 4.69 & 3.86 / \textbf{2.71} / 3.99 / 2.88 / 4.70 \\
& $h=6$ & 4.29 / \textbf{3.30} / 4.38 / 3.44 / 4.78 & 4.27 / \textbf{3.37} / 4.37 / 3.48 / 5.08 & 4.25 / \textbf{3.39} / 4.36 / 3.52 / 5.06 \\
& $h=9$ & 4.46 / \textbf{3.62} / 4.54 / 3.75 / 5.01 & 4.44 / \textbf{3.65} / 4.52 / 3.79 / 5.27 & 4.38 / \textbf{3.75} / 4.47 / 3.78 / 5.16 \\
\midrule
\multirow{3}{*}{Changping}
& $h=3$ & 3.06 / \textbf{2.03} / 3.16 / 2.04 / 3.25 & 3.00 / \textbf{1.98} / 3.10 / 2.04 / 3.22 & 2.91 / \textbf{1.92} / 2.99 / 2.03 / 3.15 \\
& $h=6$ & 3.38 / \textbf{2.48} / 3.48 / 2.50 / 3.51 & 3.23 / \textbf{2.41} / 3.32 / 2.48 / 3.46 & 3.12 / \textbf{2.37} / 3.19 / 2.43 / 3.37 \\
& $h=9$ & 3.61 / 2.77 / 3.69 / \textbf{2.75} / 3.64 & 3.40 / \textbf{2.65} / 3.46 / 2.67 / 3.55 & 3.28 / \textbf{2.59} / 3.31 / 2.62 / 3.49 \\
\midrule
\multirow{3}{*}{Dingling}
& $h=3$ & 2.18 / 1.52 / 2.23 / \textbf{1.42} / 2.37 & 2.19 / 1.55 / 2.23 / \textbf{1.46} / 2.47 & 2.18 / 1.55 / 2.23 / \textbf{1.50} / 2.47 \\
& $h=6$ & 2.38 / 1.82 / 2.47 / \textbf{1.77} / 2.49 & 2.35 / 1.84 / 2.38 / \textbf{1.80} / 2.58 & 2.34 / 1.84 / 2.38 / \textbf{1.83} / 2.59 \\
& $h=9$ & 2.57 / 2.04 / 2.63 / \textbf{2.01} / 2.62 & 2.50 / \textbf{2.03} / 2.58 / \textbf{2.03} / 2.70 & 2.49 / 2.09 / 2.52 / \textbf{2.06} / 2.75 \\
\midrule
\multirow{3}{*}{Dongsi}
& $h=3$ & 3.56 / \textbf{2.43} / 3.70 / 2.50 / 3.88 & 3.54 / \textbf{2.39} / 3.69 / 2.52 / 4.13 & 3.44 / \textbf{2.40} / 3.62 / 2.53 / 4.02 \\
& $h=6$ & 3.94 / \textbf{3.08} / 4.04 / 3.14 / 4.27 & 3.88 / \textbf{2.99} / 4.00 / 3.13 / 4.43 & 3.77 / \textbf{2.89} / 3.91 / 3.09 / 4.25 \\
& $h=9$ & 4.12 / \textbf{3.35} / 4.21 / 3.44 / 4.49 & 4.01 / \textbf{3.17} / 4.12 / 3.40 / 4.57 & 3.92 / \textbf{3.09} / 4.05 / 3.34 / 4.38 \\
\midrule
\multirow{3}{*}{Guanyuan}
& $h=3$ & 3.53 / 2.49 / 3.65 / \textbf{2.42} / 3.88 & 3.53 / \textbf{2.40} / 3.65 / 2.47 / 4.13 & 3.47 / \textbf{2.37} / 3.61 / 2.50 / 4.13 \\
& $h=6$ & 3.94 / \textbf{3.06} / 4.04 / 3.10 / 4.32 & 3.90 / \textbf{3.04} / 4.00 / 3.12 / 4.52 & 3.83 / \textbf{3.00} / 3.94 / 3.10 / 4.44 \\
& $h=9$ & 4.11 / \textbf{3.33} / 4.19 / 3.44 / 4.54 & 4.03 / \textbf{3.30} / 4.12 / 3.40 / 4.65 & 3.94 / \textbf{3.25} / 4.04 / 3.34 / 4.50 \\
\midrule
\multirow{3}{*}{Gucheng}
& $h=3$ & 4.00 / \textbf{2.75} / 4.13 / 2.78 / 4.55 & 3.98 / \textbf{2.70} / 4.10 / 2.84 / 4.81 & 3.89 / \textbf{2.66} / 4.02 / 2.85 / 4.73 \\
& $h=6$ & 4.46 / \textbf{3.40} / 4.55 / 3.44 / 5.08 & 4.29 / \textbf{3.28} / 4.39 / 3.45 / 5.13 & 4.21 / \textbf{3.22} / 4.31 / 3.51 / 5.06 \\
& $h=9$ & 4.61 / \textbf{3.75} / 4.73 / 3.79 / 5.20 & 4.47 / \textbf{3.64} / 4.55 / 3.80 / 5.32 & 4.34 / \textbf{3.59} / 4.43 / 3.80 / 5.17 \\
\midrule
\multirow{3}{*}{Huairou}
& $h=3$ & 2.10 / \textbf{1.32} / 2.13 / 1.36 / 2.29 & 2.09 / \textbf{1.34} / 2.13 / 1.37 / 2.38 & 2.09 / \textbf{1.32} / 2.13 / 1.43 / 2.36 \\
& $h=6$ & 2.37 / \textbf{1.71} / 2.44 / 1.72 / 2.43 & 2.24 / \textbf{1.73} / 2.27 / 1.76 / 2.51 & 2.21 / \textbf{1.72} / 2.24 / 1.79 / 2.47 \\
& $h=9$ & 2.50 / \textbf{1.96} / 2.57 / \textbf{1.96} / 2.54 & 2.36 / \textbf{1.97} / 2.44 / \textbf{1.97} / 2.56 & 2.26 / 2.00 / 2.29 / \textbf{1.95} / 2.49 \\
\midrule
\multirow{3}{*}{Nongzhanguan}
& $h=3$ & 4.17 / \textbf{2.78} / 4.30 / 2.80 / 4.71 & 4.13 / \textbf{2.68} / 4.25 / 2.85 / 5.00 & 4.07 / \textbf{2.67} / 4.20 / 2.89 / 4.95 \\
& $h=6$ & 4.59 / \textbf{3.53} / 4.70 / 3.61 / 5.14 & 4.54 / \textbf{3.45} / 4.63 / 3.66 / 5.39 & 4.46 / \textbf{3.44} / 4.57 / 3.66 / 5.25 \\
& $h=9$ & 4.81 / \textbf{3.86} / 4.89 / 4.07 / 5.40 & 4.71 / \textbf{3.76} / 4.79 / 4.04 / 5.60 & 4.68 / \textbf{3.81} / 4.77 / 4.02 / 5.44 \\
\midrule
\multirow{3}{*}{Shunyi}
& $h=3$ & 3.39 / \textbf{1.91} / 3.47 / 2.09 / 3.56 & 3.24 / \textbf{1.87} / 3.33 / 2.12 / 3.54 & 3.16 / \textbf{1.81} / 3.19 / 2.11 / 3.46 \\
& $h=6$ & 3.72 / \textbf{2.55} / 3.78 / 2.72 / 3.95 & 3.62 / \textbf{2.48} / 3.68 / 2.71 / 3.93 & 3.44 / \textbf{2.53} / 3.52 / 2.73 / 3.79 \\
& $h=9$ & 3.95 / \textbf{3.00} / 4.05 / 3.08 / 4.19 & 3.75 / \textbf{2.93} / 3.81 / 3.05 / 4.12 & 3.58 / \textbf{2.84} / 3.63 / 2.97 / 3.90 \\
\midrule
\multirow{3}{*}{Tiantan}
& $h=3$ & 3.70 / 2.33 / 3.78 / \textbf{2.32} / 4.12 & 3.66 / \textbf{2.33} / 3.75 / 2.36 / 4.08 & 3.58 / \textbf{2.33} / 3.68 / 2.38 / 4.10 \\
& $h=6$ & 4.00 / \textbf{2.93} / 4.12 / 2.96 / 4.35 & 3.92 / \textbf{2.89} / 3.99 / 2.98 / 4.36 & 3.86 / \textbf{2.87} / 3.94 / 2.97 / 4.36 \\
& $h=9$ & 4.16 / \textbf{3.26} / 4.26 / 3.27 / 4.54 & 4.12 / \textbf{3.15} / 4.23 / 3.25 / 4.59 & 4.09 / \textbf{3.18} / 4.21 / 3.26 / 4.59 \\
\midrule
\multirow{3}{*}{Wanliu}
& $h=3$ & 3.65 / 2.53 / 3.77 / \textbf{2.46} / 4.05 & 3.66 / 2.51 / 3.78 / \textbf{2.50} / 4.29 & 3.68 / \textbf{2.55} / 3.82 / 2.56 / 4.42 \\
& $h=6$ & 4.00 / 3.06 / 4.09 / \textbf{3.04} / 4.39 & 4.04 / 3.12 / 4.14 / \textbf{3.11} / 4.65 & 4.06 / 3.18 / 4.16 / \textbf{3.15} / 4.73 \\
& $h=9$ & 4.24 / 3.41 / 4.31 / \textbf{3.39} / 4.64 & 4.23 / 3.46 / 4.32 / \textbf{3.41} / 4.88 & 4.20 / 3.45 / 4.29 / \textbf{3.39} / 4.87 \\
\midrule
\multirow{3}{*}{Wanshouxigong}
& $h=3$ & 4.01 / \textbf{2.56} / 4.13 / 2.59 / 4.41 & 3.96 / \textbf{2.47} / 4.08 / 2.61 / 4.58 & 3.90 / \textbf{2.44} / 4.03 / 2.63 / 4.53 \\
& $h=6$ & 4.33 / \textbf{3.21} / 4.43 / 3.32 / 4.76 & 4.25 / \textbf{3.10} / 4.34 / 3.28 / 4.85 & 4.20 / \textbf{3.08} / 4.30 / 3.27 / 4.76 \\
& $h=9$ & 4.54 / \textbf{3.50} / 4.62 / 3.62 / 4.97 & 4.41 / \textbf{3.47} / 4.49 / 3.57 / 5.03 & 4.38 / \textbf{3.53} / 4.47 / 3.56 / 4.91 \\
\bottomrule
\end{tabular}
\end{table}
\newpage

\subsection{Real data: Chicago 2015 crime distribution estimation}
\label{app:chicago}
We further evaluate our framework on the 2015 Chicago Crime dataset (publicly available at \url{https://data.gov/open-gov/}), following an experimental setup similar to that of~\citet{tansey2018maximum, padilla2025risk} but extended to a multivariate response. Specifically, whereas~\citet{padilla2025risk} models the conditional CDF of total log-crime counts through univariate distributional regression, we model the joint distribution of crime composition across the most frequent categories. This category-level formulation is important for downstream decision making, since regions with similar total crime volume may differ substantially in their dominant crime types and therefore require different policing interventions.

Following \citet{tansey2018maximum}, the city is discretized into a $100 \times 100$ grid based on latitude and longitude, and grid cells with no recorded crimes are excluded, yielding $n = 3844$ cells. The covariate $X$ is two dimensional and records the longitude bin index and latitude bin index of each cell, so the regression is purely spatial in our setting. The response variable is four dimensional, consisting of the log-transformed crime counts of the four most frequent crime categories in the dataset, namely theft, battery, criminal damage, and narcotics.

We construct a finite collection of nested target sets indexed by quantile exceedances of the four crime categories. For each crime category $k \in \{1, \ldots, 4\}$, let $\mathcal{T}_k = \{t_0^{(k)}, t_1^{(k)}, t_2^{(k)}, t_3^{(k)}\}$ denote the four marginal quantile thresholds at levels $20\%, 40\%, 60\%, 80\%$ of the empirical response within that category, where $t_\ell^{(k)}$ is the $\ell$-th threshold for category $k$ and increases in $\ell$. Each candidate exceedance set is specified by a level tuple $(\ell_1, \ldots, \ell_4)$ with $\ell_k \in \{0, 1, 2, 3\}$ encoding the threshold level applied to each category, defining $A_{(\ell_1, \ldots, \ell_4)} = \bigl\{ y \in \mathbb{R}^4 : y_k \geq t_{\ell_k}^{(k)} \text{ for all } k = 1, \ldots, 4 \bigr\}.$ The full Cartesian product yields $256$ candidate sets, which we reduce by the same pruning procedure as in Appendix~\ref{app:beijing-sets}. Specifically, we retain only sets whose empirical prevalence is at least $1\%$, then construct the DAG of the retained sets under the partial ordering $A_{(\ell_1, \ldots, \ell_4)} \subseteq A_{(\ell_1', \ldots, \ell_4')}$ if and only if $\ell_k \geq \ell_k'$ for all $k$. Starting from the deepest level, we remove any set whose empirical prevalence differs from that of its closest parent by less than $1$ percentage point. This step removes $174$ sets, leaving $m = 82$ nested sets whose inclusion relations form a non-trivial DAG.

We use an experimental configuration similar to that in \citet{padilla2025risk}. Specifically, each neural network instance in \textbf{Pre-NN} and \textbf{Constrain-NN} is a five-layer MLP with ReLU activations and hidden width $256$, trained for $1000$ epochs with the Adam optimizer at a learning rate of $10^{-3}$ annealing to $10^{-6}$. For \textbf{Engression}, we use the same model architecture and draw $500$ samples per query point for distribution estimation. For \textbf{DRF}, we fit $2000$ trees with the same covariate inputs as the neural network methods. We conduct $30$ Monte Carlo rounds of the experiment, with each replication using an independent random $75/25$ train–test split.

\begin{table}[h]
\centering
\caption{Test MSE across $m = 82$ sets, averaged over $30$ random splits. Standard deviations in parentheses.}
\scriptsize
\label{tab:chicago-results}
\begin{tabular}{lccccc}
\toprule
       & Pre-NN & Constrain-NN & Pre-NN-Clip & DRF & Engression  \\
\midrule
MSE     & $0.1194\ (0.0046)$ & $\mathbf{0.1186\ (0.0044)}$ & $0.1248\ (0.0046)$ & $0.1200\ (0.0032)$ & $0.1408\ (0.0039)$  \\
\bottomrule
\end{tabular}
\end{table}

Table~\ref{tab:chicago-results} reports the test MSE averaged over the $m = 82$ sets and $30$ replications. \textbf{Constrain-NN} achieves the best performance with the lowest MSE, followed by \textbf{Pre-NN} and DRF, which perform comparably, while \textbf{Pre-NN-Clip} performs worse than \textbf{Pre-NN}. These results demonstrate the importance of imposing the structural constraint, as done in \textbf{Constrain-NN}.
{
\subsection{Practical guidance for constructing a set family}
\label{app:set-guideline}

The target set family should be selected according to the conditional probabilities relevant to the scientific problem. We suggest the following general procedure.

\begin{enumerate}
\item \textbf{Use domain-relevant events when they are available.} If the application provides meaningful thresholds or downstream decisions, the sets can be defined directly from them. For example, the Beijing air-quality experiment uses joint pollutant exceedance events based on official grading thresholds.

\item \textbf{Use measurable summaries or metrics for complex responses.} A response object can be mapped to one or more interpretable scalar summaries, with quantiles of their marginal distributions used as thresholds. For a network response, examples include network statistics and distances to a reference network. For a functional response, examples include $L^1$, $L^2$, and $L^\infty$ distances to a reference function. Sublevel sets of a single summary give a nested family, while intersections constructed from several summaries can produce a general DAG.

\item \textbf{Remove empirically redundant sets when necessary.} If an initial family is large, sets having nearly the same empirical prevalence as their closest supersets provide little additional distributional information and may be removed. After a node is removed, its adjacent inclusion relations can be reconnected and transitively redundant edges deleted, preserving the inclusion order among the retained sets.
\end{enumerate}

The estimation procedure only requires membership indicators for the selected sets and the DAG encoding their inclusion relations. It therefore applies in the same way whether the sets arise from scientific knowledge, response summaries, or metrics.
}

\end{document}